\documentclass[10pt,onecolumn,letterpaper]{article}
\usepackage{cvpr}
\usepackage{times}
\usepackage{subfigure}
\usepackage{epsfig}
\usepackage{graphicx}
\usepackage{amsthm}
\usepackage{amsmath}
\usepackage{amssymb}
\usepackage{algorithm,algpseudocode}

\newtheorem{theorem}{Theorem}
\newtheorem{remark}{Remark}
\newtheorem{lemma}{Lemma}
\newtheorem{definition}{Definition}
\newtheorem{assumption}{Assumption}
\newtheorem{corollary}{Corollary}

\usepackage{datetime}
\usepackage{wrapfig}
\usepackage{amsthm}

\usepackage{appendix}
\usepackage{multirow}
\usepackage{hhline}
\usepackage{lipsum}
\usepackage{authblk}
\cvprfinalcopy 
\def\cvprPaperID{****} 
\allowdisplaybreaks

\begin{document}
	
	\title{Variance Reduced methods for Non-convex  Composition Optimization}
	\author[1]{ Liu Liu\thanks{lliu8101@uni.sydney.edu.au}}
	\author[2]{ Ji Liu\thanks{ji.liu.uwisc@gmail.com}}
	\author[1]{ Dacheng Tao\thanks{dacheng.tao@sydney.edu.au}}
	\affil[1]{UBTECH Sydney AI Centre and SIT, FEIT, The University of Sydney}
	\affil[2]{Department of Computer Science, University of Rochester}
	\maketitle

	\begin{abstract}
		{This paper explores the non-convex composition optimization in the form including inner and outer finite-sum functions with a large number of component functions. This problem arises in some important applications  such as nonlinear embedding and reinforcement learning. Although existing approaches such as stochastic gradient descent (SGD) and stochastic variance reduced gradient (SVRG) descent can be applied to solve this problem, their query complexity tends to be high, especially when the number of inner component functions is large. In this paper, we apply the variance-reduced technique to derive two variance reduced algorithms that significantly improve the query complexity if the number of inner component functions is large. To the best of our knowledge, this is the first work that establishes the query complexity analysis for non-convex stochastic composition. Experiments validate the proposed algorithms and theoretical analysis.
		} 
	\end{abstract}
	
	
	\section{Introduction}
	In this paper, we study the problem of the following non-convex composition minimization
	\begin{align}\label{VRNonCS:ProblemMainComposition}
	\mathop {\min }\limits_{x \in {\mathbb{R}^N}} \left\{ {f( x )\mathop  = \limits^{{\rm{def}}} F( {G( x )} ) \mathop  = \limits^{{\rm{def}}} \frac{1}{n}\sum\limits_{i = 1}^n {{F_i}\left( {\frac{1}{m}\sum\limits_{j = 1}^m {{G_j}( x )} } \right)} } \right\},
	\end{align}
	where $f$: ${\mathbb{R}^N} \to \mathbb{R}$ is a non-convex function,  each $F_i$: ${\mathbb{R}^M} \to \mathbb{R}$ is a smooth function, each  $G_i$: ${\mathbb{R}^N} \to {\mathbb{R}^M}$ is a mapping function, $n$ is the number of $F_i$'s, and $m$ is the number of $G_j$'s. 	 We call $G(x)$:$=\frac{1}{m}\sum\nolimits_{j=1}^mG_j(x)$  the inner function, and $F( {w} )$:$ = \frac{1}{n}\sum\nolimits_{i = 1}^n {{F_i}( { w} )} $ the outer function. 
	
	This composition between two finite-sum structures $ \frac{1}{n}\sum\nolimits_{i = 1}^n {{F_i}( {\frac{1}{m}\sum\nolimits_{j = 1}^m {{G_j}( x )} } )} $ arises in many machine learning applications { such as reinforcement learning \cite{wang2017stochastic,wang2016accelerating,dai2016learning} and nonlinear embedding \cite{hinton2003stochastic}.} For example, stochastic neighbor embedding (SNE)  \cite{hinton2003stochastic} is a powerful approach to map data from a high dimensional space to a low dimensional space. Let $\{z_i\}_{i=1}^n$ and $\{x_i\}_{i=1}^n$ denote the representation of $n$ data points in the high dimensional space and the low dimensional space, respectively. The objective is to pursue a low dimensional representation $\{x_i\}_{i=1}^n$, such that the distribution in the low dimensional space is as close to the distribution in the high dimensional space as possible. This problem is essentially a composition optimization problem:
	
	\begin{align}\label{VRNonCS:Problem-SNE}
	\mathop {\min }\limits_x \sum\limits_t^{} {\sum\limits_i^{} {{p_{i|t}}\log \frac{{{p_{i|t}}}}{{{q_{i|t}}}}} } ,
	\end{align}
	where 
	\begin{align*}
	{p_{i|t}} = \frac{{\exp ( - {{\left\| {{z_t} - {z_i}} \right\|}^2}/2\sigma _i^2)}}{{\sum\nolimits_{j \ne t} {\exp ( - {{\left\| {{z_t} - {z_j}} \right\|}^2}/2\sigma _i^2)} }}, \quad {q_{i|t}} = \frac{{\exp ( - {{\left\| {{x_t} - {x_i}} \right\|}^2})}}{{\sum\nolimits_{j \ne t} {\exp ( - {{\left\| {{x_t} - {x_j}} \right\|}^2})} }}, 
	\end{align*}
	and $\sigma_i$ is {the predefined parameter to control the sensitivity to the distance.} 
	Problem (\ref{VRNonCS:Problem-SNE}) can be transformed into a composition problem as in (\ref{VRNonCS:ProblemMainComposition}), where 
	\begin{align*}
	{G_j}\left( x \right) = &{\left[ {x,{e^{ - {{\left\| {{x_1} - {x_j}} \right\|}^2}}} - 1,...,{e^{ - {{\left\| {{x_n} - {x_j}} \right\|}^2}}} - 1} \right]^\mathsf{T}} ,\\
	{F_i}\left( w \right) =&{p_{i|1}}({\left\| {{w_i} - {w_1}} \right\|^2} - \log ({w_{n + 1}})) + ... + {p_{i|n}}({\left\| {{w_i} - {w_n}} \right\|^2} - \log ({w_{n + n}})).
	\end{align*}
	More details can be found in Appendix B, and examples about reinforcement learning can be found in \cite{wang2017stochastic,wang2016accelerating, dai2016learning}.

	Such a finite-sum structure allows us to perform stochastic gradient descent (SGD). In particular, when minimizing problem (\ref{VRNonCS:ProblemMainComposition}), the stochastic gradient can be obtained by randomly and independently selecting $i$ and $j$ from $[m]$ and $[n]$ to form $(\partial {G_j}( x ))^\mathsf{T}\nabla {F_i}( {G( x )} )$, which satisfies
	\[\mathbb{E}[(\partial G_j( x ))^\mathsf{T}\nabla {F_i}( G( x ) )]=(\partial {G}(x))^\mathsf{T}\nabla {F}( {G( x )} ).\]
	When the inner function $G(x)$ and its partial gradient $\partial G(x)$ are computed directly for each iteration, the problem in (\ref{VRNonCS:ProblemMainComposition}) can be turned into the one finite-sum minimization problem $\frac{1}{n}\sum\nolimits_{i = 1}^n {{F_i}\left( {G(x)} \right)} $. Recently, \cite{allen2016improved} and  \cite{reddi2016stochastic} proposed the stochastic variance-reduced gradient (SVRG) method to solve such non-convex problems. Despite the best gradient complexity being provided, they did not apply SVRG to the composition of two finite-sum structures. Moreover, two main problems are encountered in such a composition of two finite-sum structures when using SGD:	
	\begin{itemize}
		\item The inner function $G(x)$ admits the finite-sum structure. Computing the inner function will be extremely expensive in large-scale data problems.  However, if $G(x)$ is estimated and replaced by $\hat G(x)$, that is $\mathbb{E}(\hat G(x))=G(x)$,  the estimated gradient of $f$ will result in a biased estimate. That is, $\mathbb{E}[(\partial G_j( x ))^\mathsf{T}\nabla F_i( G( x ) )] \ne (\partial G( x ))^\mathsf{T}\nabla F( G( x ) )$. Can variance reduction technology  be applied to the estimation of such an inner function?
		\item For the large number of inner sub-function, the SVRG-based method \cite{allen2016improved, reddi2016stochastic} will need more query complexity. Because it needs to compute the sum of inner sub-function and its corresponding partial gradient. Can variance reduction technology  also  reduce the query complexity for the composition problem?
	\end{itemize}

		\begin{table}[t]
		\centering
		\begin{tabular}{|l|c|c|c|c|}
			\hline
			\multirow{2}{*}{Algorithm} & \multirow{2}{*}{Iteration Complexity}& \multirow{2}{*}{Gradient Complexity} & \multicolumn{2}{c|}{Query Complexity} \\ \cline{4-5} &			&                    & $0\le m_0\le1$        & {$m_0>1 $}          \\ 
			\hhline{|=|=|=|=|=|}
			Full GD \cite{ghadimi2016accelerated} 
			&\multicolumn{1}{c|}{$\mathcal{O}(1/\varepsilon )$}
			&\multicolumn{1}{c|}{$\mathcal{O}(n/\varepsilon )$}
			&\multicolumn{1}{c|}{$\mathcal{O}(n/\varepsilon )$}
			&\multicolumn{1}{c|}{$\mathcal{O}(n^{m_0}/\varepsilon )$}	\\
			\hline
			SGD \cite{ghadimi2016accelerated}
			& \multicolumn{1}{c|}{$\mathcal{O}(1/\varepsilon^2 )$} 
			& \multicolumn{1}{c|}{$\mathcal{O}(1/\varepsilon^2 )$} 
			& \multicolumn{2}{c|}{$\mathcal{O}(n^{m_0}/\varepsilon^2)$}  \\
			\hline
			SCGD \cite{wang2017stochastic}
			&\multicolumn{1}{c|}{$\mathcal{O}(1/\varepsilon^{4})$}
			&\multicolumn{1}{c|}{$\mathcal{O}(1/\varepsilon^{4})$}
			&\multicolumn{2}{c|}{$\mathcal{O}(1/\varepsilon^{4})$}\\
			\hline
			Acc-SCGD\cite{wang2017stochastic}
			&\multicolumn{1}{c|}{$\mathcal{O}(1/\varepsilon^{7/2})$}
			&\multicolumn{1}{c|}{$\mathcal{O}(1/\varepsilon^{7/2})$}
			&\multicolumn{2}{c|}{$\mathcal{O}(1/\varepsilon^{7/2})$} \\
			\hline
			ASC-PG \cite{wang2016accelerating}
			&\multicolumn{1}{c|}{$\mathcal{O}(1/\varepsilon^{9/4})$}
			&\multicolumn{1}{c|}{$\mathcal{O}(1/\varepsilon^{9/4})$}
			&\multicolumn{2}{c|}{$\mathcal{O}(1/\varepsilon^{9/4})$}\\
			\hline
			SVRG \cite{allen2016improved}\cite{reddi2016stochastic} 
			&\multicolumn{1}{c|}{$\mathcal{O}(n^{2/3}/\varepsilon)$}
			&\multicolumn{1}{c|}{$\mathcal{O}(n^{2/3}/\varepsilon)$}
			& {$\mathcal{O}(n^{2/3+m_0/3}/\varepsilon)$} 
			&{$\mathcal{O}(n^{m_0}/\varepsilon)$}   \\
			\hline
			SCVR 
			&\multicolumn{1}{c|}{$\mathcal{O}(n^{4/5}/\varepsilon)$}
			&\multicolumn{1}{c|}{$\mathcal{O}(n^{4/5}/\varepsilon)$}   
			&{$\mathcal{O}(n^{4/5}/\varepsilon)$} 
			&{$\mathcal{O}(n^{4m_0/5}/\varepsilon)$} \\
			\hline
		\end{tabular}%
		\vspace{5pt}
		\caption{Comparison of the complexity with different algorithms for non-convex problem. }
		\label{VRNonCS:Tabel}%
	\end{table}%

	Under the classical benchmark of non-convex optimization \cite{ghadimi2016accelerated}, we aim to propose an {efficient} algorithm to answer the above questions and find an approximate stationary point $x$ satisfying $\| {{(\partial G(x))^\mathsf{T}}\nabla F(G(x))} \|^2 \le \varepsilon $. For fair comparison, we analyze the effectiveness of the algorithm using  query complexity (QC), which is measured in terms of the number of {component function} queries used to compute the gradient. For instance, computing the gradient of $F_i(G(x))$ needs $2m+1$ queries, that is $m$ queries for $G(x)$, $m$ queries for $\partial G(x)$, and $1$ query for $F_i(\cdot)$. Furthermore, QC is related to the iteration and gradient complexities. The iteration complexity is the number of iterations taken to converge to the stationary point. The gradient complexity in \cite{allen2016improved} and \cite{ xiao2014proximal} is measured in terms of the number of gradient evaluations of $f$  including the computation  of inner function $G(x)$. For instance, the iteration complexities of the full gradient descent (GD) method \cite{ghadimi2016accelerated} and SGD \cite{ghadimi2016accelerated} are {$\mathcal{O}(1/\varepsilon )$} and {$\mathcal{O}(1/\varepsilon^2 )$} respectively, while their gradient complexities are {$\mathcal{O}(n/\varepsilon )$} and {$\mathcal{O}(1/\varepsilon^2 )$}. This is because, at each iteration, the full GD method needs to compute the full gradient of $f(x)$. Furthermore,  $\mathcal{O}(m) $ queries are also required to compute $G( x )$ and $\partial G( x )$, respectively. However, these two methods can deal with the composition of two finite-sums problem but will have high query complexity.  
	
	\cite{wang2017stochastic} first proposed the stochastic compositional gradient descent (SCGD) method, which mainly focuses on  the composition of {two infinite-sum} structures problem. Subsequently, \cite{wang2016accelerating}  and \cite{wang2017stochastic} proposed the corresponding accelerated method, Acc-SCGD   and accelerated stochastic compositional proximal gradient (ASC-PG),  for such a problem, and in doing so improved the iteration complexity from $\mathcal{O}(1/\varepsilon^{4})$ to  $\mathcal{O}(1/\varepsilon^{9/4})$. For each iteration, there will be $\mathcal{O}(1)$ queries to compute the gradient of $f$, so the query complexity is the same as the iteration complexity.
	
	However, the convergence rates of these stochastic composition methods are independent  of $n$. Through the cost per iteration of the  stochastic method is faster than full GD, the total number of iterations is large.   \cite{allen2016improved} and \cite{reddi2016stochastic} proposed the SVRG method for solving the one finite-sum non-convex problem, which has better gradient complexity. Although SVRG method has not previously been  applied to the composition problem,  we can obtain the query complexity by adding $\mathcal{O}(m)$ (the same as $O(n^{m_0})$ \footnote{Note that, throughout this paper, we define $m=n^{m_0}$, $m_0\ge 0$  to represent the size of inner sub-function $G_j(x)$ for easy analysis}) queries to each  gradient complexity, that is,
	\begin{center}
		$\text{QC} = \left\{ {\begin{array}{*{20}{l}}
			{\mathcal{O}({n^{\frac{2}{3} + \frac{1}{3}{m_0}}}/\varepsilon ),}&{{m_0} \le 1;}\\
			{\mathcal{O}({n^{{m_0}}}/\varepsilon ),}&{{m_0} > 1}.
			\end{array}} \right.$
	\end{center}
	The analysis of the query complexity of SVRG  can be found as Corollary \ref{VRNonCS:CorollaryComplexityQCS}. Figure \ref{VRNonCS:FigurAlg} clearly shows the different of query complexity between SVRG and SGD.  SVRG improves the query complexity when $m_0\le 1$, but is unchanged when $m_0> 1$. When $m_0>1$, the query complexities of SVRG and SGD are the same since computing the inner function both need $O(n^{m_0})$ query complexity. However, as noted above, can we improve the query complexity and also tackle the difficulty encountered in the composition problem?

	\subsection{Results}
	\begin{wrapfigure}{R}{0.5\textwidth}
		\vspace{-20pt}
		\begin{center}
			\includegraphics[width=0.42\textwidth]{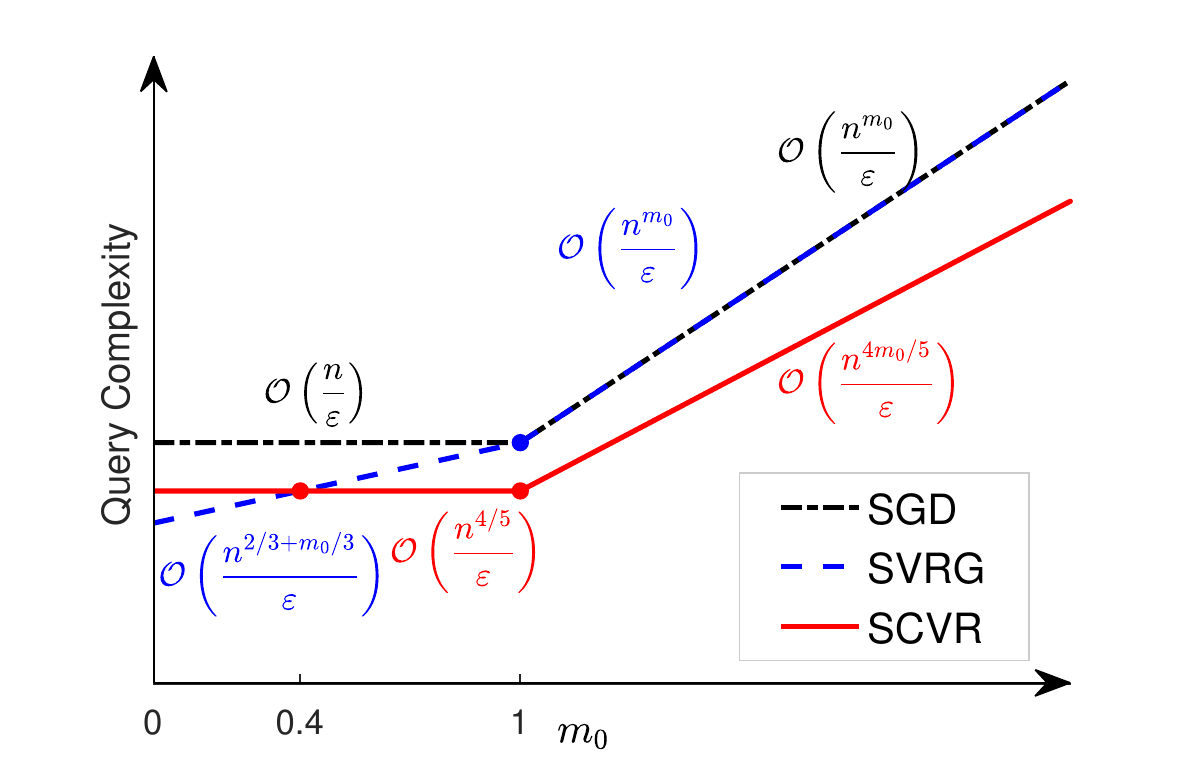}
		\end{center}
		\vspace{-10pt}
		\caption{The QC comparison of SGD, SVRG and SCVR with different sizes of $m_0$.}
		\vspace{-40pt}
		\label{VRNonCS:FigurAlg}
	\end{wrapfigure}
	We discover an interesting phenomenon with respect to the size of inner function $G(x)$ for the non-convex composition problem. In particular, we show that if $m_0>2/5$, our proposed method, stochastic composition with variance reduction (SCVR), improves the query complexity for such non-convex problem,
	\begin{center}
		$QC = \left\{ {\begin{array}{*{20}{l}}
			{\mathcal{O}({n^{4/5}}/\varepsilon) ,}&{0< m_0\le 1};\\
			{\mathcal{O}({n^{4{m_0}/5}}/\varepsilon) ,}&{1<m_0}.
			\end{array}} \right.$
	\end{center}
	In other words, SCVR is faster than SVRG by a factor of $\Omega(n^{m_0/3-2/15})$ if $2/5<m_0<1$, and $\Omega(n^{m_0/5})$ if $1<m_0$. Figure \ref{VRNonCS:FigurAlg} intuitively shows the improvement in query complexity. Furthermore, when $m_0\le 2/5$, we can choose SVRG directly without considering the estimate of the inner function $G(x)$.
	
	\textbf{Mini-batch} We also consider the mini-batch stochastic setting, which is analogous to the mini-batch SVRG. Mini-batch $\mathcal{I}$ is formed by randomly selecting from $[n]$. The query complexity can be improved by a factor of $\Omega ( {{n^{| \mathcal{I} |/5}}} )$ when the size of mini-batch $| \mathcal{I} |\le 2/3$,
	
	\begin{center}
		$QC = \left\{ {\begin{array}{*{20}{l}}
			{O({n^{4/5 - \left| \mathcal{I} \right|/5}}/\varepsilon ),}&{0 < {m_0} \le 1;}\\
			{O({n^{2{m_0}/3 - \left| \mathcal{I} \right|/5}}/\varepsilon ),}&{{m_0} > 1.}
			\end{array}} \right.$
	\end{center}
	When  $| \mathcal{I} |\ge 2/3$, query complexity  becomes
	\begin{center}
		$QC= \left\{ {\begin{array}{*{20}{l}}
			{{\cal O}({n^{2/3}}/\varepsilon )},&{0 < {m_0} \le 1};\\
			{{\cal O}({n^{2{m_0}/3 }}/\varepsilon )},&{{m_0}>1}.
			\end{array}} \right.$
	\end{center}
	\subsection{Our Technique}
	Let us first recall the variance reduction technology proposed in \cite{johnson2013accelerating, allen2016improved, reddi2016stochastic}, and then  answer our first question: can variance reduction technology also be applied to the estimating of such an inner function?
	
	The SVRG algorithms in \cite{ allen2016improved} and \cite{reddi2016stochastic} for the non-convex problem are the same as that in \cite{johnson2013accelerating} for convex problem. That is, dividing the iteration into epochs. At the beginning of each epoch, the full gradient of $f$ will be computed at a snapshot $\tilde{x}$, which is maintained for the current epoch. At each epoch, the unbiased gradient estimator will be used to update the iteration, that is  ${x_{k + 1}} = {x_k} - \eta {\nabla _k}$, where ${\nabla _k} = (\partial G({x_k}))^\mathsf{T}\nabla {F_i}(G({x_k})) - (\partial G(\tilde x))^\mathsf{T}\nabla {F_i}(G(\tilde x)) + \nabla f(\tilde x)$ satisfying $\mathbb{E}[{\nabla _k}] = \nabla f({x_k})$.
	
	However, when the inner function $G(x)$ is non-affine with a large number of sub-functions $G_j(x)$ in problem (\ref{VRNonCS:ProblemMainComposition}), more query complexity is needed by directly using the SVRG-based unbiased estimator method. \cite{wang2016accelerating}, \cite{wang2017stochastic} and \cite{lian2016finite} proposed the biased estimator method for non-convex and convex problems, respectively. Due to fewer queries for each iteration,  the query complexities proposed in \cite{wang2017stochastic}, \cite{wang2016accelerating} and \cite{lian2016finite} are improved. This motivates us to study how to improve the query complexity by using the biased SVRG-based  method for the non-convex problem. 
	
	Based on different estimators of inner function $G(x)$, we propose SCVRI, SCVRII, and mini-batch SCVR:
	\begin{itemize}
		\item In SCVRI, we only consider the estimate of inner function $G(x)$ by using variance reduction technology, denoted $\hat{G}$. Then, we replace the $G(x)$ in ${\nabla _k}$ with the estimator $\hat{G}$ to form the new estimator $\nabla \hat{f}$. Though estimator $\hat{G}$  is unbiased, the provided  $\nabla \hat{f}$ is a biased estimator, that is $\mathbb{E}[\hat G] = G(x)$, while $\mathbb{E}[\nabla \hat{f}]\ne (\partial G( x ))^\mathsf{T}\nabla F( G( x ) )$. However, our theoretical analysis suggests that estimating $G(x)$  through the proper size of sub-function $G_j(x)$ will improve  query complexity. We provide pseudocode in Algorithm \ref{VRNonCS:AlgorithmII} and illustrate the query complexity in Figure \ref{VRNonCS:FigurAlg}.

		\item In SCVRII, besides estimating of $G(x)$, we also estimate the partial gradient $\partial G(x)$ by variance reduction technology, denoted  $\partial \hat G$. We  replace $\partial G(x)$ in $\nabla_k$ with $\partial \hat G$ to form another new estimator $\nabla \tilde{f}$. This estimator is also biased, that is,  $\mathbb{E}[\nabla \tilde{f}]\ne (\partial G( x ))^\mathsf{T}\nabla F( G( x ) )$, while $\mathbb{E}[\partial \hat G] = \partial G(x)$. Even though the SCVRGII method does not increase the order of query complexity, it does increase the convergence rate.  More details can be found to Algorithm \ref{VRNonCS:AlgorithmII}.
		
		\item In mini-batch SCVR, we study the mini-batch version method, which is popular in stochastic optimization. Similar to mini-batch SVRG \cite{reddi2016stochastic}, we also randomly select sub-function $F_i$ from $[n]$ to form mini-batch $\mathcal{I}$, which are all used to estimate the gradient of $f(x)$. Our theoretical analysis suggests that under the  proper size of mini-batch $\mathcal{I}$, there will be an improvement in query complexity. Furthermore, stochastic gradient can also be computed in parallel over mini-bath $\mathcal{I}$, resulting in faster speeds both in theory and practice. We provide pseudocode in Algorithm \ref{VRNonCS:AlgorithmMiniBatch}. 
	\end{itemize}
	
	\subsection{Related work}
	Stochastic non-convex optimization has attracted a lot of attention, not least in machine learning and deep learning. Many first-order methods have recently been proposed. Most of these gradient methods aim to find an approximate stationary point. The theoretical analysis is based on the gradient descent method in \cite{nesterov2013introductory}. For example, the convergence rate of the stochastic gradient method for the non-convex problem in \cite{ghadimi2016accelerated} and \cite{reddi2016stochastic} is framed in terms of the expected gradient norm. Furthermore, accelerated gradient descent \cite{nesterov2013introductory} has also been applied to non-convex stochastic optimization. Although not providing theoretical improvements over current convergence rates, \cite {ghadimi2016accelerated} provided a unified theoretical analysis of the convex and non-convex problem  based on a modified Nesterov's method. This has the same convergence rate as SGD for the non-convex problem and maintained an accelerated convergence rate for the convex problem.  \cite{li2015accelerated} also proposed an accelerated proximal gradient method for the non-convex problem and also retained the accelerated convergence rate for the convex problem. However, when directly applying SGD to composition problem with two finite-sums structure, at each iteration, the above-mentioned method need over $\mathcal{O}(m)$ queries to compute the inner function $G$. 
	
	\cite{wang2016accelerating} and \cite{wang2017stochastic} subsequently proposed the SCGD-based method, focusing on such a structure, the difficulty being in computing of the inner infinite-sum function. To tackle this problem,  \cite{wang2017stochastic} employed the two-timescale quasi-gradient method and Nesterov's method to accelerate the convergence rate. Furthermore,  \cite{wang2016stochastic} also deployed the SCGD based method to consider  corrupted samples with Markov noise.  \cite{dentcheva2016statistical} established a central limit theorem to consider the special composition problem. However, the query complexity does not depend on $n$, motivating us to consider a  more {efficient} algorithm that has a relationship with $n$. 
	
	Recently, variance reduction-based methods have received intensive attention in convex optimization, because the variance reduction technology can improve the convergence rate from sublinear to linear. Two popular methods are SVRG \cite{johnson2013accelerating, xiao2014proximal} and SAGA \cite{defazio2014saga}. The stochastic dual coordinate ascent (SDCA) \cite{shalev2014accelerated,shalev2013stochastic} can also be considered as the variance reduction method.  \cite{lian2016finite} applied SVRG based method to stochastic composition convex problem. For the non-convex problem,  \cite{allen2016improved} and  \cite{ reddi2016stochastic} both proposed the SVRG based method for the non-convex problem and gave the same iteration complexity $\mathcal{ O}(n^{2/3}/\varepsilon )$. Subsequently, \cite{ NIPS2016_6116}  also proposed the SAGA-based proximal stochastic method. Through these methods  haven't applied to stochastic composition problem,  they can be used directly. However, the same problem as in SGD will also be encountered in SVRG method, that is they do not consider the size of the inner subfunction $G_j(x)$ such that more queries will be needed.
	
	There is also another situation in the  non-convex problem. to prevent the point falling into the saddle point, \cite{ge2015escaping}  proposed an SGD with a noise-injected method to escape the saddle point with a running time being a polynomial in the dimension. \cite{lee2016gradient} applied the stable manifold theorem from dynamical system theory to prove that the gradient method with random initialization  converged to local minimum. To escape the saddle point, the second-order method is a better alternative but with expensive computation of the Hessian matrix. Many researchers \cite{carmon2016accelerated, agarwal2016finding} investigate the Hessian-free based  method, such as use accelerated eigenvector computation instead. Although the convergence rate improved, at each iteration, they will need  more computation comparing with SGD, let alone  the case of a composition of two finite-sums structure problem.
	
	\section{Preliminaries}
	Throughout this paper, we use Euclidean norm denoted by $\|\cdot\|$. We use $i \in [ n ]$ and $j \in [ m ]$ to denote that $i$ and $j$ are generated from $[ n ] = \{ {1,2,...,n} \}$, and $[ m ] = \{ {1,2,...,m} \}$. We denote by ${( {\partial G( x )} )^\mathsf{T}}\nabla F( {G( x )} )$ the full gradient of the function $f$,  $\partial G( x )$ the partial gradient of $G$, and  ${( {\partial G_j( x )} )^\mathsf{T}}\nabla F_i( {G( x )} )$ as the stochastic gradient of the function $f$.  
	
	Recall two definitions on Lipschitz function and smooth function.
	\begin{definition} \label{VRNonCS:DefinitionLipschitzFunction}
		A function $p$ is called a Lipschitz function on $ \mathcal{X}$ if there is a constant $B_p$ such that $	\| {p( x ) - p( y )} \| \le {B_p}\| {x - y} \|$, $\forall x,y\in \mathcal{X}$.
	\end{definition}
	\begin{definition}
		A function $p$ is called a $L_p$-smooth function on $ \mathcal{X}$ if there is a constant $L_p$ such that $\| {\nabla p( x ) - \nabla p( y )} \| \le {L_p}\| {x - y} \|$, and equal to $p( y ) \le p( x ) + \langle {\nabla p( x ),y - x} \rangle  + {L_p}/2{\| {y - x} \|^2}$, $\forall x,y\in \mathcal{X}$.
	\end{definition}
	We make the following assumptions to discuss the convergence rate and complexity analysis.
	\begin{assumption}\label{VRNonCS:AssumptionG}
		For function $G_j$: ${\mathbb{R}^N} \to {\mathbb{R}^M}$, all $j\in [m]$,  
		\begin{itemize}
			\item $G_j$ has the bounded Jacobian with a constant ${B_G}$, that is $\| {\partial G_j( x )} \| \le {B_G}$, $\forall  x \in {\mathbb{R}^N}$,  then  $ G_j( x )$ is also a Lipschitz function that satisfying $	\| {G_j( x ) - G_j( y )} \| \le {B_G}\| {x - y} \|$,  $\forall  x,y \in {\mathbb{R}^N}$.
			\item $G_j$ is $L_G$-smooth satisfying $	\| {\partial G_j( x ) - \partial G_j( y )} \| \le {L_G}\| {x - y} \|$,  $\forall  x,y \in {\mathbb{R}^N}$.
		\end{itemize}
	\end{assumption}
	
	\begin{assumption}\label{VRNonCS:AssumptionF}
		For function $F_i$: ${\mathbb{R}^M} \to {\mathbb{R}}$, all $i\in [n]$,  
		\begin{itemize}
			\item $F_i$ has the bounded gradient with a constant ${B_F}$,  that is $\|{\nabla F_i( y )} \| \le {B_F}$, $\forall  y \in {\mathbb{R}^M}$.
			\item  $F_i$ is $L_F$-smooth satisfying $	\| {\nabla F_i( x ) - \nabla  F_i( y )} \| \le {L_F}\| {x - y} \|$,  $\forall  x,y \in {\mathbb{R}^M}$.
		\end{itemize}
	\end{assumption}
	\begin{assumption} \label{VRNonCS:AssumptionGF}
		For function $F_i(G(x))$: ${\mathbb{R}^N} \to {\mathbb{R}}$, all $i\in [n]$,  there exist a constant $L_f$ satisfying 	
		\begin{align}\label{VRNonCS:AssumptionInequality}
		\| {{( {\partial {G_j}( x )} )^\mathsf{T}}\nabla {F_i}( {G( x )} ) - {( {\partial {G_j}( y )} )^\mathsf{T}}\nabla {F_i}( {G( x )} )} \| \le {L_f}\| {x - y} \|,\forall j \in [m], \forall  x,y \in {\mathbb{R}^N}.
		\end{align}		
	\end{assumption}
	Furthermore, if Assumption \ref{VRNonCS:AssumptionGF} holds, then $f(x)$ is  $L_f$-smooth function due to the fact that 
	\begin{align*}
	\mathbb{E}[ {{{\| {\nabla f( x ) - \nabla f( x )} \|^2}}} ]\le& \frac{1}{n}\sum\limits_{i = 1}^n {\frac{1}{m}\sum\limits_{j = 1}^m {{{\| {{(\partial {G_j}(x))^\mathsf{T}}\nabla {F_i}(G(x)) - {(\partial {G_j}(y))^\mathsf{T}}\nabla {F_i}(G(x))} \|^2}}} } \\
	\le& {L_f}{\| {x - y} \|^2}.
	\end{align*} 
	\begin{assumption} \label{VRNonCS:AssumptionIndependent}
		We assume that $i$ and $j$ are independently and randomly selected from $[n]$ and $[m]$, that is 
		\begin{align*}
		\mathbb{E}[ {{( {\partial {G_j}( x )} )^\mathsf{T}}\nabla {F_j}( {G( x )} )} ] = {( {\partial G( x )} )^\mathsf{T}}\nabla F( {G( x )} ).
		\end{align*} 	
	\end{assumption}
	In the paper, we denote by $x_k^s$ the $k$-th inner iteration at $s$-th epoch. But in each epoch analysis, we drop the superscript $s$ and denote by $x_k$ for $x_k^s$ . We let $x^*$ be the optimal solution of $f(x)$. Throughout the convergence analysis, we use  $\mathcal{O}( \cdot )$ notation to avoid many constants, such as $B_F$, $B_G$, $L_F$, $L_G$ and $L_f$,... that are irrelevant with the convergence rate  and provide insights to analyze the iteration and query complexity.
	
	\section{Variance reduction method I for non-convex composition problem}
	
	We now apply SVRG method for non-convex composition problem, which is used for convex composition problem in \cite{lian2016finite}. We use variance reduction method for estimating the inner function $G(x)$ and the gradient of $f(x)$,  and exploit the benefit of non-convex composition problem, referred as SCVRI.  Algorithm \ref{VRNonCS:AlgorithmI} presents SCVRI's pseudocode. 
	
	Consider the inner function, we estimate $G(x)$ through variance reduction technology at $k$-th iteration of $s$-th epoch,
	\begin{align}\label{VRNonCS:DefinitionHatG}
	{{\hat G}_k}= \frac{1}{A}\sum\limits_{1 \le j \le A}^{} {( {{G_{{\mathcal{A}_k}[j]}}( {{x_k}} ) - {G_{{\mathcal{A}_k}[j]}}( {{{\tilde x}_s}} )} )}  + G( {{{\tilde x}_s}} ),
	\end{align}
	where $\mathcal{A}_k$ is the mini-batch formed by randomly sampling from $[m]$ with $A$ times. Furthermore, we can see that $\mathbb{E}[ {{{\hat G}_k}} ] = G( x_k )$. Based on the estimated inner function ${{\hat G}_k}$, the stochastic gradient of $f$ can be obtained through variance reduction technology,
	\begin{align}\label{VRNonCS:DefinitionPartialHatf}
	\nabla {{\hat f}_k} &= {( {\partial {G_{{j_k}}}( {{x_k}} )} )^\mathsf{T}}\nabla {F_{{i_k}}}( {{{\hat G}_k}} ) - {( {\partial {G_{{j_k}}}( {{{\tilde x}_s}} )} )^\mathsf{T}}\nabla {F_{{i_k}}}( {G( {{{\tilde x}_s}} )} ) + \nabla f( {{{\tilde x}_s}} ),
	\end{align}
	where $\mathbb{E}[ {\nabla {\hat f_k}} ] = {( {\partial G( {{x_k}} )} )^\mathsf{T}}\nabla F( {{{\hat G}_k}} )$ is based on the Assumption \ref{VRNonCS:AssumptionIndependent}. However, since the inner function is estimated, the expectation of  $\nabla {{\hat f}_k}$ with respect to $i_k$ and $j_k$ is not equal to the full gradient, that is $\mathbb{E}[\nabla \hat f_k] \ne (\partial G({x_k}))^\mathsf{T}\nabla F(G({x_k}))$. In the following subsection, we give the upper bounds for the unbiased estimation of inner function $G(x)$ and biased estimation of the gradient of full function $f(x)$, which are used for analyzing the convergence of non-convex problem. Furthermore, we also give the convergence analysis and query complexity. The proof details can be found in Section \ref{VRNonCS:SectionBoundAnalysis} and \ref{VRNonCS:SectionConvergenceAnalysis}.
	
	\begin{algorithm}[h]
		\caption{Stochastic composition variance reduction for Non-convex Composition I}
		\label{VRNonCS:AlgorithmI}
		\begin{algorithmic}
			\Require $K$, $S$, $\eta$ (learning rate), and $\tilde{x}_0$
			\For{$s =0,2,\cdots,S-1$}
			\State $	G( {{{\tilde x}_s}} ) = \frac{1}{m}\sum\limits_{j = 1}^m {{G_j}( {{{\tilde x}_s}} )}$\Comment{m Queries}
			\State $	\partial G( {{{\tilde x}_s}} ) = \frac{1}{m}\sum\limits_{j = 1}^m {{\partial G_j}( {{{\tilde x}_s}} )}$\Comment{m Queries}
			\State $\nabla f( {{{\tilde x}_s}} ) =( \partial G( {{{\tilde x}_s}} ))^\mathsf{T}\frac{1}{n}\sum\limits_{i = 1}^n {\nabla{F_i}( {G( {{{\tilde x}_s}} )} )} $\Comment{n Queries}
			\State $x_0=\tilde{x}_s$
			\For{$k =0,1,2,\cdots,K-1$}
			\State
			Sample from $[m]$ for A times to form mini-batch $\mathcal{A}_k$ with replacement
			\State
			${{\hat G}_k} = \frac{1}{A}\sum\limits_{1 \le j \le A}^{} {( {{G_{{\mathcal{A}_k}[j]}}( {{x_k}} ) - {G_{{\mathcal{A}_k}[j]}}( {{{\tilde x}_s}} )} )}  + G( {{{\tilde x}_s}} )$\Comment{2A Queries}
			\State Uniformly and randomly pick $i_k$ and $j_k$ from $[n]$ and $[m]$ 
			\State
			$\nabla {{\hat f}_k}= {( {\partial {G_{{j_k}}}( {{x_k}} )} )^\mathsf{T}}\nabla {F_{{i_k}}}( {{{\hat G}_k}} ) - {( {\partial {G_{{j_k}}}( {{{\tilde x}_s}} )} )^\mathsf{T}}\nabla {F_{{i_k}}}( {G( {{{\tilde x}_s}} )} ) + \nabla f( {{{\tilde x}_s}} )$\Comment{4 Queries}
			\State
			${x_{k+1}}= {x_k} - {\eta} \nabla {{\hat f}_k} $
			\EndFor
			\State Update $\tilde{x}_{s+1}=x_K $
			\EndFor \\	
			\textbf{Output:}  $\tilde x_k^s$ is uniformly and randomly chosen from  $s=\{0,...,S-1\}$ and k=$\{0,..,K-1\}$.
		\end{algorithmic}
	\end{algorithm}
	\subsection{Upper bound of estimator function}
	The following lemmas give the upper bounds of the estimated inner function ${\hat G}_k$ and gradient estimator $\nabla \hat f_k$.
	\begin{lemma}\label{VRNonCS:LemmaBoundVarianceG}
		Suppose Assumption \ref{VRNonCS:AssumptionG} holds, for ${\hat G}_k$ defined in (\ref{VRNonCS:DefinitionHatG}), we have the upper bound
		\begin{align*}
		\mathbb{E}[{\| { {{\hat G}_k} -  G({{\tilde x}_s})} \|^2}]\le B_G^2\frac{1}{A}\mathbb{E}[{\| {{x_k} - {{\tilde x}_s}} \|^2}].
		\end{align*}
	\end{lemma}
	
	\begin{lemma}\label{VRNonCS:LemmaBoundVarianceEstimatGradient}
		Suppose Assumption \ref{VRNonCS:AssumptionG}-\ref{VRNonCS:AssumptionGF} hold, the estimated $ \nabla {\hat f_k}$ defined in (\ref{VRNonCS:DefinitionPartialHatf}) can be bounded by
		\begin{align*}
		\mathbb{E}[ {\| { \nabla {\hat f_k}} \|^2} ] \le 4\mathbb{E}[ {\| {\nabla f( {{x_k}} )} \|^2} ] + 4\left( {2L_f^2 + {B_G^4L_F^2}\frac{1}{A}} \right)\mathbb{E}[ {\| {{x_k} - \tilde x_s} \|^2} ].
		\end{align*}
	\end{lemma}
	As can be seen from the above lemmas, when the sample times $A$ increase, the estimated ${\hat G}_k$ can be well approximating to the real inner function $G$. Furthermore, the bound of the gradient estimator ${ \nabla {\hat f_k}}$ is tighter. As $x_k$ approach to the stationary point, both $\mathbb{E}[ {\| {\nabla f( {{x_k}} )} \|^2} ]$ and $\mathbb{E}[ {\| {{x_k} - \tilde x_s} \|^2}$ are approximating to zero such that ${ \nabla {\hat f_k}}$ will approximate to zero.   Based on these basic lemmas, we will  analyze if we can obtain and how to choose a proper size of sample times $A$ such that can reach the best query complexity in the large-scale data.
	\subsection{Convergence analysis}
	In this subsection, we first give the convergence rate for the composition with two finite-sums functions, which is not related to $n$. Then we consider the convergence rate that has a relationship with $n$ through three different kinds of  mini-batch $\mathcal{A}$: Corollary \ref{VRNonCS:CorollaryComplexityEstimation} gives the convergence rate with the mini-batch $\mathcal{A}$ formed by randomly selecting from $[n]$ with $A$ times; Corollary \ref{VRNonCS:CorollaryComplexityEstimationM} 's  mini-batch $\mathcal{A}$ is the inner function $G(x)$ itself; Corollary \ref{VRNonCS:CorollaryComplexityEstimationinfinit}'s mini-batch $\mathcal{A}$ is formed by infinite sampling from $[n]$ with sample times $A=+\infty$.
	\begin{theorem}\label{VRNonCS:TheoremMain}
		For the algorithm \ref{VRNonCS:AlgorithmI}, Let $h,d, \eta > 0$ such that
		\begin{align}\label{VRNonCS:DefinitionUk}
		{u_k} = ( {1/2- {c_{k + 1}}h} )\eta  - ( {2{L_f} + 4{c_{k + 1}}} ){\eta ^2},\forall k \ge 0,
		\end{align}
		where
		\begin{align}\label{VRNonCS:DefinitionCk}
		{c_k} =&{c_{k + 1}}\left( {1 + \left( {\frac{1}{h} + \frac{1}{d} + \frac{dB_G^4L_F^2}{A}} \right)\eta  + 4\left( {2L_f^2 + \frac{B_G^4L_F^2}{A}} \right){\eta ^2}} \right)\nonumber\\&+ {\frac{B_G^4L_F^2}{{2A}}\eta	 + 2{L_f}\left( {2L_f^2 + \frac{B_G^4L_F^2}{A}} \right){\eta ^2}},
		\end{align}
		$B_G$, $L_f$, and $L_F$ are parameters defined in Assumption \ref{VRNonCS:AssumptionG}-\ref{VRNonCS:AssumptionGF}, and A is the sample times for forming the mini-batch $\mathcal{A}_k$. Let $K$ be the number of inner iteration, $S$ be the  number of inner iteration,  and define $u$ to be ${\min _{0 \le k \le K - 1}}\{ {u_k}\} $, we have
		\begin{align*}
		\mathbb{E}[{\| {\nabla f({\tilde x_k^s })} \|^2}] \le \frac{{f({x_0}) - f({x^*})}}{{uKS}},
		\end{align*}
		where $\tilde x_k^s$ is uniformly and randomly chosen from  $s=\{0,...,S-1\}$ and k=$\{0,..,K-1\}$.		
	\end{theorem}
	
	\begin{remark}
		The above theorem gives the convergence of the proposed algorithm, however, parameters, such as $h,d,\eta$, are not clearly defined. Furthermore, the convergence rate is independent of $n$. In the following corollaries, we give an analysis to choose the best parameters such that obtain the best query complexity. Moreover, the method for choosing the parameter is based on \cite{reddi2016stochastic}, however, we give more exact and clear explanation, and extend to different kinds of situations.
	\end{remark}
	
	\begin{corollary}\label{VRNonCS:CorollaryComplexityEstimation}
		In Algorithm \ref{VRNonCS:AlgorithmI}, let $\eta  =n^{-\alpha}/(1L_f(2L_f^2+B^4_GL^2_F/A)) ,d = {n^{{d_0}}},h =n^{h_0}/(e-1) $,
		where $1\ge\alpha,h_0,d_0>0$.  $K$ is the number of inner iteration, $S$ is the  number of inner iteration, $A = B_G^4L_F^2{n^{{A_0}}}/2$ is the sample time for mini-batch $\mathcal{A}_k$, $A_0>0$. There exist two constant $v_1>0$ and $w_1>0$ such that $K = w_1L_f^3{n^{3\alpha /2}}$ and $u = {n^{ - \alpha }}v_1/L_f^3$. The output $\tilde x_k^s$ satisfies
		\begin{align*}
		\mathbb{E}[{\| {\nabla f({\tilde x_k^s })} \|^2}]  \le \frac{{{n^\alpha }L_f^3(f({x_0}) - f({x^*}))}}{{v_1SK}}.
		\end{align*}
	\end{corollary}
	
	\begin{remark}
		Corollary \ref{VRNonCS:CorollaryComplexityEstimation} shows that the convergence rate depends on $n$. Based on the proof in Corollary \ref{VRNonCS:CorollaryComplexityEstimation}, we can also obtain the following corollary with more clear and simple explanation  for the case that the mini-batch $\mathcal{A}$ is inner function $G(x)$ itself. Note that the iteration complexity is the same as in Corollary \ref{VRNonCS:CorollaryComplexityEstimation}.  
	\end{remark}
	
	\begin{corollary}\label{VRNonCS:CorollaryComplexityEstimationM}
		\cite{reddi2016stochastic}
		If the mini-batch $\mathcal{A}_k$ is formed by the non-repeat samples, which has the size $A=m$, that is ${\hat G_k}  = G( {{x_k}} ),\forall k \in [ K ]$. Let $\eta  = {{{n^{ - \alpha }}} \mathord{/{\vphantom {{{n^{ - \alpha }}} {4L_f^3}}}\kern-\nulldelimiterspace} {4L_f^3}},h = {{{n^{{h_0}}}} \mathord{/{\vphantom {{{n^{{h_0}}}} {( {e - 1} )}}}\kern-\nulldelimiterspace} {( {e - 1} )}}$, $h_0, \alpha>0$, there exist two constant $w_2$, $v_2>0$ such that $K = {w_2}L_f^3{n^{3\alpha /2}},u = {n^{ - \alpha }}{v_2}/(4L^3_f)$. The output $\tilde x_k^s$ satisfies
		\begin{align*}
		\mathbb{E}[{\| {\nabla f({\tilde x_k^s })} \|^2}]   \le \frac{{4{n^\alpha }L_f^3(f({x_0}) - f({x^*}))}}{{v_2SK}}.
		\end{align*}
	\end{corollary}
	Now, we consider the case that the sample times $A$ is positive infinity such that ${B_G^4L_F^2}/A$ can approximate to 0, in other words,  the function $G(x)$ can be considered as fully estimated, $\hat G(x) \approx \frac{1}{m}\sum\nolimits_{j = 1}^m {{G_i}( x )} $. Then, we wonder whether the iteration complexity increase or equal to iteration complexity in \cite{reddi2016stochastic}. If the iteration complexity does not change, or equal to iteration complexity in \cite{reddi2016stochastic},	how to choose the best sample times $A$ to get the better query complexity. We first give the following Corollary to verify the iteration complexity.
	\begin{corollary}\label{VRNonCS:CorollaryComplexityEstimationinfinit}
		Consider  the sample times $A=+\infty $, let $h = {{{n^{{h_0}}}} \mathord{/{\vphantom {{{n^{{h_0}}}} {( {e - 1} )}}} \kern-\nulldelimiterspace} {( {e - 1} )}},d = {n^{{d_0}}},\eta  = {{{n^{ - \alpha }}} \mathord{/{\vphantom {{{n^{ - \alpha }}} {4L_f^3}}} 	\kern-\nulldelimiterspace} {4L_f^3}}$, where $1\ge\alpha,h_0,d_0>0$. There exist two constants $w_3,v_3>0$ such that $	K = w_3{n^{ - 3\alpha /2}},u = v_3{n^{ - \alpha }}.$
		Thus, the output $\tilde x_k^s$ satisfies
		\begin{align*}
		\mathbb{E}[{\| {\nabla f({\tilde x_k^s})} \|^2}] \le \frac{{{4{n^\alpha }L_f^3} ( { {f( {{{ x}_0}} )}  - f( {{x^*}} )} )}}{{v_3SK}}.
		\end{align*}
	\end{corollary}
	
	\begin{remark}
		As shown in Corollary \ref{VRNonCS:CorollaryComplexityEstimation}, Corollary \ref{VRNonCS:CorollaryComplexityEstimationinfinit} and  Corollary \ref{VRNonCS:CorollaryComplexityEstimationM}, in order to keep the output point $ {\nabla f({\tilde x_k^s })}$ satisfying $	E[{\| {\nabla f({\tilde x_k^s })} \|^2}] \le \varepsilon $, the total number of iterations are
		\begin{center}
			$\mathcal{O}\left( {\frac{{4{n^\alpha }L_f^3(f({x_0}) - f({x^*}))}}{{{v_1}\varepsilon }}} \right)$, $\mathcal{O}\left( {\frac{{4{n^\alpha }L_f^3(f({x_0}) - f({x^*}))}}{{{v_2}\varepsilon }}} \right)$ and $\mathcal{O}\left( {\frac{{4{n^\alpha }L_f^3(f({x_0}) - f({x^*}))}}{{{v_3}\varepsilon }}} \right)$
		\end{center}
		with the same order of $O(n^\alpha/\varepsilon)$. However, the query complexities are different. Because methods for computing the inner function $G$ are different such that result in the different query complexities. We can image two difference  extremity cases that the sizes of inner subfunction $G_j$ are one and positive infinity. Actually, the iteration complexity in \cite{reddi2016stochastic} corresponds to the first case. However, does it also fit  the second case, which will leave for the next complexity analysis.
	\end{remark}
	\subsection{Query Complexity analysis}
	In this subsection, we compute the query complexity for two cases: the mini-batch  $\mathcal{A}_k$ is formed by randomly sampling  from $[n]$ with $A$ times, and the mini-batch  $\mathcal{A}_k$ is $G(x)$ itself with size $A=m$.  We analyze these two cases and decide whether there is a better mini-batch  $\mathcal{A}_k$ that has the best query complexity.
	\begin{corollary}\label{VRNonCS:CorollaryQueryComplexity}
		Let $T$ is the total number of iteration, $K$ is the number of inner iteration, $S$ is the number of outer iteration, and $A$ is the sample times for forming a mini-batch  $\mathcal{A}_k$. To achieve a fixed solution accuracy $\varepsilon>0$, that is $\mathbb{E}[{\| {\nabla f({\tilde x_k^s })} \|^2}] \le \varepsilon $, the query complexity is $
		\mathcal{ O}( {( {m + n + {n^{{5\alpha}/{2}  }} } )( {{{{n^{ - {\alpha }/{2}}}}}/{\varepsilon }} )} ).
		$
	\end{corollary}
	
	\begin{corollary}\label{VRNonCS:CorollaryComplexityQCCS} For the case that  the mini-batch  $\mathcal{A}_k$ is formed by randomly sampling  with $A$ times, let the size of inner sub-function $G_j(x)$, $j\in[m]$  is $m=n^{m_0}$, $m_0>0$. The query complexity  of composition stochastic (QCCS) is 
		\begin{align*}
		\text{QCCS} = \left\{ {\begin{array}{*{20}{l}}
			{\mathcal{O}({n^{4/5}}/\varepsilon) ,}&{\alpha  = 2/5,}&{0< m_0\le 1};\\
			{\mathcal{O}({n^{4{m_0}/5}}/\varepsilon) ,}&{\alpha  = 2{m_0}/5,}&{1<m_0}.
			\end{array}} \right.
		\end{align*}
	\end{corollary}

	\begin{corollary}\label{VRNonCS:CorollaryComplexityQCS} For the case that  the mini-batch  $\mathcal{A}_k$ is function $G(x)$ itself  with $A=m$ sub-function $G_j(x)$, $j\in[m]$, let the size of $G_i$ function is $m=A=n^{m_0}=n^{A_0}$, $m_0\ge 0$. The query complexity  of  stochastic (QCS) is 
		\begin{align*}
		\text{QCS} = \left\{ {\begin{array}{*{20}{l}}
			{\mathcal{O}({n^{\frac{2}{3} + \frac{1}{3}{m_0}}}/\varepsilon ),}&{\alpha  = \frac{{2\left( {1 - {m_0}} \right)}}{3},}&{{m_0} \le 1;}\\
			{\mathcal{O}({n^{{m_0} }}/\varepsilon ),}&{\alpha  = 0,}&{{m_0} > 1.}
			\end{array}} \right.
		\end{align*}
	\end{corollary}
	
	\begin{remark}
		We use QCS to indicate that the inner function is fully computed without estimation. This  stochastic optimization process can be considered as dealing with  general empirical minimization problem with one finite-sum structure. When $m_0=0$, that is $m=n^{m_0}=1$, the problem turns into the general empirical problem, and the complexity result coincides with \cite{reddi2016stochastic} and \cite{allen2016improved}. Here note that parameters setting is different in (\ref{VRNonCS:DefinitionParameters}), that is we do not require $A_0=\alpha$. Because there is no estimation computation for the inner function such that there is no term include $A$. The detailed proof for this kind of condition can be referred to \cite{reddi2016stochastic}.
	\end{remark}
	\begin{remark} Based on Corollary \ref{VRNonCS:CorollaryComplexityQCCS} and \ref{VRNonCS:CorollaryComplexityQCS},
		we can obtain a better query  complexity (QC)\footnote{We use QC indicates the query complexity including classical stochastic optimization (QCS) and composition stochastic optimization (QCSC).} through analyzing the different range of $m_0$.
		\begin{itemize}
			\item $m_0\le1$: setting $n^{4/5}=n^{2/3+m_0/3}$, we have $m_0=2/5$, then, we obtain
			\begin{align*}
			\text{QC} = \left\{ {\begin{array}{*{20}{l}}
				{{\cal O}({n^{\frac{2}{3} + \frac{1}{3}{m_0}}}/\varepsilon ),}&{{m_0} < 2/5,}&{QCS};\\
				{{\cal O}({n^{4/5}}/\varepsilon ),}&{2/5 \le {m_0} \le 1,}&{\text{QCCS}}.
				\end{array}} \right.
			\end{align*}
			\item $m_0>1$: we obtain the $\text{QC} ={\cal O}({n^{4{m_0}/5}}/\varepsilon ),\text{QCCS}$. 
		\end{itemize}
		All in all, we can obtain when $m_0\ge2/5$, QCSC is better than that of QCS.
	\end{remark}
	From above description, we can see that when $m_0\le 2/5$, we can compute the full inner function of $G(x)=\frac{1}{m}\sum\nolimits_{j = 1}^m {{G_j}(x)} $ directly rather than the estimated ${\hat G}$. This means that the inner function is no longer  suitable to be estimated;  when $m_0> 2/5$, we can estimate the inner function  through forming mini-batch $\mathcal{A}$ with $A$ time samplings. This estimation can reduce the query complexity when facing large-scale data.

	
	\section{Variance reduction method II for non-convex composition problem}
	We now turn to the extended method used in SVRG for convex composition problem in \cite{lian2016finite}. We use variance reduction method for estimating the partial gradient of $G(x)$ and exploit the benefit of non-convex composition problem, referred as SCVRII.	Algorithm \ref{VRNonCS:AlgorithmII} presents SCVRII's pseudocode.
	\begin{algorithm}[h]
		\caption{Stochastic composition variance reduction for Non-convex Composition II}
		\label{VRNonCS:AlgorithmII}
		\begin{algorithmic}
			\Require $K$, $S$, $\eta$ (learning rate), and $\tilde{x}_1$
			\For{$s =0,2,\cdots,S-1$}
			\State $	G( {{{\tilde x}_s}} ) = \frac{1}{m}\sum\limits_{j = 1}^m {{G_j}( {{{\tilde x}_s}} )}$\Comment{m queries}
			\State $	\partial G( {{{\tilde x}_s}} ) = \frac{1}{m}\sum\limits_{j = 1}^m {{\partial G_j}( {{{\tilde x}_s}} )}$\Comment{m queries}
			\State $\nabla f( {{{\tilde x}_s}} ) =( \partial G( {{{\tilde x}_s}} ))^\mathsf{T}\frac{1}{n}\sum\limits_{i = 1}^n {\nabla{F_i}( {G( {{{\tilde x}_s}} )} )} $\Comment{n queries}
			\State $x_1=\tilde{x}_s$
			\For{$k =0,2,3,\cdots,K-1$}
			\State
			sample from $[m]$ for A times to form mini-batch multiset $\mathcal{A}_k$
			\State 
			sample from $[m]$ for B times to form mini-batch multiset $\mathcal{B}_k$
			\State
			${{\hat G}_k} = \frac{1}{A}\sum\limits_{1 \le j \le A}^{} {\left( {{G_{{A_k}[j]}}\left( {{x_k}} \right) - {G_{{A_k}[j]}}\left( {{{\tilde x}_s}} \right)} \right)}  + G\left( {{{\tilde x}_s}} \right)$\Comment{2A queries}
			\State
			$\partial\hat G_k = \frac{1}{B}\sum\limits_{1 \le j \le B}^{} {( {{\partial G_{{\mathcal{B}_k}[j]}}( {{x_k}} ) - {\partial G_{{\mathcal{B}_k}[j]}}( {{{\tilde x}_s}} )} )}  + \partial G( {{{\tilde x}_s}} )$\Comment{2B queries}
			\State Randomly pick $i_k$ from $[n]$ and 
			\State
			$\nabla {{\tilde f}_k} = {(\partial\hat G_k )^\mathsf{T}}\nabla {F_{{i_k}}}( {{{\hat G}_k}} ) - {( \partial G(\tilde x_s) )^\mathsf{T}}\nabla {F_{{i_k}}}( {G( {{{\tilde x}_s}} )} ) + \nabla f( {{{\tilde x}_s}} )$\Comment{2 queries}
			\State
			${x_{k+1}}= {x_k} - {\eta} \nabla {{\tilde f}_k}$ 
			\EndFor
			\State Update $\tilde{x}_{s+1}=x_K $
			\EndFor \\	
			\textbf{Output:}  ${\tilde x_k^s }$ is uniformly and randomly chosen from  $s=\{0,...,S-1\}$ and k=$\{0,..,K-1\}$
		\end{algorithmic}
	\end{algorithm}
	
	Besides the estimation of inner function, we also estimate the partial gradient of inner function through variance reduction technology at $k$-th iteration of $s$-th epoch,
	\begin{align}\label{VRNonCS:DefinitionPartialHatG}
	\partial\hat G_k = \frac{1}{B}\sum\limits_{1 \le j \le B}^{} {( {{\partial G_{{\mathcal{B}_k}[j]}}( {{x_k}} ) - {\partial G_{{\mathcal{B}_k}[j]}}( {{{\tilde x}_s}} )} )}  + \partial G( {{{\tilde x}_s}} ),
	\end{align}
	where $\mathcal{B}_k$ is the mini-batch formed by randomly sampling from $[m]$ with $B$ times. Furthermore, we can see that $\mathbb{E}[ {{\partial{\hat G}_k}} ] =\partial G( x_k )$. Based on the estimated partial gradient inner function ${\partial{\hat G}_k}$, the stochastic gradient of $f$ can also be obtained,
	\begin{align}\label{VRNonCS:DefinitionPartialTildef}
	\nabla {{\tilde f}_k} &= {( \partial\hat G_k)^\mathsf{T}}\nabla {F_{{i_k}}}( {{{\hat G}_k}} ) - {( \partial\hat G_k )^\mathsf{T}}\nabla {F_{{i_k}}}( {G( {{{\tilde x}_s}} )} ) + \nabla f( {{{\tilde x}_s}} ),
	\end{align}
	where $\mathbb{E}[ {\nabla {\tilde f_k}} ] = {( {\partial G( {{x_k}} )} )^\mathsf{T}}\nabla F( {{{\hat G}_k}} )$.  Though $\mathbb{E}[\nabla \tilde f_k] \ne (\partial G({x_k}))^\mathsf{T}\nabla F(G({x_k}))$ is biased estimator, we also give the upper bound  of the unbiased estimated partial gradient of the inner function $G(x)$ and the biased estimation of the gradient of function $f(x)$, which are used for analyzing the convergence of non-convex function.  The following lemmas show the bound with respect to the estimated partial gradient of $\partial G(x)$ and estimated gradient of $f(x)$, which are more intuitive by the upper bound. Furthermore, SCVRII's convergence analysis and query complexity are provided in the subsection. The proof details can also be found in Section \ref{VRNonCS:SectionBoundAnalysis} and \ref{VRNonCS:SectionConvergenceAnalysis}.
	
	\begin{lemma}\label{VRNonCS:LemmaBoundVariancePartialG}
		Suppose Assumption \ref{VRNonCS:AssumptionG} holds, for ${\partial\hat G}_k$ defined in (\ref{VRNonCS:DefinitionPartialHatG}), we have the upper bound
		\begin{align*}
		\mathbb{E}[{\| {\partial {{\hat G}_k} - \partial G({{\tilde x}_s})} \|^2}]\le L_G^2\frac{1}{B}\mathbb{E}[{\| {{x_k} - {{\tilde x}_s}} \|^2}].
		\end{align*}
	\end{lemma}
	For the case that $b=1$ from Lemma \ref{VRNonCS:LemmaBoundVarianceEstimatGradientfullMiniBatch}, we can obtain the following lemma,
	\begin{lemma}\label{VRNonCS:LemmaBoundVarianceEstimatGradientfull}
		Suppose Assumption \ref{VRNonCS:AssumptionG}-\ref{VRNonCS:AssumptionGF} hold, the estimated $ \nabla {\hat f_k}$ defined in (\ref{VRNonCS:DefinitionPartialTildef}) can be bounded by
		\begin{align*}
		\mathbb{E}[ {\| { \nabla {\tilde f_k}} \|^2} ] \le 4\mathbb{E}[{\| {\nabla f({x_k})} \|^2}] + 4\left( {B_G^4L_F^2\frac{1}{A} + B_F^2L_G^2\frac{1}{B} + L_f^2} \right)\mathbb{E}[{\| {{x_k} - {{\tilde x}_s}} \|^2}].
		\end{align*}
	\end{lemma}
	As can be seen from above lemmas, when $B$ increase, the estimated partial gradient  $\partial \hat G$ is more approximating to the $\partial G$. Furthermore, the upper bound of ${ \nabla {\tilde f_k}}$ is tighter.
	\subsection{convergence analysis}
	In this subsection, we  give the convergence and complexity analysis for SCVRII, which are similar to SCVRI. Based on the convergence rate from Theorem \ref{VRNonCS:TheoremMainFull}, we obtain an alternative convergence rate that is dependent on $n$ and its corresponding query complexity.
	\begin{theorem}\label{VRNonCS:TheoremMainFull}
		For the algorithm \ref{VRNonCS:AlgorithmII}, Let $h,d, \eta > 0$ such that
		\begin{align}\label{VRNonCS:DefinitionUkFull}
		{u_k} = ( {1/2- {c_{k + 1}}h} )\eta  - ( {2{L_f} + 4{c_{k + 1}}} ){\eta ^2},\forall k \ge 0,
		\end{align}
		where
		\begin{align}\label{VRNonCS:DefinitionCkfull}
		{c_k} =&{c_{k + 1}}\left( {1 + \left( {\frac{1}{h} + \frac{1}{d} + dB_G^4L_F^2\frac{1}{A}} \right)\eta  + 4\left( {B_G^4L_F^2\frac{1}{A} + B_F^2L_G^2\frac{1}{B} + L_f^2} \right){\eta ^2}} \right)\nonumber\\
		&+ B_G^4L_F^2\frac{1}{{2A}}\eta  + 2{L_f}\left( {B_G^4L_F^2\frac{1}{A} + B_F^2L_G^2\frac{1}{B} + L_f^2} \right){\eta ^2},
		\end{align}
		$B_G$, $L_f$, and $L_F$ are parameters defined in Assumption \ref{VRNonCS:AssumptionG}-\ref{VRNonCS:AssumptionGF}, and A is the sample times for forming the mini-batch $\mathcal{A}_k$, $B$ is the sample times for mini-batch $\mathcal{B}_k$. Let $K$ is the number of inner iteration, $S$ is the  number of inner iteration,  and define $u = {\min _{0 \le k \le K - 1}}\{ {u_k}\} $, we have
		\begin{align*}
		\mathbb{E}[{\| {\nabla f({\tilde x_k^s })} \|^2}] \le \frac{{f({x_0}) - f({x^*})}}{{uKS}},
		\end{align*}
		where $\tilde x_k^s$ is uniformly and randomly chosen from  $s=\{0,...,S-1\}$ and k=$\{0,..,K-1\}$.		
	\end{theorem}
	
	\begin{corollary}\label{VRNonCS:CorollaryComplexityEstimationFull}
		In Algorithm \ref{VRNonCS:AlgorithmI}, let $\eta  = {n^{ - \alpha }}/2{L_f}( {B_G^4L_F^2/A + B_F^2L_G^2/B + L_f^2} ),d = {n^{{d_0}}},h = n^{h_0}/(e-1)$,
		where $\alpha,h_0,d_0>0$.  $K$ is the number of inner iteration, $S$ is the  number of inner iteration, $A = B_G^4L_F^2{n^{{A_0}}}/2$ is the sample times for mini-batch $\mathcal{A}_k$, $A_0>0$, $B$ is the sample times for mini-batch $\mathcal{B}_k$. There exist two constant $v_4>0$ and $w_4>0$ such that $K = w_4L_f^3{n^{3\alpha /2}}$ and $u = {n^{ - \alpha }}v_4/L_f^3$. The output $\tilde x_k^s$ satisfies
		\begin{align*}
		\mathbb{E}[{\| {\nabla f({\tilde x_k^s })} \|^2}]  \le \frac{{{n^\alpha }L_f^3(f({x_0}) - f({x^*}))}}{{v_4SK}}.
		\end{align*}
	\end{corollary}

	\begin{remark}
		Let us consider the function
		\begin{align} \label{VRNonCS:Definitionu}
		( {{1/2- {c_0}h} } )\eta  - \left( {2{L_f} + 4{c_0}} \right){\eta ^2}.
		\end{align}
		The parameter $c_0$ in (\ref{VRNonCS:Definitionc_0}) and (\ref{VRNonCS:Definitionc_0Full}) are the same except parameters $C$.	We assume that the bound of $C$ in (\ref{VRNonCS:DefinitionC}) and (\ref{VRNonCS:DefinitionCFull}) can be almost approximating the upper bound $e$, even though the bound $e$ cannot be exactly reached but actually can almost be reached. Furthermore, as can be seen in (\ref{VRNonCS:Definitionc_0h}) and (\ref{VRNonCS:Definitionc_0hFull}), the value of $c_0h$ are almost the same such that do not greatly affect the coefficient in (\ref{VRNonCS:DefinitionMaxu}) and (\ref{VRNonCS:DefinitionMaxuTilde}). What's more, since $1/2-c_0h>0$, we can obtain the bound  $c_0\ge 2(e-1)n^{\alpha/2}$. 	
		
		Thus, we only consider the value of $u$, that is the value of the function in (\ref{VRNonCS:Definitionu}). Since $u$ as the denominator of the convergence bound in Theorem \ref{VRNonCS:TheoremMainFull} and Theorem \ref{VRNonCS:TheoremMain}, the bigger of $u$ can result in better convergence rate. The function of (\ref{VRNonCS:Definitionu}) is the increase function as $\eta$ increase if $\eta\le ( {1/2 - {c_0}h} )/(2( {2{L_f} + 4{c_0}} ))\le \mathcal{O}(n^{\alpha/2})$. Since the step in (\ref{VRNonCS:DefinitionParametersAll}) and (\ref{VRNonCS:DefinitionParametersAllFull}) are $ \eta=\mathcal{O}(n^{-\alpha})$, the second term in (\ref{VRNonCS:Definitionu}) can be  ignored.
		
		Based on above analysis, we conclude that if $B = B_F^2L_G^2/L_f^2$, the step defined in (\ref{VRNonCS:DefinitionParametersAll}) and (\ref{VRNonCS:DefinitionParametersAllFull}) are the same. Furthermore, the number of inner iteration are also the same. So they have the same convergence rate. When $B > B_F^2L_G^2/L_f^2$,  the step defined in (\ref{VRNonCS:DefinitionParametersAllFull}) is larger than that of  (\ref{VRNonCS:DefinitionParametersAll}), which has the better convergence rate even though they share the same order of convergence rate with respect to $O(n^\alpha)$. 
	\end{remark}
	\begin{corollary}\label{VRNonCS:CorollaryQueryComplexityFull}
		Let $T$ is the total number of iteration, $K$ is the number of inner iteration, $S$ is the number of outer iteration. $A$ and $B$ is the sample times for forming a mini-batch  $\mathcal{A}_k$ and $\mathcal{B}_k$, $B=n^{B_0}$, $B_0>0$.  to achieve a fixed solution accuracy $\varepsilon>0$, that is $\mathbb{E}[{\left\| {\nabla f({\tilde x_k^s })} \right\|^2}] \le \varepsilon $, the query complexity of Algorithm \ref{VRNonCS:AlgorithmII} is ${\cal O}((m + n + {n^{{3\alpha }/{2} + \alpha }} + {n^{{3\alpha }/{2} + {B_0}}} )({n^{ - {\alpha }/{2}}}/\varepsilon )).$
	\end{corollary}
	
	\begin{remark}\label{VRNonCS:RemarkVRNCSIISizeB}
		Note that the size of $B$ does not affect the order of convergence rate but will have an influence on the query complexity. When $B_0\le\alpha$, the query complexity becomes ${\cal O}((m + n + n^{5\alpha/2} )(n^{ - \alpha /2}/\varepsilon ))$, which is the same as  Algorithm \ref{VRNonCS:AlgorithmII}; when  $B_0>\alpha$, the query complexity will increase to ${\cal O}((m + n +  {n^{{3\alpha}/{2}  + {B_0}}} )({n^{ - {\alpha }/{2}}}/\varepsilon ))$. Since $3\alpha/2+B_0>5\alpha/2$, here we do not need to analyze the value of $\alpha$. Because when $3\alpha/2+B_0$ is smaller than 1 or $m_0$, the query complexity will be equal to Algorithm \ref{VRNonCS:AlgorithmII}; when $3\alpha/2+B_0$ is bigger than 1 or $m_0$, the query complexity will be greater than Algorithm \ref{VRNonCS:AlgorithmII}. Therefore, in order to keep the query complexity non-increase and convergence increase, we should set the size of $B$ equal to the $A$.

	\end{remark}
	\section{Mini-Batch variance reduction for Non-convex Composition problem}
	In this section, we consider the mini-batch variance reduction method for non-convex composition problem, referred as Mini-Batch SCVR.  Algorithm \ref{VRNonCS:AlgorithmMiniBatch} presents Mini-Batch SCVR's pseudocode. Different from SCVRI and SCVRII, we redefine the estimated gradient of $f(x)$ as 
	\begin{align}
	\tilde \nabla_k  =& \frac{1}{{| {{\mathcal{I}_k}} |}}\sum\limits_{1 \le i \le | {{\mathcal{I}_k}} |} {( {{(\partial {G_k})^\mathsf{T}}\nabla {F_i}({{\hat G}_k}) - {(\partial {G_k})^\mathsf{T}}\nabla {F_i}(G({{\tilde x}_s}))} )}  + \nabla f({{\tilde x}_s}),\label{VRNonCS:DefinitionTildeNabla1}\\
	or, {{\tilde \nabla }_k} =& \frac{1}{{|{{\cal I}_k}|}}\sum\limits_{1 \le i \le |{{\cal I}_k}|} {{(\partial {G_{{B_k}}}({x_k}))^\mathsf{T}}\nabla {F_{{{\cal I}_k}[i]}}({{\hat G}_k}) - {(\partial {G_{{B_k}}}({{\tilde x}_s}))^\mathsf{T}}\nabla {F_{{{\cal I}_k}[i]}}(G({{\tilde x}_s})))}  + \nabla f({{\tilde x}_s})\label{VRNonCS:DefinitionTildeNabla2},
	\end{align}
	where ${{\mathcal{I}_k}}$ is the mini-batch set  for outer function $F_i$, formed by randomly sampling from $[n]$, and $\partial {G_{{{\cal B}_k}}} = \frac{1}{B}\sum\nolimits_{1 \le j \le B} {\partial {G_{{{\cal B}_k}[j]}}({x_k})} $. For a simple analysis, we define the size of ${{\mathcal{I}_k}}$ as ${| {{\mathcal{I}_k}} |}=b$. The following gives the key lemma for bounding the estimated gradient of $f$,
	\begin{algorithm}[h]
		\caption{Mini-batch Stochastic composition variance reduction}
		\label{VRNonCS:AlgorithmMiniBatch}
		\begin{algorithmic}
			\Require $K$, $S$, $\eta$ (learning rate), and $\tilde{x}_1$
			\For{$s =1,2,\cdots,S$}
			\State $	G( {{{\tilde x}_s}} ) = \frac{1}{m}\sum\limits_{j = 1}^m {{G_j}( {{{\tilde x}_s}} )}$\Comment{m queries}
			\State $	\partial G( {{{\tilde x}_s}} ) = \frac{1}{m}\sum\limits_{j = 1}^m {{\partial G_j}( {{{\tilde x}_s}} )}$\Comment{m queries}
			\State $\nabla f( {{{\tilde x}_s}} ) =( \partial G( {{{\tilde x}_s}} ))^\mathsf{T}\frac{1}{n}\sum\limits_{i = 1}^n {\nabla{F_i}( {G( {{{\tilde x}_s}} )} )} $\Comment{n queries}
			\State $x_1=\tilde{x}_s$
			\For{$k =1,2,3,\cdots,K$}
			\State
			Sampling from $[m]$ for A times to form mini-batch multiset $\mathcal{A}_k$
			\State 
			Sampling from $[m]$ for B times to form mini-batch multiset $\mathcal{B}_k$
			\State
			${{\hat G}_k} = \frac{1}{A}\sum\limits_{1 \le j \le A}^{} {\left( {{G_{{A_k}[j]}}\left( {{x_k}} \right) - {G_{{A_k}[j]}}\left( {{{\tilde x}_s}} \right)} \right)}  + G\left( {{{\tilde x}_s}} \right)$\Comment{2A queries}
			\State
			$\partial\hat G_k = \frac{1}{B}\sum\limits_{1 \le j \le B}^{} {( {{\partial G_{{\mathcal{B}_k}[j]}}( {{x_k}} ) - {\partial G_{{\mathcal{B}_k}[j]}}( {{{\tilde x}_s}} )} )}  + \partial G( {{{\tilde x}_s}} )$\Comment{2B queries}
			\State
			Sampling from $[n]$ for b times to form mini-batch multiset $\mathcal{I}_k$ 
			\State
			$\tilde \nabla_k  = \frac{1}{{| {{\mathcal{I}_k}} |}}\sum\limits_{1 \le i \le | {{\mathcal{I}_k}} |} {( {{{(\partial {{\hat G}_k})}^\mathsf{T}}\nabla {F_{{\mathcal{I}_k}[{i}]}}({{\hat G}_k}) - {({\partial G({{\tilde x}_s})})^\mathsf{T}}\nabla {F_{{\mathcal{I}_k}[{i}]}}(G({{\tilde x}_s}))} )}  + \nabla f({{\tilde x}_s})$\Comment{2b queries}
			\State or
			\State
			${{\tilde \nabla }_k} = \frac{1}{{|{{\cal I}_k}|}}\sum\limits_{1 \le i \le |{{\cal I}_k}|} {{(\partial {G_{{B_k}}}({x_k}))^\mathsf{T}}\nabla {F_{{{\cal I}_k}[i]}}({{\hat G}_k}) - {(\partial {G_{{B_k}}}({{\tilde x}_s}))^\mathsf{T}}\nabla {F_{{{\cal I}_k}[i]}}(G({{\tilde x}_s})))}  + \nabla f({{\tilde x}_s})$\Comment{2b queries}
			\State
			$x_{k+1}= x_k - \eta \tilde \nabla_k $ 
			\EndFor
			\State Update $\tilde{x}_{s+1}=x_K $
			\EndFor \\	
			\textbf{Output:}  $\tilde x_k^s$ is uniformly and randomly chosen from  $s=\{0,...,S-1\}$ and k=$\{0,..,K-1\}$.
		\end{algorithmic}
	\end{algorithm}
	
	\begin{lemma}\label{VRNonCS:LemmaBoundVarianceEstimatGradientfullMiniBatch}
		Suppose Assumption \ref{VRNonCS:AssumptionG}-\ref{VRNonCS:AssumptionGF} hold, let $b\ge1$, the estimated $ \tilde\nabla_k$ defined in (\ref{VRNonCS:DefinitionTildeNabla1}) can be bounded by
		\begin{align*}
		\mathbb{E}[{\| \tilde\nabla_k\|^2}] \le4\mathbb{E}[{\| {\nabla f({x_k})} \|^2}] + 4\left( {B_G^4L_F^2\frac{1}{A} + B_F^2L_G^2\frac{1}{B} + bL_f^2} \right){\frac{1}{b}}\mathbb{E}[{\| {{x_k} - {{\tilde x}_s}} \|^2}].
		\end{align*}
	\end{lemma}
	As $b$ increase, the bound will be  tighter.Furthermore, the parameter $b$ also affect the  convergence rate in the following analysis.  Note that here, we do not give the upper bound of $\mathbb{E}[{\| \tilde\nabla_k\|^2}]$ defined in (\ref{VRNonCS:DefinitionTildeNabla2}), the brief proof can be referred to Lemma \ref{VRNonCS:LemmaBoundVarianceEstimatGradient}. What's more, the order effects by the defined gradient estimator in (\ref{VRNonCS:DefinitionTildeNabla2}) and (\ref{VRNonCS:LemmaBoundVarianceEstimatGradient}) on  query complexity are the same.

	\subsection{Convergence analysis}
	Based on Lemma \ref{VRNonCS:LemmaBoundvariableXFullMiniBatch} and Lemm \ref{VRNonCS:LemmaBoundFunctionFFullMiniBatch}, we obtain the following theorem. Note that the proof details can be referred to Theorem \ref{VRNonCS:TheoremMainFull}. Through the convergence rate in Theorem in \ref{VRNonCS:TheoremMainFullMiniBatch}, we can obtain the corresponding result in Corollary \ref{VRNonCS:CorollaryComplexityEstimationFullMiniBatch}  that is dependent on $n$.
	\begin{theorem}\label{VRNonCS:TheoremMainFullMiniBatch}
		For the algorithm \ref{VRNonCS:AlgorithmMiniBatch}, let $h,d, \eta > 0$, and $b\ge1$ such that
		\begin{align}\label{VRNonCS:DefinitionUkFullMiniBatch}
		{u_k} = ( {1/2- {c_{k + 1}}h} )\eta  - ( {2{L_f} + 4{c_{k + 1}}} ){\eta ^2},\forall k \ge 0,
		\end{align}
		where
		\begin{align}\label{VRNonCS:DefinitionCkfullMiniBatch}
		{c_k} =&{c_{k + 1}}\left( {1 + \left( {\frac{1}{h} + \frac{1}{d} + dB_G^4L_F^2\frac{1}{A}} \right)\eta  + 4\left( {B_G^4L_F^2\frac{1}{A} + B_F^2L_G^2\frac{1}{B} + bL_f^2} \right)\frac{1}{b}{\eta ^2}} \right)\nonumber\\
		&+ B_G^4L_F^2\frac{1}{{2A}}\eta  + 2{L_f}\left( {B_G^4L_F^2\frac{1}{A} + B_F^2L_G^2\frac{1}{B} + bL_f^2} \right)\frac{1}{b}{\eta ^2},
		\end{align}
		$B_G$, $L_f$, and $L_F$ are parameters defined in Assumption \ref{VRNonCS:AssumptionG}-\ref{VRNonCS:AssumptionGF}, and A is the sample times for forming the mini-batch $\mathcal{A}_k$, $B$ is the sample times for mini-batch $\mathcal{B}_k$. Let $K$ is the number of inner iteration, $S$ is the  number of inner iteration,  and define $u = {\min _{0 \le k \le K - 1}}\{ {u_k}\} $, we have
		\begin{align*}
		\mathbb{E}[{\| {\nabla f({{\tilde x_k^s } })} \|^2}] \le \frac{{f({x_0}) - f({x^*})}}{{uKS}},
		\end{align*}
		where ${\tilde x_k^s }$ is uniformly and randomly chosen from  $s=\{0,...,S-1\}$ and k=$\{0,..,K-1\}$.		
	\end{theorem}
	\begin{corollary}\label{VRNonCS:CorollaryComplexityEstimationFullMiniBatch}
		In Algorithm \ref{VRNonCS:AlgorithmI}, let $\eta  = {bn^{ - \alpha }}/(2{L_f}( {B_G^4L_F^2/A + B_F^2L_G^2/B + bL_f^2} )),d = {n^{{d_0}}},h = n^{h_0}/(e-1)$,
		where $\alpha,h_0,d_0>0$ and $b\ge1$.  $K$ is the number of inner iteration, $S$ is the  number of inner iteration, $A = B_G^4L_F^2{n^{{A_0}}}/2$ is the sample times for mini-batch $\mathcal{A}_k$, $A_0>0$, $B=B_F^2L_G^2n^{B_0},$ is the sample times for mini-batch $\mathcal{B}_k$, $B_0>0$. There exist two constant $v_5>0$ and $w_5>0$ such that $K = w_5L_f^3{n^{3\alpha /2}}/b$ and $u =bv_5 {n^{ - \alpha }}/L_f^3$. The output ${\tilde x_k^s }$ satisfy
		\begin{align*}
		\mathbb{E}[{\| {\nabla f({{\tilde x_k^s } })} \|^2}]  \le \frac{{{n^\alpha }L_f^3(f({x_0}) - f({x^*}))}}{b{v_5SK}}.
		\end{align*}
	\end{corollary}
	
	Comparing with the convergence rate in Corollary \ref{VRNonCS:CorollaryComplexityEstimationFullMiniBatch} with Corollary \ref{VRNonCS:CorollaryComplexityEstimationFull} and \ref{VRNonCS:CorollaryComplexityEstimation}, we obtain the improvement of the convergence rate. But whether the query complexity improves, the next subsection gives the analysis of complexity.
	
	\subsection{Complexity analysis}
	In this subsection, we give the query complexity with two cases: the gradient of $f$  with mini-batch $\mathcal{I}$ is computed in parallel and non-parallel. For the case of parallel setting, we also divide into two conditions: the gradient  of $F_{\mathcal{I}}(G(x))$ is computed in parallel, and the gradient of $F_{\mathcal{I}}(G(x))$ and partial gradient $\partial\hat G_{\mathcal{B}}$ are both compute in parallel, respectively. Note that following Remark \ref{VRNonCS:RemarkVRNCSIISizeB}, we assume that the size of B is equal or small than the size of $A$.
	
	\begin{corollary}\label{VRNonCS:CorollaryQueryComplexityMiniBatchParallel} 
		In Algorithm \ref{VRNonCS:AlgorithmMiniBatch}, suppose that the stochastic gradient of $F_{\mathcal{I}}(G(x))$  with mini-batch $\mathcal{I}_k$ is computed in parallel.  Let  $b=n^{b_0}$,$b_0>0$, to achieve a fixed solution accuracy $\varepsilon>0$, that is $\mathbb{E}[{\| {\nabla f({\tilde x_k^s })} \|^2}] \le \varepsilon $, the query complexity of Algorithm \ref{VRNonCS:AlgorithmMiniBatch} is 
		\begin{center}
			$	\text{QC} = \left\{ {\begin{array}{*{20}{l}}
				{{\cal O}({n^{4/5 - {b_0}/5}}/\varepsilon )},&{0 < {m_0} \le 1};\\
				{{\cal O}({n^{4/5{m_0} - {b_0}/5}}/\varepsilon )},&{1 < {m_0} }.
				\end{array}} \right.$
		\end{center}
	\end{corollary}
	
	\begin{corollary}\label{VRNonCS:CorollaryQueryComplexityMiniBatchParallelAddG} 
		In Algorithm \ref{VRNonCS:AlgorithmMiniBatch}, suppose that the gradient of $f$  with mini-batch $\mathcal{I}_k$, the partial gradient  $\partial\hat G_k$  with mini-batch $\mathcal{B}_k$ and inner function $\hat G$ with mini-batch $\mathcal{A}_k$ are computed in parallel.  Let  $b=n^{b_0}$,$b_0>0$, to achieve a fixed solution accuracy $\varepsilon>0$, that is $\mathbb{E}[{\| {\nabla f({\tilde x_k^s })} \|^2}] \le \varepsilon $, the query complexity of Algorithm \ref{VRNonCS:AlgorithmMiniBatch} is 
		\begin{center}
			$	\text{QC} = \left\{ {\begin{array}{*{20}{l}}
				{{\cal O}({n^{2/3 - {b_0}/3}}/\varepsilon )},&{0 < {m_0} \le 1};\\
				{{\cal O}({n^{2/3{m_0} - {b_0}/3}}/\varepsilon )},&{1 < {m_0} }.
				\end{array}} \right.$
		\end{center}
	\end{corollary}
	
	For the non-parallel, we give the following query complexity results. Based on different sizes of the mini-batch $\mathcal{I}$, we obtain different query complexities.

	\begin{corollary}\label{VRNonCS:CorollaryQueryComplexityMiniBatchUnParallel}
		In Algorithm \ref{VRNonCS:AlgorithmMiniBatch},  let  $b=n^{b_0}$,$b_0>0$, to achieve a fixed solution accuracy $\varepsilon>0$, that is $\mathbb{E}[{\| {\nabla f({\tilde x_k^s })} \|^2}] \le \varepsilon $, the query complexity of Algorithm \ref{VRNonCS:AlgorithmMiniBatch} is 
		\begin{center}
			$\text{QC} = \left\{ {\begin{array}{*{20}{l}}
				{\left\{ {\begin{array}{*{20}{l}}
						{O({n^{4/5 - {b_0}/5}}/\varepsilon ),}&{0 < {m_0} \le 1;}\\
						{O({n^{4/5{m_0} - {b_0}/5}}/\varepsilon ),}&{1 < {m_0} ;}
						\end{array}} \right.}&{{b_0} \le {2}/{3}};\\
				{\left\{ {\begin{array}{*{20}{l}}
						{O({n^{2/3}}/\varepsilon ),}&{0 < {m_0} \le 1;}\\
						{O({n^{2/3{m_0}}}/\varepsilon ),}&{1 < {m_0};}
						\end{array}} \right.}&{{b_0} > {2}/{3}}.
				\end{array}} \right.$
		\end{center}
	\end{corollary}
	
	\begin{remark}
		Compare Corollary \ref{VRNonCS:CorollaryQueryComplexityMiniBatchParallel} with Corollary \ref{VRNonCS:CorollaryComplexityQCCS}, we can see that  the QC will reduce a factor of $\mathit{\Omega} (n^{b_0/5})$. In Corollary \ref{VRNonCS:CorollaryQueryComplexityMiniBatchParallelAddG} we can see  that the query complexity computation for composition problem with the  parallel setting is actually reduced to the general empirical minimization problem. Compare Corollary \ref{VRNonCS:CorollaryQueryComplexityMiniBatchParallelAddG} with Corollary \ref{VRNonCS:CorollaryComplexityQCS},  when $m_0\le 1$, the QC of Corollary \ref{VRNonCS:CorollaryQueryComplexityMiniBatchParallelAddG} will reduce a factor of $\mathit{\Omega}(n^{b_0/5+m_0/3})$ times. For the non-parallel setting, when $b_0>2/3$, there will be the best QC for mini-batch SCVR. Comparing parallel and non-parallel, the QC in Corollary \ref{VRNonCS:CorollaryQueryComplexityMiniBatchUnParallel}  is the same as in Corollary \ref{VRNonCS:CorollaryQueryComplexityMiniBatchParallel} if $b_0\le2/3$, and it will be worse than Corollary \ref{VRNonCS:CorollaryQueryComplexityMiniBatchParallel}  and Corollary \ref{VRNonCS:CorollaryQueryComplexityMiniBatchParallelAddG} if $b_0> 2/3$.
	\end{remark}

	\section{Bound analysis for non-convex stochastic composition  problem}\label{VRNonCS:SectionBoundAnalysis}
	In this section, we mainly give  different kinds of bounds for each algorithm. These bounds will be used to analyze the convergence rate. We assume that these algorithms are both under Assumption \ref{VRNonCS:AssumptionG}-\ref{VRNonCS:AssumptionIndependent}. Parameters such as $B_G$, $B_F$, $L_F$, $L_G$ and $L_f$ in the bound are from these Assumptions. We do not define the exact value of parameters such as $h$, $d$, $A$  and $B$, which have great influence on the convergence and will be defined in different algorithms.
	\subsection{The bound of estimated inner function $G(x)$}
	We have the following lemmas concerning the bound of estimated inner function $G(x)$. We give the bound proof of Lemma \ref{VRNonCS:LemmaBoundVarianceG} and \ref{VRNonCS:LemmaBoundVariancePartialG}  of the estimated $\hat{G}$ and $\partial\hat{G}$.
	There are two kinds of the estimators: $\hat{G}$ and $\partial\hat{G}$ results in the biased estimation of the gradient of $f(x)$. However, as the variable $x$ approach to the optimal solution, the upper bound will approximate to zero, which can be illustrated by Lemma \ref{VRNonCS:LemmaBoundVarianceEstimatGradientSub} and \ref{VRNonCS:LemmaBoundVarianceEstimatGradientFullSub}.
	
	\textbf{Proof of Lemma \ref{VRNonCS:LemmaBoundVarianceG}:}
	\begin{proof} Based on the definition of ${\hat G}_k$  in (\ref{VRNonCS:DefinitionHatG}), we have 
		\begin{align*}
		\mathbb{E}[{\| {{{\hat G}_k} - G({{\tilde x}_s})} \|^2}] =& \mathbb{E}[{\| {\frac{1}{A}\sum\limits_{1 \le j \le A}^{} {({G_{{{\cal A}_k}[j]}}({x_k}) - {G_{{{\cal A}_k}[j]}}({{\tilde x}_s}))}  + G({{\tilde x}_s}) - G({{\tilde x}_s})} \|^2}]\\
		\mathop  \le \limits^{\scriptsize \textcircled{\tiny{1}}}& \frac{1}{{{A^2}}}\sum\limits_{1 \le j \le A}^{} {\mathbb{E}[{{\| {{G_{{\mathcal{A}_k}[j]}}({x_k}) - {G_{{\mathcal{A}_k}[j]}}({{\tilde x}_s}) + G({{\tilde x}_s}) - G({{\tilde x}_s})} \|}^2}]} \\
		\mathop  \le \limits^{\scriptsize \textcircled{\tiny{2}}}& \frac{1}{{{A^2}}}\sum\limits_{1 \le j \le A}^{} {\mathbb{E}[{{\| {{G_{{\mathcal{A}_k}[j]}}({x_k}) - {G_{{\mathcal{A}_k}[j]}}({{\tilde x}_s})} \|}^2}]} \\
		=& \frac{1}{A}\mathbb{E}[{\| {{G_{{\mathcal{A}_k}[j]}}({x_k}) - {G_{{\mathcal{A}_k}[j]}}({{\tilde x}_s})} \|^2}]\\
		\mathop  \le \limits^{\scriptsize \textcircled{\tiny{3}}}& B_G^2\frac{1}{A}\mathbb{E}[{\| {{x_k} - {{\tilde x}_s}} \|^2}],
		\end{align*}
		where inequalities  ${\small\textcircled{\scriptsize{1}}}$ and  ${\small  \textcircled{\scriptsize{2}}}$ use Lemma \ref{VRNonCS:AppendixLemmaInequation} and Lemma \ref{VRNonCS:AppendixLemmaRandomVariableInequation}, and inequality   ${\small  \textcircled{\scriptsize{3}}}$ is based on the Definition \ref{VRNonCS:DefinitionLipschitzFunction} of Lipschitz function.
	\end{proof}
	
	\textbf{Proof of Lemma \ref{VRNonCS:LemmaBoundVariancePartialG}:}
	\begin{proof} Based on the definition of ${\partial\hat G}_k$  in (\ref{VRNonCS:DefinitionPartialHatG}), we have
		\begin{align*}
		\mathbb{E}[{\| {\partial {{\hat G}_k} - \partial G({{\tilde x}_s})} \|^2}] =& \mathbb{E}[{\| {\frac{1}{B}\sum\limits_{1 \le j \le B}^{} {(\partial {G_{{{\cal B}_k}[j]}}({x_k}) - \partial {G_{{{\cal B}_k}[j]}}({{\tilde x}_s}))}  + \partial G({{\tilde x}_s}) - \partial G({{\tilde x}_s})} \|^2}]\\
		\mathop  \le \limits^{\scriptsize \textcircled{\tiny{1}}}& \frac{1}{{{B^2}}}\sum\limits_{1 \le j \le B}^{} {\mathbb{E}[{{\| {\partial {G_{{{\cal B}_k}[j]}}({x_k}) - \partial {G_{{{\cal B}_k}[j]}}({{\tilde x}_s}) + \partial G({{\tilde x}_s}) - \partial G({{\tilde x}_s})} \|}^2}]} \\
		\mathop  \le \limits^{\scriptsize \textcircled{\tiny{2}}}& \frac{1}{{{B^2}}}\sum\limits_{1 \le j \le B}^{} {\mathbb{E}[{{\| {\partial {G_{{{\cal B}_k}[j]}}({x_k}) - \partial {G_{{{\cal B}_k}[j]}}({{\tilde x}_s})} \|}^2}]} \\
		=& \frac{1}{B}\mathbb{E}[{\| {\partial {G_{{{\cal B}_k}[j]}}({x_k}) - \partial {G_{{{\cal B}_k}[j]}}({{\tilde x}_s})} \|^2}]\\
		\mathop  \le \limits^{\scriptsize \textcircled{\tiny{3}}}& L_G^2\frac{1}{B}\mathbb{E}[{\| {{x_k} - {{\tilde x}_s}} \|^2}],
		\end{align*}
		where inequalities ${\small  \textcircled{\scriptsize{1}}}$ and ${\small \textcircled{\scriptsize{2}}}$ follow from Lemma \ref{VRNonCS:AppendixLemmaInequation} and Lemma \ref{VRNonCS:AppendixLemmaRandomVariableInequation}, inequality ${\small \textcircled{\scriptsize{3}}}$  is based on the $L_G$ smooth of $G$.
	\end{proof}
	
	\begin{lemma}\label{VRNonCS:LemmaBoundVarianceEstimatGradientSub}
		In algorithm \ref{VRNonCS:AlgorithmI}, for the intermediated iteration at $x_k$ of $s$-th epoch, and $\hat{G}_k$ defined in (\ref{VRNonCS:DefinitionHatG}),  we have,
		\begin{align*}
		\mathbb{E}[ {\| ( {\partial {G_{{{j_k}}}}( {{x_k}} ) )^\mathsf{T}\nabla {F_{i_k}}( {{{\hat G}_k}} ) - {( {\partial {G_{j_k}}( {{{\tilde x}_s}} )} )^\mathsf{T}}\nabla {F_{i_k}}( {G( {{{\tilde x}_s}} )} )} \|^2} ] \le 2\left( {B_G^4L_F^2\frac{1}{A} + L_f^2} \right)\mathbb{E}[{\| {{x_k} - {{\tilde x}_s}} \|^2}].
		\end{align*}
	\end{lemma}
	\begin{proof}
		Through adding and subtracting the term ${{( {\partial {G_{j_k}}( {{x_k}} )} )^\mathsf{T}}\nabla {F_{i_k}}( {G( {{x_k}} )} )}$, we have	
		\begin{align*}
		&\mathbb{E}[ {\| {{( {\partial {G_{j_k}}( {{x_k}} )} )^\mathsf{T}}\nabla {F_{i_k}}( {{{\hat G}_k}} ) - {( {\partial {G_{j_k}}( {{{\tilde x}_s}} )} )^\mathsf{T}}\nabla {F_{i_k}}( {G( {{{\tilde x}_s}} )} )} \|^2} ]\\
		\mathop  \le \limits^{\scriptsize \textcircled{\tiny{1}}}& 2\mathbb{E}[{\| {{(\partial {G_{j_k}}({x_k}))^\mathsf{T}}\nabla {F_{i_k}}({{\hat G}_k}) - {(\partial {G_{j_k}}({x_k}))^\mathsf{T}}\nabla {F_{i_k}}(G({x_k}))} \|^2}]\\&+ 2\mathbb{E}[{\| {{(\partial {G_{j_k}}({x_k}))^\mathsf{T}}\nabla {F_{i_k}}(G({x_k})) - {(\partial {G_{j_k}}({{\tilde x}_s}))^\mathsf{T}}\nabla {F_{i_k}}(G({{\tilde x}_s}))} \|^2}]\\
		\mathop  \le \limits^{\scriptsize \textcircled{\tiny{2}}}& 2B_G^2\mathbb{E}[{\| {\nabla {F_{i_k}}({{\hat G}_k}) - \nabla {F_{i_k}}(G({x_k}))} \|^2}] + 2L_f^2\mathbb{E}[{\| {{x_k} - {{\tilde x}_s}} \|^2}]\\
		\mathop  \le \limits^{\scriptsize \textcircled{\tiny{3}}}& 2B_G^2L_F^2\mathbb{E}[{\| {{{\hat G}_k} - G({x_k})} \|^2}] + 2L_f^2\mathbb{E}[{\| {{x_k} - {{\tilde x}_s}} \|^2}]\\
		\mathop  \le \limits^{\scriptsize \textcircled{\tiny{4}}}& 2B_G^2L_F^2B_G^2\frac{1}{A}\mathbb{E}[{\| {{x_k} - {{\tilde x}_s}} \|^2}] + 2L_f^2\mathbb{E}[{\| {{x_k} - {{\tilde x}_s}} \|^2}]\\
		=& 2\left( {B_G^4L_F^2\frac{1}{A} + L_f^2} \right)\mathbb{E}[{\| {{x_k} - {{\tilde x}_s}} \|^2}],			
		\end{align*}
		where inequality ${\small \textcircled{\scriptsize{1}}}$ use Lemma \ref{VRNonCS:AppendixLemmaInequation}; inequality ${\small \textcircled{\scriptsize{2}}}$ is based on the bounded Jacobian of $G_{j_k}$ and (\ref{VRNonCS:AssumptionInequality}); inequality ${\small \textcircled{\scriptsize{3}}}$ follows the smoothness of $F_{i_k}$; inequality ${\small \textcircled{\scriptsize{4}}}$ use Lemma \ref{VRNonCS:LemmaBoundVarianceG}.
	\end{proof}
	\begin{lemma}\label{VRNonCS:LemmaBoundVarianceEstimatGradientFullSub}
		In Algorithm \ref{VRNonCS:AlgorithmII},  for the intermediated iteration at $x_k$ of $s$-th epoch, and $\partial\hat{G}_k$ defined in (\ref{VRNonCS:DefinitionPartialHatG}),  we have,
		\begin{align*}
		&\mathbb{E}[{\| {{(\partial {{\hat G}_k})^\mathsf{T}}\nabla {F_{{i_k}}}({{\hat G}_k}) - {(\partial G({{\tilde x}_s}))^\mathsf{T}}\nabla {F_{{{i_k}}}}(G({{\tilde x}_s}))} \|^2}]\le 2\left( {B_G^4L_F^2\frac{1}{A} + B_F^2L_G^2\frac{1}{B}} \right)\mathbb{E}[{\| {{x_k} - {{\tilde x}_s}} \|^2}].
		\end{align*}
	\end{lemma}
	\begin{proof} Through adding and subtracting ${(\partial {{\hat G}_k})^\mathsf{T}}\nabla {F_i}(G({{\tilde x}_s}))$, we have
		\begin{align*}
		&\mathbb{E}[{\| {{(\partial {{\hat G}_k})^\mathsf{T}}\nabla {F_{i_k}}({{\hat G}_k}) - {(\partial G({{\tilde x}_s}))^\mathsf{T}}\nabla {F_{i_k}}(G({{\tilde x}_s}))} \|^2}]\\
		=& \mathbb{E}[{\| {{(\partial {{\hat G}_k})^\mathsf{T}}\nabla {F_i}({{\hat G}_k}) - {(\partial {{\hat G}_k})^\mathsf{T}}\nabla {F_i}(G({{\tilde x}_s})) + {(\partial {{\hat G}_k})^\mathsf{T}}\nabla {F_i}(G({{\tilde x}_s})) - {(\partial G({{\tilde x}_s}))^\mathsf{T}}\nabla {F_i}(G({{\tilde x}_s}))} \|^2}]\\
		\mathop  \le \limits^{\scriptsize \textcircled{\tiny{1}}}& 2\mathbb{E}[{\| {{(\partial {{\hat G}_k})^\mathsf{T}}\nabla {F_{i_k}}({{\hat G}_k}) - {(\partial {{\hat G}_k})^\mathsf{T}}\nabla {F_{i_k}}(G({{\tilde x}_s}))} \|^2}]\\& + 2\mathbb{E}[{\| {{(\partial {{\hat G}_k})^\mathsf{T}}\nabla {F_{i_k}}(G({{\tilde x}_s})) - {(\partial G({{\tilde x}_s}))^\mathsf{T}}\nabla {F_{i_k}}(G({{\tilde x}_s}))} \|^2}]\\
		\mathop  \le \limits^{\scriptsize \textcircled{\tiny{2}}}& 2B_G^2\mathbb{E}[{\| {\nabla {F_{i_k}}({{\hat G}_k}) - \nabla {F_{i_k}}(G({{\tilde x}_s}))} \|^2}] + 2B_F^2\mathbb{E}[{\| {\partial {{\hat G}_k} - \partial G({{\tilde x}_s})} \|^2}]\\
		\mathop  \le \limits^{\scriptsize \textcircled{\tiny{3}}}& 2B_G^2L_F^2\mathbb{E}[{\| {{{\hat G}_k} - G({{\tilde x}_s})} \|^2}] + 2B_F^2\mathbb{E}[{\| {\partial {{\hat G}_k} - \partial G({{\tilde x}_s})} \|^2}]\\
		\mathop  \le \limits^{\scriptsize \textcircled{\tiny{4}}}& 2B_G^2L_F^2B_G^2\frac{1}{A}\mathbb{E}[{\| {{x_k} - {{\tilde x}_s}} \|^2}] + 2B_F^2L_G^2\frac{1}{B}\mathbb{E}[{\| {{x_k} - {{\tilde x}_s}} \|^2}]\\
		=& 2\left( {B_G^4L_F^2\frac{1}{A} + B_F^2L_G^2\frac{1}{B}} \right)\mathbb{E}[{\| {{x_k} - {{\tilde x}_s}} \|^2}],
		\end{align*}
		where inequality ${\small \textcircled{\scriptsize{1}}}$ uses Lemma \ref{VRNonCS:AppendixLemmaInequation}; inequality ${\small \textcircled{\scriptsize{2}}}$ is based on the bounded Jacobian of $G$ and (\ref{VRNonCS:AssumptionInequality}); inequality ${\small \textcircled{\scriptsize{3}}}$ follows the smoothness of $F_{i_k}$; inequality ${\small \textcircled{\scriptsize{4}}}$ use Lemma \ref{VRNonCS:LemmaBoundVarianceG} and Lemma \ref{VRNonCS:LemmaBoundVariancePartialG}.
	\end{proof}
	\subsection{The bound of the estimated gradient $f(x)$}
	In this subsection, we give all kinds of the bounds involving the estimated inner function of $G(x)$, estimated partial gradient of inner function $G(x)$, and estimated gradient of $f(x)$. All the bounds  can not only be considered as the tool for analyzing the convergence, but also illustrate the  variance reduction technology.
	\subsubsection{The norm bound of the estimated gradient}
	We give the following proof of the bounds concerning the norm of the estimated gradient: ${\nabla f( {{x_k}} )}$ and ${ \tilde\nabla_k} $.
	
	\textbf{Proof of Lemma \ref{VRNonCS:LemmaBoundVarianceEstimatGradient}}
	\begin{proof}
		Through adding and subtracting the term ${\nabla f( {{x_k}} )}$, we have
		\begin{align*}
		&\mathbb{E}[ {\| { \nabla {\hat f_k}} \|^2} ]\\
		=& \mathbb{E}[ {\| {{( {\partial {G_{j_k}}( {{x_k}} )} )^\mathsf{T}}\nabla {F_{i_k}}( {{{\hat G}_k}} ) - {( {\partial {G_{j_k}}( {{{\tilde x}_s}} )} )^\mathsf{T}}\nabla {F_{i_k}}( {G( {{{\tilde x}_s}} )} ) + \nabla f( {{{\tilde x}_s}} )} \|^2} ]\\
		=& \mathbb{E}[ {\| {\nabla f( {{x_k}} ) - \nabla f( {{x_k}} ) + {( {\partial {G_{j_k}}( {{x_k}} )} )^\mathsf{T}}\nabla {F_{i_k}}( {{{\hat G}_k}} ) - {( {\partial {G_{j_k}}( {{{\tilde x}_s}} )} )^\mathsf{T}}\nabla {F_{i_k}}( {G( {{{\tilde x}_s}} )} ) + \nabla f( {{{\tilde x}_s}} )} \|^2} ]\\
		\mathop  \le \limits^{\scriptsize \textcircled{\tiny{1}}}&2\mathbb{E}[ {\| {\nabla f( {{x_k}} ) - \nabla f( {{x_k}} ) + \nabla f( {{{\tilde x}_s}} )} \|^2} ] + 2\mathbb{E}[ {\| {{( {\partial {G_{{j_k}}}( {{x_k}} )} )^\mathsf{T}}\nabla {F_{{i_k}}}( {{{\hat G}_k}} ) - {( {\partial {G_{{j_k}}}( {{{\tilde x}_s}} )} )^\mathsf{T}}\nabla {F_{{i_k}}}( {G( {{{\tilde x}_s}} )} )} \|^2} ]\\
		\mathop  \le \limits^{\scriptsize \textcircled{\tiny{2}}}& 4\left( {\mathbb{E}[ {\| {\nabla f( {{x_k}} )} \|^2} ] + E[ {\| {\nabla f( {{x_k}} ) - \nabla f( {\tilde x_s} )} \|^2} ]} \right) + 4\left( {B_G^4L_F^2}{\frac{1}{A} + L_f^2} \right)\mathbb{E}[ {\| {{x_k} - \tilde x_s} \|^2} ]\\
		\mathop  \le \limits^{\scriptsize \textcircled{\tiny{3}}}& 4\left( {\mathbb{E}[ {\| {\nabla f( {{x_k}} )} \|^2} ] + L_f^2\mathbb{E}[ {\| {{x_k} - \tilde x_s} \|^2} ]} \right) + 4\left( {B_G^4L_F^2}{\frac{1}{A} + L_f^2} \right)\mathbb{E}[ {\| {{x_k} - \tilde x_s} \|^2} ]\\
		=& 4\mathbb{E}[ {\| {\nabla f( {{x_k}} )} \|^2} ] + 4\left( {2L_f^2 + {B_G^4L_F^2}\frac{1}{A}} \right)\mathbb{E}[ {\| {{x_k} - \tilde x_s} \|^2} ],			
		\end{align*}
		where inequality  ${\small \textcircled{\scriptsize{1}}}$ use Lemma \ref{VRNonCS:AppendixLemmaInequation}; inequality  ${\small \textcircled{\scriptsize{2}}}$ use Lemma \ref{VRNonCS:AppendixLemmaInequation} and Lemma \ref{VRNonCS:LemmaBoundVarianceEstimatGradientSub}; inequality  ${\small \textcircled{\scriptsize{3}}}$ use the smoothness of function $f$.
	\end{proof}
	\textbf{Proof of Lemma \ref{VRNonCS:LemmaBoundVarianceEstimatGradientfullMiniBatch}}
	\begin{proof}
		Through adding and subtracting the term ${\nabla f( {{x_k}} )}$, we have
		\begin{align*}
		&\mathbb{E}[{\| \tilde\nabla_k\|^2}] \\
		=& E[{\|  \frac{1}{{b}}\sum\limits_{1 \le i \le b} {( {{(\partial {G_k})^\mathsf{T}}\nabla {F_i}({{\hat G}_k}) - {(\partial {G_k})^\mathsf{T}}\nabla {F_i}(G({{\tilde x}_s}))} )}  + \nabla f({{\tilde x}_s}) \|^2}]\\
		=& \mathbb{E}[{\| {\nabla f({x_k}) - \nabla f({x_k}) + \frac{1}{{b}}\sum\limits_{1 \le i \le b} {( {{(\partial {G_k})^\mathsf{T}}\nabla {F_{\mathcal{I}_k[i]}}({{\hat G}_k}) - {(\partial {G_k})^\mathsf{T}}\nabla {F_{\mathcal{I}_k[i]}}(G({{\tilde x}_s}))} )} + \nabla f({{\tilde x}_s})} \|^2}]\\
		\mathop  \le \limits^{\scriptsize \textcircled{\tiny{1}}}& 2\mathbb{E}[{\| {\nabla f({x_k}) - \nabla f({x_k}) + \nabla f({{\tilde x}_s})} \|^2}] +2\mathbb{E}[{\| {\frac{1}{b}\sum\limits_{1 \le i \le b} {({(\partial {G_k})^\mathsf{T}}\nabla {F_{\mathcal{I}_k[i]}}({{\hat G}_k}) - {(\partial {G_k})^\mathsf{T}}\nabla {F_{\mathcal{I}_k[i]}}(G({{\tilde x}_s})))} } \|^2}]\\
		\mathop  \le \limits^{\scriptsize \textcircled{\tiny{2}}}& 2\mathbb{E}[{\| {\nabla f({x_k}) - \nabla f({x_k}) + \nabla f({{\tilde x}_s})} \|^2}] +2\frac{1}{{{b^2}}}\sum\limits_{1 \le i \le b} {\mathbb{E}[{{\left\| {({(\partial {G_k})^\mathsf{T}}\nabla {F_{\mathcal{I}_k[i]}}({{\hat G}_k}) - {(\partial {G_k})^\mathsf{T}}\nabla {F_{\mathcal{I}_k[i]}}(G({{\tilde x}_s})))} \right\|}^2}]} \\
		\mathop  \le \limits^{\scriptsize \textcircled{\tiny{3}}}& 4( {\mathbb{E}[{\| {\nabla f({x_k})} \|^2}] + \mathbb{E}[{\| {\nabla f({x_k}) - \nabla f({{\tilde x}_s})} \|^2}]} ) + 4\left( {B_G^4L_F^2\frac{1}{A} + B_F^2L_G^2\frac{1}{B}} \right){\frac{1}{b}}\mathbb{E}[{\| {{x_k} - {{\tilde x}_s}} \|^2}]\\
		\mathop  \le \limits^{\scriptsize \textcircled{\tiny{4}}}& 4( {\mathbb{E}[{\| {\nabla f({x_k})} \|^2}] + L_f^2\mathbb{E}[{\| {{x_k} - {{\tilde x}_s}} \|^2}]} ) + 4\left( {B_G^4L_F^2\frac{1}{A} + B_F^2L_G^2\frac{1}{B}} \right){\frac{1}{b}}\mathbb{E}[{\| {{x_k} - {{\tilde x}_s}} \|^2}]\\
		=& 4\mathbb{E}[{\| {\nabla f({x_k})} \|^2}] + 4\left( {B_G^4L_F^2\frac{1}{A} + B_F^2L_G^2\frac{1}{B} + bL_f^2} \right){\frac{1}{b}}\mathbb{E}[{\| {{x_k} - {{\tilde x}_s}} \|^2}],			
		\end{align*}
		where inequality  ${\small \textcircled{\scriptsize{1}}}$ and inequality  ${\small \textcircled{\scriptsize{2}}}$ use Lemma \ref{VRNonCS:AppendixLemmaInequation}; inequality  ${\small \textcircled{\scriptsize{3}}}$ use Lemma \ref{VRNonCS:AppendixLemmaInequation} and Lemma \ref{VRNonCS:LemmaBoundVarianceEstimatGradientFullSub}; inequality  ${\small \textcircled{\scriptsize{4}}}$ use the smoothness of function $f$.
	\end{proof}
	\subsubsection{The bound tool for convergence form}
	The following lemmas show the estimated upper bounds, which are used for the convergence rates. There are three kinds of the estimated gradient of $f$: $ \nabla {{\hat f}_k}$, $ \nabla {{\tilde f}_k}$, and ${ \tilde\nabla_k}$. However, the expectations of them are the same. So the upper bound of Lemma \ref{VRNonCS:LemmaBoundvariableXSub}, \ref{VRNonCS:LemmaBoundvariableXFullSub} and \ref{VRNonCS:LemmaBoundvariableXFullSubMiniBatch} are the same, and Lemma \ref{VRNonCS:LemmaBoundFunctionFSub}, \ref{VRNonCS:LemmaBoundFunctionFFullSub} and \ref{VRNonCS:LemmaBoundFunctionFFullSubMiniBatch} are the same.

	\begin{lemma}\label{VRNonCS:LemmaBoundvariableXSub}
		In Algorithm \ref{VRNonCS:AlgorithmI}, suppose Assumption \ref{VRNonCS:AssumptionG}, \ref{VRNonCS:AssumptionF} and \ref{VRNonCS:AssumptionIndependent} hold, $ \nabla {{\hat f}_k}$ defined in (\ref{VRNonCS:DefinitionPartialHatf}), we can obtain the following bound
		\begin{align*}
		\mathbb{E}[\langle \nabla {{\hat f}_k},{x_k} - {{\tilde x}_s}\rangle ] \le    - h\frac{1}{2}{\left\| {\nabla f({x_k})} \right\|^2} - \left( {\frac{1}{2h} + \frac{1}{2d}+ dB_G^4L_F^2\frac{1}{2A} } \right)\mathbb{E}[{\| {{x_k} - {{\tilde x}_s}} \|^2}],
		\end{align*}
		where parameters $d,h>0$.
	\end{lemma}
	\begin{proof} Through adding and subtracting the term $ {(\partial {G_j}({x_k}))^\mathsf{T}}\nabla {F_i}(G({x_k})) $, we have
		\begin{align*}
		&\mathbb{E}[\langle \nabla {{\hat f}_k},{x_k} - {{\tilde x}_s}\rangle ]\\
		=& \mathbb{E}[\langle {(\partial {G_{j_k}}({x_k}))^\mathsf{T}}\nabla {F_{i_k}}({{\hat G}_k}) - {(\partial {G_{j_k}}({{\tilde x}_s}))^\mathsf{T}}\nabla {F_{i_k}}(G({{\tilde x}_s})) + \nabla f({{\tilde x}_s}),{x_k} - {{\tilde x}_s}\rangle ]\\
		=& \mathbb{E}[\langle {(\partial {G_{j_k}}({x_k}))^\mathsf{T}}\nabla {F_{i_k}}({{\hat G}_k}),{x_k} - {{\tilde x}_s}\rangle ]\\
		=& \mathbb{E}[\langle {(\partial {G_{j_k}}({x_k}))^\mathsf{T}}\nabla {F_{i_k}}({{\hat G}_k}) - {(\partial {G_{j_k}}({x_k}))^\mathsf{T}}\nabla {F_{i_k}}(G({x_k})) + {(\partial {G_{j_k}}({x_k}))^\mathsf{T}}\nabla {F_{i_k}}(G({x_k})),{x_k} - {{\tilde x}_s}\rangle ]\\
		=& \mathbb{E}[\langle {(\partial {G_{j_k}}({x_k}))^\mathsf{T}}\nabla {F_{i_k}}(G({x_k})),\tilde x_s - {x_k}\rangle ]\\& + \mathbb{E}[\langle {(\partial {G_{j_k}}({x_k}))^\mathsf{T}}\nabla {F_{i_k}}({{\hat G}_k}) - {(\partial {G_{j_k}}({x_k}))^\mathsf{T}}\nabla {F_{i_k}}(G({x_k})),{x_k} - {{\tilde x}_s}\rangle ]\\
		\mathop  \le \limits^{\scriptsize \textcircled{\tiny{1}}}& \langle \nabla f({x_k}),{x_k} - {{\tilde x}_s}\rangle - \frac{1}{2d}{\| {{x_k} - {{\tilde x}_s}} \|^2}  - d\frac{1}{2}\mathbb{E}[{\| {\nabla {G_{j_k}}({x_k})\nabla {F_{i_k}}({{\hat G}_k}) - \nabla {G_{j_k}}({x_k})\nabla {F_{i_k}}(G({x_k}))} \|^2}]\\
		\mathop  \le \limits^{\scriptsize \textcircled{\tiny{2}}}&  - h\frac{1}{2}{\| {\nabla f({x_k})} \|^2} - \frac{1}{2h}{\| {{x_k} - {{\tilde x}_s}} \|^2} - \frac{1}{2d}{\| {{x_k} - {{\tilde x}_s}} \|^2} - dB_G^2\frac{1}{2}\mathbb{E}[{\| {\nabla {F_{i_k}}({{\hat G}_k}) - \nabla {F_{i_k}}(G({x_k}))} \|^2}]\\
		\mathop  \le \limits^{\scriptsize \textcircled{\tiny{3}}}&  - h\frac{1}{2}{\| {\nabla f({x_k})} \|^2} - \frac{1}{2h}{\| {{x_k} - {{\tilde x}_s}} \|^2} - \frac{1}{2d}{\| {{x_k} - {{\tilde x}_s}} \|^2} - dB_G^2L_F^2\frac{1}{2}\mathbb{E}[{\| {{{\hat G}_k} - G({x_k})} \|^2}]\\
		\mathop  \le \limits^{\scriptsize \textcircled{\tiny{4}}}&  - h\frac{1}{2}{\| {\nabla f({x_k})} \|^2} - \frac{1}{2h}{\| {{x_k} - {{\tilde x}_s}} \|^2} - \frac{1}{2d}{\| {{x_k} - {{\tilde x}_s}} \|^2} - dB_G^2L_F^2B_G^2\frac{1}{2A}\mathbb{E}[{\| {{x_k} - {{\tilde x}_s}} \|^2}]\\
		=&  - h\frac{1}{2}{\left\| {\nabla f({x_k})} \right\|^2} - \left( {\frac{1}{2h} + \frac{1}{2d}+ dB_G^4L_F^2\frac{1}{2A} } \right)\mathbb{E}[{\| {{x_k} - {{\tilde x}_s}} \|^2}],
		\end{align*}
		where inequality ${\small \textcircled{\scriptsize{1}}}$  use the equality $E[ {(\partial {G_j}({x_k}))^\mathsf{T}}\nabla {F_i}(G({x_k}))]= \nabla f({x_k})$ and Lemma \ref{VRNonCS:AppendixLemmaInEquationWithab}, $d>0$; inequality ${\small \textcircled{\scriptsize{2}}}$  use Lemma \ref{VRNonCS:AppendixLemmaInEquationWithab}, $h>0$, and the bounded of Jacobian of $G$; inequality ${\small \textcircled{\scriptsize{3}}}$ use the smoothness of $F_i$; inequality ${\small \textcircled{\scriptsize{4}}}$ use Lemma \ref{VRNonCS:LemmaBoundVarianceG}.
	\end{proof}
	Based on the equation $\mathbb{E}[ \nabla {{\hat f}_k}]=\mathbb{E}[ \nabla {{\tilde f}_k}]=\mathbb{E}[{ \tilde\nabla_k}]$, we can also obtain the following bounds,
	\begin{lemma}\label{VRNonCS:LemmaBoundvariableXFullSub}
		In Algorithm \ref{VRNonCS:AlgorithmII}, suppose Assumption \ref{VRNonCS:AssumptionG} and \ref{VRNonCS:AssumptionF} hold, $ \nabla {{\tilde f}_k}$ defined in (\ref{VRNonCS:DefinitionPartialTildef}), we can obtain the following bound
		\begin{align*}
		\mathbb{E}[\langle \nabla {{\tilde f}_k},{x_k} - {{\tilde x}_s}\rangle ] \le  - h\frac{1}{2}{\left\| {\nabla f({x_k})} \right\|^2} - \left( {\frac{1}{2h} + \frac{1}{2d}+ dB_G^4L_F^2\frac{1}{2A} } \right)\mathbb{E}[{\| {{x_k} - {{\tilde x}_s}} \|^2}],
		\end{align*}
		where parameters $d,h>0$.
	\end{lemma}
	\begin{lemma}\label{VRNonCS:LemmaBoundvariableXFullSubMiniBatch}
		In Algorithm \ref{VRNonCS:AlgorithmII}, suppose Assumption \ref{VRNonCS:AssumptionG} and \ref{VRNonCS:AssumptionF} hold, $ \nabla {{\tilde f}_k}$ defined in (\ref{VRNonCS:DefinitionTildeNabla1}) and (\ref{VRNonCS:DefinitionTildeNabla2}), we can obtain the following bound
		\begin{align*}
		\mathbb{E}[\langle \tilde\nabla_k,{x_k} - {{\tilde x}_s}\rangle ] \le  - h\frac{1}{2}{\left\| {\nabla f({x_k})} \right\|^2} - \left( {\frac{1}{2h} + \frac{1}{2d}+ dB_G^4L_F^2\frac{1}{2A} } \right)\mathbb{E}[{\| {{x_k} - {{\tilde x}_s}} \|^2}],
		\end{align*}
		where parameters $d,h>0$.
	\end{lemma}
	\begin{lemma}\label{VRNonCS:LemmaBoundFunctionFSub}
		In Algorithm \ref{VRNonCS:AlgorithmI}, for the intermediated iteration at $x_k$ and $ \nabla {{\hat f}_k}$ defined in (\ref{VRNonCS:DefinitionPartialHatf}), we have
		\begin{align*}
		\mathbb{E}[\langle \nabla f({x_k}),\nabla {{\hat f}_k}\rangle ] \ge \frac{1}{2}\mathbb{E}[{\| {\nabla f({x_k})} \|^2}] - B_G^4L_F^2\frac{1}{{2A}}\mathbb{E}[{\| {{x_k} - \tilde x_s} \|^2}].
		\end{align*}
	\end{lemma}
	\begin{proof} Through adding and subtracting the term ${(\partial {G_{j_k}}({x_k}))^\mathsf{T}}\nabla {F_{i_k}}(G({x_k}))$, we have
		\begin{align*}
		&\mathbb{E}[\langle \nabla f({x_k}),\nabla {{\hat f}_k}\rangle ]\\
		=& \mathbb{E}[\langle \nabla f({x_k}),{(\partial {G_{j_k}}({x_k}))^\mathsf{T}}\nabla {F_{i_k}}({{\hat G}_k}) - {(\partial {G_{j_k}}({{\tilde x}_s}))^\mathsf{T}}\nabla {F_{i_k}}(G({{\tilde x}_s})) + \nabla f({{\tilde x}_s})]\\
		=& \mathbb{E}[\langle \nabla f({x_k}),{(\partial {G_{j_k}}({x_k}))^\mathsf{T}}\nabla {F_{i_k}}({{\hat G}_k})\rangle ]\\
		=& \mathbb{E}[\langle \nabla f({x_k}),{(\partial {G_{j_k}}({x_k}))^\mathsf{T}}\nabla {F_{i_k}}({{\hat G}_k}) - {(\partial {G_{j_k}}({x_k}))^\mathsf{T}}\nabla {F_{i_k}}(G({x_k})) + {(\partial {G_{j_k}}({x_k}))^\mathsf{T}}\nabla {F_{i_k}}(G({x_k}))\rangle ]\\
		=& \mathbb{E}[\langle \nabla f({x_k}),{(\partial {G_{j_k}}({x_k}))^\mathsf{T}}\nabla {F_{i_k}}(G({x_k}))\rangle ]\\& + \mathbb{E}[\langle \nabla f({x_k}),{(\partial {G_{j_k}}({x_k}))^\mathsf{T}}\nabla {F_{i_k}}({{\hat G}_k}) - {(\partial {G_{j_k}}({x_k}))^\mathsf{T}}\nabla {F_{i_k}}(G({x_k}))\rangle ]\\
		\mathop  \le \limits^{\scriptsize \textcircled{\tiny{1}}}& \mathbb{E}[{\| {\nabla f({x_k})} \|^2}] - \frac{1}{2}\mathbb{E}[{\| {\nabla f({x_k})} \|^2}] - \frac{1}{2}\mathbb{E}{\| {{(\partial {G_{j_k}}({x_k}))^\mathsf{T}}\nabla {F_{i_k}}({{\hat G}_k}) - {(\partial {G_{j_k}}({x_k}))^\mathsf{T}}\nabla {F_{i_k}}(G({x_k}))} \|^2}\\
		\mathop  \le \limits^{\scriptsize \textcircled{\tiny{2}}}&  \frac{1}{2}\mathbb{E}[{\| {\nabla f({x_k})} \|^2}] - B_G^2\frac{1}{2}\mathbb{E}[{\| {\nabla {F_{i_k}}({{\hat G}_k}) - \nabla {F_{i_k}}(G({x_k}))} \|^2}]\\
		\mathop  \le \limits^{\scriptsize \textcircled{\tiny{3}}}& \frac{1}{2}\mathbb{E}[{\| {\nabla f({x_k})} \|^2}] - B_G^2L_F^2\frac{1}{2}\mathbb{E}[{\| {{{\hat G}_k} - G({x_k})} \|^2}]\\
		\mathop  \le \limits^{\scriptsize \textcircled{\tiny{4}}}& \frac{1}{2}\mathbb{E}[{\| {\nabla f({x_k})} \|^2}] - B_G^4L_F^2\frac{1}{{2A}}\mathbb{E}[{\| {{x_k} - \tilde x_s} \|^2}],
		\end{align*}
		where inequality ${\small \textcircled{\scriptsize{1}}}$ use Lemma \ref{VRNonCS:AppendixLemmaInEquationWithab}; inequality ${\small \textcircled{\scriptsize{2}}}$ and ${\small \textcircled{\scriptsize{3}}}$ are based on the bounded Jacobian of $G_j$ and smoothness of $F_i$; inequality ${\small \textcircled{\scriptsize{4}}}$ use Lemma \ref{VRNonCS:LemmaBoundVarianceG}.
	\end{proof}
	Because of $\mathbb{E}[\nabla {{\hat f}_k} ]=\mathbb{E}[\nabla {{\tilde f}_k} ]=\mathbb{E}[{ \tilde\nabla_k}]$, we have the same bounds of $\mathbb{E}[\langle \nabla f({x_k}),\nabla {{\tilde f}_k}\rangle ]$ as in Lemma \ref{VRNonCS:LemmaBoundFunctionFSub}.
	\begin{lemma}\label{VRNonCS:LemmaBoundFunctionFFullSub}
		In Algorithm \ref{VRNonCS:AlgorithmII}, for the intermediated iteration at $x_k$ and $ \nabla {{\tilde f}_k}$ defined in (\ref{VRNonCS:DefinitionPartialTildef}), we have
		\begin{align*}
		\mathbb{E}[\langle \nabla f({x_k}),\nabla {{\tilde f}_k}\rangle ] \ge \frac{1}{2}\mathbb{E}[{\| {\nabla f({x_k})} \|^2}] - B_G^4L_F^2\frac{1}{{2A}}\mathbb{E}[{\| {{x_k} - \tilde x_s} \|^2}].
		\end{align*}
	\end{lemma}
	\begin{lemma}\label{VRNonCS:LemmaBoundFunctionFFullSubMiniBatch}
		In Algorithm \ref{VRNonCS:AlgorithmII}, for the intermediated iteration at $x_k$ and $ \tilde\nabla_k$ defined in (\ref{VRNonCS:DefinitionTildeNabla1}) and (\ref{VRNonCS:DefinitionTildeNabla2}), we have
		\begin{align*}
		\mathbb{E}[\langle \tilde\nabla_k,\nabla {{\tilde f}_k}\rangle ] \ge \frac{1}{2}\mathbb{E}[{\| {\nabla f({x_k})} \|^2}] - B_G^4L_F^2\frac{1}{{2A}}\mathbb{E}[{\| {{x_k} - \tilde x_s} \|^2}].
		\end{align*}
	\end{lemma}

	\subsection{The bound of $\mathbb{E}[ {\| {{x_{k + 1}} - \tilde x_s} \|^2} ]$ and $\mathbb{E}[ {f( {{x_{k + 1}}} )} ]$}
	The bound of $\mathbb{E}[ {\| {{x_{k + 1}} - \tilde x_s} \|^2} ]$ and $\mathbb{E}[ {f( {{x_{k + 1}}} )} ]$  are used to organize the convergence formulation to obtain Lemma \ref{VRNonCS:LemmaBoundFunctionFFullMiniBatch}, which is a general process to obtain the convergence rate of $\mathbb{E}[{\left\| {\nabla f({x_k})} \right\|^2}]$.
	\subsubsection{The bound of $\mathbb{E}[ {\| {{x_{k + 1}} - \tilde x_s} \|^2} ]$ }
	We give three different kinds of bounds for three proposed algorithms. Each algorithm has different parameters such as the sampling times $A$ and $B$ and the mini-batch size of outer sub-function.
	
	\begin{lemma}\label{VRNonCS:LemmaBoundvariableX}
		In algorithm \ref{VRNonCS:AlgorithmI}, $\mathbb{E}[ {\| {{x_{k + 1}} - \tilde x_s} \|^2} ]$ { can be bounded by}
		\begin{align*}
		\mathbb{E}[ {\| {{x_{k + 1}} - \tilde x_s} \|^2} ] \le& \left( {1 + 2\left( {\frac{1}{h} + \frac{1}{d} + d{B_G^4L_F^2}\frac{1}{A}} \right){\eta} + 4\left( {2L_f^2 + {B_G^4L_F^2}\frac{1}{A}} \right)\eta^2} \right)\mathbb{E}[\| {{x_k} - \tilde x_s} \|^2] 
		\\&+ ( {{\eta}h + 4\eta^2} )\| {\nabla f( {{x_k}} )} \|^2.
		\end{align*}
	\end{lemma}
	\begin{proof}
		Based on the update of $x_{k+1}$, we have,
		\begin{align}
		\| {{x_{k + 1}} - \tilde x_s} \|^2 &= \| {{x_k} - {\eta}\nabla {{\hat f}_k} - \tilde x_s} \|^2\nonumber \\ %
		\label{VRNonCS:LemmaBoundvariableXDefinitionInequality}
		&= \| {{x_k} - \tilde x_s} \|^2 - 2{\eta}\langle {\nabla {{\hat f}_k},{x_k} - \tilde x_s} \rangle  + \eta^2\| {\nabla {{\hat f}_k}} \|^2.
		\end{align}
		Taking expectation with respect to $i_k$ and $j_k$ on both sides of (\ref{VRNonCS:LemmaBoundvariableXDefinitionInequality}), we get,
		\begin{align*}
		&\mathbb{E}[ {\| {{x_{k + 1}} - \tilde x_s} \|^2} ]\\
		=&\mathbb{E}[\| {{x_k} - \tilde x_s} \|^2] - 2{\eta}\mathbb{E}[ {\langle {\nabla {{\hat f}_k},{x_k} - \tilde x_s} \rangle } ] + \eta^2\mathbb{E}[ {\| {\nabla {{\hat f}_k}} \|^2} ]\\
		\le &\mathbb{E}[\| {{x_k} - \tilde x_s} \|^2] + 2{\eta}\left(   h\frac{1}{2}{\left\| {\nabla f({x_k})} \right\|^2} + \left( {\frac{1}{2h} + \frac{1}{2d}+ dB_G^4L_F^2\frac{1}{2A} } \right)\mathbb{E}[{\| {{x_k} - {{\tilde x}_s}} \|^2}]\right)\\
		&+ \eta^2\left( {4\mathbb{E}[ {\| {\nabla f( {{x_k}} )} \|^2} ] + 4\left( {2L_f^2 + {B_G^4L_F^2}\frac{1}{A}} \right)\mathbb{E}[ {\| {{x_k} - \tilde x_s} \|^2} ]} \right)\\
		=& \left( {1 + \left( {\frac{1}{h} + \frac{1}{d} + d{B_G^4L_F^2}\frac{1}{A}} \right){\eta} + 4\left( {2L_f^2 + {B_G^4L_F^2}\frac{1}{A}} \right)\eta^2} \right)\mathbb{E}[\| {{x_k} - \tilde x_s} \|^2 ]+ ( {{\eta}h + 4\eta^2} )\| {\nabla f( {{x_k}} )} \|^2,
		\end{align*}
		where the inequality follows from Lemma \ref{VRNonCS:LemmaBoundvariableXSub} and Lemma \ref{VRNonCS:LemmaBoundVarianceEstimatGradient}.
	\end{proof}
	\begin{lemma}\label{VRNonCS:LemmaBoundvariableXFull}
		In algorithm \ref{VRNonCS:AlgorithmII}, $E[ {\| {{x_{k + 1}} - \tilde x_s} \|^2} ]$ { can be bounded by}
		\begin{align*}
		&\mathbb{E}[ {\| {{x_{k + 1}} - \tilde x_s} \|^2} ]\\
		\le&\left( {1 + 2\left( {\frac{1}{h} + \frac{1}{d} + dB_G^4L_F^2\frac{{1}}{A}} \right){\eta} + 4\left( {B_G^4L_F^2\frac{1}{A} + B_F^2L_G^2\frac{1}{B} + L_f^2} \right)\eta^2} \right)\mathbb{E}[\| {{x_k} - \tilde x_s} \|^2]
		\\&+ ( {{\eta}h + 4\eta^2} )\| {\nabla f( {{x_k}} )} \|^2.
		\end{align*}
	\end{lemma}
	\begin{proof}
		Taking expectation with respect to $i_k$ and $j_k$, we get,
		\begin{align*}
		&\mathbb{E}[ {\| {{x_{k + 1}} - \tilde x_s} \|^2} ]\\
		=& \| {{x_k} - \tilde x_s} \|^2 - 2{\eta}E[ {\langle {\nabla {{\tilde f}_k}( {{x_k}} ),{x_k} - \tilde x_s} \rangle } ] + \eta^2\mathbb{E}[ {\| {\nabla {{\tilde f}_k}} \|^2} ]\\
		\le &\| {{x_k} - \tilde x_s} \|^2 + 2{\eta}\left(  h\frac{1}{2}{\left\| {\nabla f({x_k})} \right\|^2} +\left( {\frac{1}{2h} + \frac{1}{2d}+ dB_G^4L_F^2\frac{1}{2A} } \right)\mathbb{E}[{\| {{x_k} - {{\tilde x}_s}} \|^2}]\right)\\
		&+ \eta^2\left( 4E[{\| {\nabla f({x_k})} \|^2}] + 4\left( {B_G^4L_F^2\frac{1}{A} + B_F^2L_G^2\frac{1}{B} + L_f^2} \right)\mathbb{E}[{\| {{x_k} - {{\tilde x}_s}} \|^2}] \right)\\
		=& \left( {1 + \left( {\frac{1}{h} + \frac{1}{d} + dB_G^4L_F^2\frac{{1}}{A}} \right){\eta} + 4\left( {B_G^4L_F^2\frac{1}{A} + B_F^2L_G^2\frac{1}{B} + L_f^2} \right)\eta^2} \right)\mathbb{E}[\| {{x_k} - \tilde x_s} \|^2]\\& + \left( {{\eta}h + 4\eta^2} \right)\left\| {\nabla f\left( {{x_k}} \right)} \right\|^2,
		\end{align*}
		where the inequality follows from Lemma \ref{VRNonCS:LemmaBoundvariableXFullSub} and Lemma \ref{VRNonCS:LemmaBoundVarianceEstimatGradientfull}.
	\end{proof}
	\begin{lemma}\label{VRNonCS:LemmaBoundvariableXFullMiniBatch}
		In algorithm \ref{VRNonCS:AlgorithmMiniBatch}, let  $h,d>0$ and $b\ge 1$, $E[ {\| {{x_{k + 1}} - \tilde x_s} \|^2} ]$ { can be bounded by}
		\begin{align*}
		&\mathbb{E}[ {\| {{x_{k + 1}} - \tilde x_s} \|^2} ]\\
		\le&\left( {1 + 2\left( {\frac{1}{h} + \frac{1}{d} + dB_G^4L_F^2\frac{{1}}{A}} \right){\eta}\
			+ 4\left( {B_G^4L_F^2\frac{1}{A} + B_F^2L_G^2\frac{1}{B} + L_f^2} \right){\frac{1}{b}}\eta^2} \right)\mathbb{E}[\| {{x_k} - \tilde x_s} \|^2]\\
		&( {{\eta}h + 4\eta^2} )\| {\nabla f( {{x_k}} )} \|^2. 
		\end{align*}
	\end{lemma}
	\begin{proof}
		Taking expectation with respect to $i_k$ and $j_k$, we get,
		\begin{align*}
		&\mathbb{E}[ {\| {{x_{k + 1}} - \tilde x_s} \|^2} ]\\
		=& \| {{x_k} - \tilde x_s} \|^2 - 2{\eta}E[ {\langle {\nabla {{\tilde f}_k}( {{x_k}} ),{x_k} - \tilde x_s} \rangle } ] + \eta^2\mathbb{E}[ {\| {\nabla {{\tilde f}_k}} \|^2} ]\\
		\le &\left( {1 + \left( {\frac{1}{h} + \frac{1}{d} + dB_G^4L_F^2\frac{{1}}{A}} \right){\eta} + 4\left( {B_G^4L_F^2\frac{1}{A} + B_F^2L_G^2\frac{1}{B} + L_f^2} \right){\frac{1}{b}}\eta^2} \right)\mathbb{E}[\| {{x_k} - \tilde x_s} \|^2]\\& + ( {{\eta}h + 4\eta^2} )\| {\nabla f( {{x_k}} )} \|^2,
		\end{align*}
		where the inequality follows from Lemma \ref{VRNonCS:LemmaBoundvariableXFullSubMiniBatch} and Lemma \ref{VRNonCS:LemmaBoundVarianceEstimatGradientfullMiniBatch}.
	\end{proof}
	\subsubsection{The bound of  $\mathbb{E}[ {f( {{x_{k + 1}}} )} ]$ }
	Based on different algorithms, we also give three different upper bounds for  $\mathbb{E}[ {f( {{x_{k + 1}}} )} ]$, which are used for analyzing the convergence rate.
	\begin{lemma}\label{VRNonCS:LemmaBoundFunctionF}
		In algorithm \ref{VRNonCS:AlgorithmI}, $\mathbb{E}[ {f( {{x_{k + 1}}} )} ]$  can be bounded by,
		\begin{align*}
		\mathbb{E}[ {f( {{x_{k + 1}}} )} ] \le&  \mathbb{E}[f({x_k})] - \left( {\frac{1}{2}\eta  - 2{L_f}{\eta ^2}} \right)\mathbb{E}[{\left\| {\nabla f({x_k})} \right\|^2}] \\
		&+ \left( {B_G^4L_F^2\frac{1}{{2A}}\eta  + 2{L_f}\left( {2L_f^2 + B_G^4L_F^2\frac{1}{A}} \right){\eta ^2}} \right)\mathbb{E}[{\left\| {{x_k} - \tilde x_s} \right\|^2}],
		\end{align*}
	\end{lemma}
	\begin{proof}
		Based on the smoothness of $f(x)$, we have,
		\begin{align}
		f( {{x_{k + 1}}} ) \le& f( {{x_k}} ) + \langle {\nabla f( {{x_k}} ),{x_{k + 1}} - {x_k}} \rangle  + \frac{{{L_f}}}{2}\| {{x_{k + 1}} - {x_k}} \|^2\nonumber \\
		\label{VRNonCS:LemmaBoundFunctionFDefinitionInequality}
		=& f( {{x_k}} ) - {\eta}\langle {\nabla f( {{x_k}} ), \nabla {\hat f_k}} \rangle  + \frac{{{L_f}}}{2}\eta^2\| { \nabla {\hat f_k}} \|^2.
		\end{align}
		Taking expectation with respect to $i_k$ and $j_k$ on both sides of (\ref{VRNonCS:LemmaBoundFunctionFDefinitionInequality}), we have
		\begin{align*}
		\mathbb{E}[ {f( {{x_{k + 1}}} )} ] \le& \mathbb{E}[ {f( {{x_k}} )} ] { - {\eta}\mathbb{E}[ {\langle {\nabla f( {{x_k}} ),\nabla {\hat f}_k} \rangle } ]} + \frac{{{L_f}}}{2}\eta^2 {\mathbb{E}[ {\| {\nabla {{\hat f}_k}} \|^2} ]}\\
		\le& \mathbb{E}[ {f( {{x_k}} )} ]+\eta\left (-\frac{1}{2}\mathbb{E}[{\| {\nabla f({x_k})} \|^2}] + B_G^4L_F^2\frac{1}{{2A}}\mathbb{E}[{\| {{x_k} - \tilde x_s} \|^2}]\right)\\
		&+ \frac{{{L_f}}}{2}\eta^2\left( 4\mathbb{E}[ {\| {\nabla f( {{x_k}} )} \|^2} ] + 4\left( {2L_f^2 + {B_G^4L_F^2}\frac{1}{A}} \right)\mathbb{E}[ {\| {{x_k} - \tilde x_s} \|^2} ] \right)\\
		=&  \mathbb{E}[f({x_k})] - \left( {\frac{1}{2}\eta  - 2{L_f}{\eta ^2}} \right)\mathbb{E}[{\left\| {\nabla f({x_k})} \right\|^2}]\\
		&+ \left( {B_G^4L_F^2\frac{1}{{2A}}\eta  + 2{L_f}\left( {2L_f^2 + B_G^4L_F^2\frac{1}{A}} \right){\eta ^2}} \right)\mathbb{E}[{\| {{x_k} - \tilde x_s} \|^2}],
		\end{align*}
		where the second inequality follows from Lemma \ref{VRNonCS:LemmaBoundFunctionFSub} and Lemma \ref{VRNonCS:LemmaBoundVarianceEstimatGradient}.
	\end{proof}
	\begin{lemma}\label{VRNonCS:LemmaBoundFunctionFFull}
		In algorithm \ref{VRNonCS:AlgorithmII}, $\mathbb{E}[ {f( {{x_{k + 1}}} )} ]$  can be bounded by,
		\begin{align*}
		\mathbb{E}[ {f( {{x_{k + 1}}} )} ] \le& E[f({x_k})] - \left( {\frac{1}{2}\eta  - 2{L_f}{\eta ^2}} \right)\mathbb{E}[{\left\| {\nabla f({x_k})} \right\|^2}]\\
		&+ \left( {B_G^4L_F^2\frac{1}{{2A}}\eta  + 2{L_f}\left( {B_G^4L_F^2\frac{1}{A} + B_F^2L_G^2\frac{1}{B} + L_f^2} \right){\eta ^2}} \right)\mathbb{E}[{\| {{x_k} - \tilde x_s} \|^2}],
		\end{align*}
	\end{lemma}
	\begin{proof}
		Based on the smoothness of $f(x)$, we have,
		taking expectation with respect to $i$ and $j$, we have
		\begin{align*}
		\mathbb{E}[ {f( {{x_{k + 1}}} )} ] \le& \mathbb{E}[ {f( {{x_k}} )} ] { - {\eta}\mathbb{E}[ {\langle {\nabla f( {{x_k}} ),\nabla {\tilde f}_k} \rangle } ]} + \frac{{{L_f}}}{2}\eta^2 {\mathbb{E}[ {\| {\nabla {{\tilde f}_k}} \|^2} ]}\\
		\le& \mathbb{E}[ {f( {{x_k}} )} ]+\eta\left (\frac{1}{2}E[{\| {\nabla f({x_k})} \|^2}] + B_G^4L_F^2\frac{1}{{2A}}\mathbb{E}[{\| {{x_k} - \tilde x_s} \|^2}]\right)\\
		&+ \frac{{{L_f}}}{2}\eta^2\left(  4\mathbb{E}[{\| {\nabla f({x_k})} \|^2}] + 4\left( {B_G^4L_F^2\frac{1}{A} + B_F^2L_G^2\frac{1}{B} + L_f^2} \right)\mathbb{E}[{\| {{x_k} - {{\tilde x}_s}} \|^2}]\right)\\
		=&  \mathbb{E}[f({x_k})] - \left( {\frac{1}{2}\eta  - 2{L_f}{\eta ^2}} \right)\mathbb{E}[{\| {\nabla f({x_k})} \|^2}]\\
		&+ \left( {B_G^4L_F^2\frac{1}{{2A}}\eta  + 2{L_f}\left( {B_G^4L_F^2\frac{1}{A} + B_F^2L_G^2\frac{1}{B} + L_f^2} \right){\eta ^2}} \right)\mathbb{E}[{\| {{x_k} - \tilde x_s} \|^2}],
		\end{align*}
		where the second inequality follows from Lemma \ref{VRNonCS:LemmaBoundFunctionFFullSub} and Lemma \ref{VRNonCS:LemmaBoundVarianceEstimatGradientfull}.
	\end{proof}
	\begin{lemma}\label{VRNonCS:LemmaBoundFunctionFFullMiniBatch}
		In algorithm \ref{VRNonCS:AlgorithmMiniBatch}, $\mathbb{E}[ {f( {{x_{k + 1}}} )} ]$  can be bounded by,
		\begin{align*}
		\mathbb{E}[ {f( {{x_{k + 1}}} )} ] \le& \mathbb{E}[f({x_k})] - \left( {\frac{1}{2}\eta  - 2{L_f}{\eta ^2}} \right)\mathbb{E}[{\left\| {\nabla f({x_k})} \right\|^2}]\\
		&+ \left( {B_G^4L_F^2\frac{1}{{2A}}\eta  + 2{L_f}\left( {B_G^4L_F^2\frac{1}{A} + B_F^2L_G^2\frac{1}{B} + L_f^2} \right){\frac{1}{b}}{\eta ^2}} \right)\mathbb{E}[{\| {{x_k} - \tilde x_s} \|^2}],
		\end{align*}
	\end{lemma}
	\begin{proof}
		Based on the smoothness of $f(x)$, we have,
		taking expectation with respect to $i_k$ and $j_k$, we have
		\begin{align*}
		E[ {f( {{x_{k + 1}}} )} ] \le& E[ {f( {{x_k}} )} ] { - {\eta}E[ {\langle {\nabla f( {{x_k}} ),\tilde\nabla_k } \rangle } ]} + \frac{{{L_f}}}{2}\eta^2 {E[ {\| {\nabla {{\tilde f}_k}} \|^2} ]}\\
		\le&  E[f({x_k})] - \left( {\frac{1}{2}\eta  - 2{L_f}{\eta ^2}} \right)E[{\| {\nabla f({x_k})} \|^2}]\\
		&+ \left( {B_G^4L_F^2\frac{1}{{2A}}\eta  + 2{L_f}\left( {B_G^4L_F^2\frac{1}{A} + B_F^2L_G^2\frac{1}{B} + L_f^2} \right){\frac{1}{b}}{\eta ^2}} \right)E[{\| {{x_k} - \tilde x_s} \|^2}],
		\end{align*}
		where the second inequality follows from Lemma \ref{VRNonCS:LemmaBoundFunctionFFullSubMiniBatch} and Lemma \ref{VRNonCS:LemmaBoundVarianceEstimatGradientfullMiniBatch}.
	\end{proof}
	
	\section{Proof of convergence and complexity analyses}\label{VRNonCS:SectionConvergenceAnalysis}
	In this section, we give the details proof for the convergence analysis and query complexity for three proposed algorithms: SCVRI, SCVRII and mini-batch SCVR. The proof processes are similar but with different parameters setting, such that result in different query complexities.

	\subsection{Stochastic Composition with Variance reduction I}
	We give the following proofs concerning the SCVRI method with convergence rate and query complexity detail analysis. Theorem \ref{VRNonCS:TheoremMain} and Corollary \ref{VRNonCS:CorollaryComplexityEstimation} provide two main and basic proofs. Other proofs such as in Corollary \ref{VRNonCS:CorollaryComplexityEstimationinfinit}, Theorem \ref{VRNonCS:TheoremMainFull}, Corollary \ref{VRNonCS:CorollaryComplexityEstimationFull}, and Corollary \ref{VRNonCS:CorollaryComplexityEstimationFullMiniBatch} are based on basic proof, but with different estimators and parameters setting.

	\textbf{Proof of Theorem \ref{VRNonCS:TheoremMain}:}
	\begin{proof}
		Based on Lemma \ref{VRNonCS:LemmaBoundFunctionF} and Lemma \ref{VRNonCS:LemmaBoundvariableX}, we form a Lyapunov function,
		\begin{align*}
		&\mathbb{E}[ {f( {{x_{k + 1}}} )} ] + {c_{k + 1}}\mathbb{E}[ {\| {{x_{k + 1}} - \tilde x_s} \|^2} ]\\
		=& \mathbb{E}[f({x_k})] - \left( {\frac{1}{2}\eta  - 2{L_f}{\eta ^2}} \right)\mathbb{E}[{\left\| {\nabla f({x_k})} \right\|^2}]+ \left( {B_G^4L_F^2\frac{1}{{2A}}\eta  + 2{L_f}\left( {2L_f^2 + B_G^4L_F^2\frac{1}{A}} \right){\eta ^2}} \right)\mathbb{E}[{\left\| {{x_k} - \tilde x_s} \right\|^2}]\\
		&+c_{k+1}\left(\left( {1 + \left( {\frac{1}{h} + \frac{1}{d} + d{B_G^4L_F^2}\frac{1}{A}} \right){\eta} + 4\left( {2L_f^2 + {B_G^4L_F^2}\frac{1}{A}} \right)\eta^2} \right)\left\| {{x_k} - \tilde x_s} \right\|^2 + ( {{\eta}h + 4\eta^2} )\| {\nabla f( {{x_k}} )} \|^2\right)\\
		=& \mathbb{E}[f({x_k})] + {c_k}\mathbb{E}[{\| {{x_k} - {{\tilde x}_s}} \|^2}]- {u_k}\mathbb{E}[{\| {\nabla f({x_k})} \|^2}] ,
		\end{align*}
		where 
		\begin{align*}
		{u_k} =& \left( {1/2 - {c_{k + 1}}h} \right)\eta  - \left( {2{L_f} + {4c_{k + 1}}} \right){\eta ^2},\\
		{c_k} =& {c_{k + 1}}\left( {1 + \left( {\frac{1}{h} + \frac{1}{d} + dB_G^4L_F^2\frac{1}{A}} \right)\eta  + 4\left( {2L_f^2 + B_G^4L_F^2\frac{1}{A}} \right){\eta ^2}} \right)
		+ B_G^4L_F^2\frac{1}{{2A}}\eta\\  
		&+ 2{L_f}\left( {2L_f^2 + B_G^4L_F^2\frac{1}{A}} \right){\eta ^2}. 
		\end{align*}
		Then, we get 
		\begin{align*}
		{u_k}\mathbb{E}[ {\| {\nabla f( {{x_k}} )} \|^2} ] \le E[ {f( {{x_k}} )} ] + {c_k}\mathbb{E}[ {\| {{x_k} - \tilde x_s} \|^2} ] - ( {\mathbb{E}[ {f( {{x_{k + 1}}} )} ] + {c_{k + 1}}\mathbb{E}[ {\| {{x_{k + 1}} - \tilde x_s} \|^2} ]} ).
		\end{align*}
		Define $u = {\min _{0 \le k \le K - 1}}\{ {u_k}\} $, sum from $k=0$ to $k=K-1$, we can get
		\begin{align*}
		\frac{1}{K}\sum\limits_{k = 0}^{K-1} {\mathbb{E}[ {\| {\nabla f( {{x_k}} )} \|^2} ]}  &\le \frac{{\mathbb{E}[ {f( {{x_0}} )} ] - ( {\mathbb{E}[ {f( {{x_K}} )} ] + {c_K}\mathbb{E}[ {\| {{x_K} - \tilde x_s} \|^2} ]} )}}{{uK}}\\
		&\le \frac{{\mathbb{E}[ {f( {{x_0}} )} ] -\mathbb{E}[ {f( {{x_K}} )} ]}}{{uK}}.
		\end{align*}
		Since $x_0=\tilde{x}_s$, let $\tilde{x}_{s+1}=x_K$, we obtain,
		\begin{align*}
		\frac{1}{K}\sum\limits_{k = 0}^{K-1} {\mathbb{E}[ {\| {\nabla f( {{x_k}} )} \|^2} ]} \le \frac{{\mathbb{E}[ {f( {{{\tilde x}_s}} )} ] - \mathbb{E}[ {f( {{{\tilde x}_{s + 1}}} )} ]}}{uK}.
		\end{align*} 
		Summing the outer iteration from $s=0$ to $S-1$, we have
		\begin{align*}
		\mathbb{E}[{\| {\nabla f({\tilde x_k^s })} \|^2}] = \frac{1}{S}\sum\limits_{s = 0}^{S - 1} {\frac{1}{K}\sum\limits_{k = 0}^{K - 1} {\mathbb{E}[{{\| {\nabla f(x_k^s)} \|}^2}]} }  \le \frac{{\mathbb{E}[f({{\tilde x}_0})] - \mathbb{E}[f({{\tilde x}_S})]}}{{uKS}} \le \frac{{f({x_0}) - f({x^*})}}{{uKS}},
		\end{align*}
		where $x_k^s$ indicates the $s$-th outer iteration at $k$-th inner iteration, and $\tilde x_k^s$ is uniformly and randomly chosen from  $s=\{0,...,S-1\}$ and k=$\{0,..,K-1\}$.
	\end{proof}
	
	\textbf{Proof of Corollary \ref{VRNonCS:CorollaryComplexityEstimation}:}
	\begin{proof}Firstly, we consider the parameter setting in Theorem \ref{VRNonCS:TheoremMain}. To analyze the bound of $c_0$, we use sequence in $c_k$ (\ref{VRNonCS:DefinitionCk}) and define ${c_{k - 1}} = {c_k}Y + U$. Based on Corollary \ref{VRNonCS:LemmaGeometriProgression}, we have
		\begin{align}
		{c_K} = {\left( {\frac{1}{Y}} \right)^K}\left( {{c_0} + \frac{U}{{Y - 1}}} \right) - \frac{U}{{Y - 1}},
		\end{align}
		where
		\begin{align*}
		Y=& {1 + \left( {\frac{1}{h} + \frac{1}{d} + dB_G^4L_F^2\frac{1}{A}} \right)\eta  + 4\left( {2L_f^2 + B_G^4L_F^2\frac{1}{A}} \right){\eta ^2}} >1 ,\\
		U=&{B_G^4L_F^2\frac{1}{{2A}}\eta  + 2{L_f}\left( {2L_f^2 + B_G^4L_F^2\frac{1}{A}} \right){\eta ^2}}>0.
		\end{align*}
		Setting  $c_K=0$, we obtain
		\begin{align*}
		{c_0} = \frac{{U{Y^K}}}{{Y - 1}} - \frac{U}{{Y - 1}} = \frac{{U\left( {{Y^K} - 1} \right)}}{{Y - 1}}.
		\end{align*}		
		Then, putting the Y and U  into the above equation. We have
		\begin{align}
		{c_0} =& \frac{{( {B_G^4L_F^2\frac{1}{2A}\eta  + 2{L_f}( {2L_f^2 + B_G^4L_F^2\frac{1}{A}} ){\eta ^2}} )( {{{( {1 + 2( {\frac{1}{h} + \frac{1}{d} + dB_G^4L_F^2\frac{1}{A}} )\eta  + 4( {2L_f^2 + B_G^4L_F^2\frac{1}{A}} ){\eta ^2}} )}^K} - 1} )}}{{( {\frac{1}{h} + \frac{1}{d} + dB_G^4L_F^2\frac{1}{A}} )\eta  + 4( {2L_f^2 + B_G^4L_F^2\frac{1}{A}} ){\eta ^2}}}\nonumber\\
		\label{VRNonCS:Definitionc_0}
		=& \frac{{{n^{ - {A_0}}} + {n^{ - \alpha }}}}{{( {( {e - 1} ){n^{ - {h_0}}} + {n^{ - {d_0}}} + {2n^{ - ( {{A_0} - {d_0}} )}}} ) + 2{n^{ - \alpha }/L_f}}}C,
		\end{align}
		where
		\begin{align}
		\label{VRNonCS:DefinitionParametersAll}
		A =& \frac{1}{2}B_G^4L_F^2{n^{{A_0}}},h = \frac{1}{e-1}{n^{{h_0}}},d = {n^{{d_0}}},\eta  = \frac{1}{{2{L_f}( {2L_f^2 + B_G^4L_F^2\frac{1}{A}} )}}{n^{ - \alpha }},\\
		\label{VRNonCS:DefinitionC}
		C =& {\left( {1 + \left( {\frac{1}{h} + \frac{1}{d} + dB_G^4L_F^2\frac{1}{A}} \right)\eta  + 4\left( {2L_f^2 + B_G^4L_F^2\frac{1}{A}} \right){\eta ^2}} \right)^K} - 1.
		\end{align}
		
		As shown in (\ref{VRNonCS:DefinitionCk}), $c_k$ is a decrease sequence. $u_k$ is defined in (\ref{VRNonCS:DefinitionUk}), we have,
		\begin{align*}
		{u_k} &= ( 1/2 + {c_{k + 1}}h  ){\eta} - ( {2{L_f} + 4{c_{k + 1}}} )\\
		&\ge (  {1/2 - {c_0}h}  ){\eta} - ( {2{L_f} + 4{c_0}} )\eta^2,\,\,\,\,\,\,\,\forall k > 0.
		\end{align*}
		In order to keep the lower bound of $u_k$ positive as  the denominator, $c_0h$  should satisfy ${c_0}h <1/2$. Thus, we set
		\begin{align}\label{VRNonCS:DefinitionParameters}
		{A_0} = \alpha ,{h_0} = {d_0} = \alpha/2,
		\end{align}
		which has no influence by  $n$. Then, $c_0h$ becomes
		\begin{align}
		{c_0}h =& \frac{{{n^{ - {A_0}}} + {n^{ - \alpha }}}}{{( ( {e - 1} ){{n^{ - {h_0}}} + {n^{ - {d_0}}} + 2{n^{ - ( {{A_0} - {d_0}} )}}} ) + {{2{n^{ - \alpha }}}/L_f}}}Ch\nonumber\\
		=&  \frac{{{n^{ - {A_0} + {h_0}}} + {n^{ - \alpha  + {h_0}}}}}{{( ( {e - 1} ){{n^{ - {h_0}}} + {n^{ - {d_0}}} + 2{n^{ - ( {{A_0} - {d_0}} )}}} ) + {{2{n^{ - \alpha }}}/L_f}}}C\frac{1}{{e - 1}}\nonumber\\
		=& \frac{{{n^{ - \frac{1}{2}\alpha }} + {n^{ - \frac{1}{2}\alpha }}}}{{( {e + 2} ){n^{ - \frac{1}{2}\alpha }} + {{2{n^{ - \alpha }}}/L_f}}}C\frac{1}{{e - 1}}\nonumber\\
		\label{VRNonCS:Definitionc_0h}
		=& \frac{2}{{( {e + 2} ) + {{2{n^{ - \frac{1}{2}\alpha }}}/L_f}}}C\frac{1}{{e - 1}}=\frac{1}{{( {e/2 + 1} ) + {{{n^{ - \frac{1}{2}\alpha }}}/L_f}}}C\frac{1}{{e - 1}}.
		\end{align}
		Based on the setting  in (\ref{VRNonCS:DefinitionParametersAll}), (\ref{VRNonCS:DefinitionC}) and  (\ref{VRNonCS:DefinitionParameters}),  we can obtain 
		\begin{align*}
		C =& {\left( {1 + \left( {\frac{1}{h} + \frac{1}{d} + dB_G^4L_F^2\frac{1}{A}} \right)\eta  + 4\left( {2L_f^2 + B_G^4L_F^2\frac{1}{A}} \right){\eta ^2}} \right)^K} - 1\\
		=& {( {1 + ( {( {e - 1} ){n^{ - {h_0}}} + {n^{ - {d_0}}} + {n^{ - {A_0} + {d_0}}}} )\eta  + 4(L_f^2 + 2{n^{ - A_0 }}){\eta ^2}} )^K} - 1\\
		=& {( {1 + ( {( {e - 1} ){n^{ - \frac{1}{2}\alpha }} + {n^{ - \frac{1}{2}\alpha }} + {n^{ - \frac{1}{2}\alpha }}} )\eta  + 4(L_f^2 + 2{n^{ - \alpha }}){\eta ^2}} )^K} - 1\\
		=& {( {1 + (e + 1){n^{ - \frac{1}{2}\alpha }}\eta  + 4(L_f^2 + 2{n^{ - \alpha }}){\eta ^2}} )^K} - 1\\
		<& e - 1,
		\end{align*}
		where the last equation is from the  character of function ${\left( {1 + \frac{1}{t}} \right)^t} \to e$, as $t \to  + \infty $, and the function is also the increase function with an upper bound of $e$. There exist a constant $w_1>0$ such that
		\begin{align*}
		K = 1/((e + 1){n^{ - \frac{1}{2}\alpha }}\eta  + 4(L_f^2 + 2{n^{ - \alpha }})){\eta ^2}) = w_1L_f^3{n^{\frac{3}{2}\alpha }}.
		\end{align*}
		Thus, we  obtain $
		{c_0}h < {1}/({( {e/2 + 1} ) + n^{ - \frac{1}{2}\alpha }/L_f }) <{1}/{2},$
		which satisfy $  {1/2 -{c_0}h}  >0$.
		
		For $u_k$ defined in (\ref{VRNonCS:DefinitionUk}) and $\eta$ define in (\ref{VRNonCS:DefinitionParameters}),  there also exists a constant $v_1>0$ such that satisfies, 
		\begin{align}\label{VRNonCS:DefinitionMaxu}
		u = \mathop {\max }\limits_k \{ {{u_k}} \} \ge ( {{1/2- {c_0}h} } )\eta  - ( {2{L_f} + 4{c_0}} ){\eta ^2} ={n^{ - \alpha }}v_1/L_f^3.
		\end{align}
		Combine with Theorem \ref{VRNonCS:TheoremMain}, we have
		\begin{align*}
		\mathbb{E}[{\| {\nabla f({\tilde x_k^s })} \|^2}]   \le \frac{{{n^\alpha }L_f(f({x_0}) - f({x^*}))}}{{v_1SK}}.
		\end{align*}
	\end{proof}
	
	\textbf{Proof of Corollary \ref{VRNonCS:CorollaryComplexityEstimationinfinit}:}
	\begin{proof}Similar to the proof in Corollary \ref{VRNonCS:CorollaryComplexityEstimation}, consider  the case that sample times $A=+\infty $, and set $h = {{{n^{{h_0}}}} \mathord{/{\vphantom {{{n^{{h_0}}}} {( {e - 1} )}}} \kern-\nulldelimiterspace} {( {e - 1} )}},d = {n^{{d_0}}},\eta  = {{{n^{ - \alpha }}} \mathord{/{\vphantom {{{n^{ - \alpha }}} {4L_f^3}}} 	\kern-\nulldelimiterspace} {4L_f^3}}$, we have
		\begin{align*}
		{c_0} =& \frac{{B_G^4L_F^2\frac{1}{2A}\eta  + 2{L_f}( {2L_f^2 + B_G^4L_F^2\frac{1}{A}} ){\eta ^2}}}{{( {\frac{1}{h} + \frac{1}{d} + dB_G^4L_F^2\frac{1}{A}} )\eta  + 4( {2L_f^2 + B_G^4L_F^2\frac{1}{A}} ){\eta ^2}}}C\\
		\approx& \frac{{{n^{ - \alpha }}}}{{( {{n^{ - {h_0}}}( {e - 1} ) + {n^{ - {d_0}}}} ) + 2{n^{ - \alpha }}/{L_f}}}C,
		\end{align*}
		where
		\begin{align*}
		C = {\left( {1 + \left( {\frac{1}{h} + \frac{1}{d} + dB_G^4L_F^2\frac{1}{A}} \right)\eta  + 4\left( {2L_f^2 + B_G^4L_F^2\frac{1}{A}} \right){\eta ^2}} \right)^K} - 1.
		\end{align*}
		
		In order to keep the lower bound of $u_k$ positive as  denominator, $c_0h$  should satisfy ${c_0}h <1/2$, we set $\alpha  = 2{h_0}$
		\begin{align*}
		{c_0}h =& \frac{{{n^{ - \alpha  + {h_0}}}}}{{( {( {e - 1} ){n^{ - {h_0}}} + {n^{ - {d_0}}}} ) + 2{n^{ - \alpha }}/{L_f}}}C\frac{1}{{e - 1}}\\
		=& \frac{{ {n^{ - \frac{1}{2}\alpha }}}}{{e{n^{ - \frac{1}{2}\alpha }} + {{2{n^{ - \alpha }}}/L_f}}}C\frac{1}{{e - 1}}\\
		=& \frac{1}{{e + 2{{{n^{ - \frac{1}{2}\alpha }}}/L_f}}}C\frac{1}{{e - 1}} .
		\end{align*}
		Based on the  character of function ${\left( {1 + \frac{1}{t}} \right)^t} \to e$ as $t \to  + \infty $, there exist a constant $w_3>0$ such that $K = w_3{n^{  3\alpha /2}}$ and
		\begin{align*}
		C =& {\left( {1 + 2\left( {\frac{1}{h} + \frac{1}{d} + dB_G^4L_F^2\frac{1}{A}} \right)\eta  + 4\left( {2L_f^2 + B_G^4L_F^2\frac{1}{A}} \right){\eta ^2}} \right)^K} - 1\\
		=& {( {1 + 2( {{n^{ - {h_0}}}( {e - 1} ) + {n^{ - {d_0}}}} ){n^{ - \alpha }} + 8L_f^2{n^{ - 2\alpha }}} )^K} - 1\\
		=& {( {1 + 2e{n^{ - \frac{3}{2}\alpha }} + 8L_f^2{n^{ - 2\alpha }}} )^K} - 1< e - 1.
		\end{align*}
		Thus,$
		{c_0}h < 1/({{e + {n^{ - \alpha/2 }}{L_f}{\rm{ /}}{L_f}}}) <1/2.$	
		
		For u, there exist a constant $v_3>0$ such that 
		\begin{align*}
		u =&	\mathop {\min}\limits_{0\le k\le {K-1} }\{ {(1/2- {c_{k + 1}}h)\eta  - (2{L_f} + 4{c_{k + 1}}){\eta ^2}} \}\\
		\ge &((1/2-{c_0}h)\eta  - (2{L_f} + 4{c_0}){\eta ^2}) = v_3{n^{ - \alpha }}/4L_f^3.
		\end{align*}
		Based on Theorem \ref{VRNonCS:TheoremMain}, we obtain
		\begin{align*}
		\mathbb{E}[{\| {\nabla f({\tilde x_k^s })} \|^2}] \le \frac{{{4{n^\alpha }L_f^3} ( { {f( {{{ x}_0}} )}  - f( {{x^*}} )} )}}{{v_3SK}}.
		\end{align*}
	\end{proof}	
	The following gives proofs of the query complexity analysis. Corollary \ref{VRNonCS:CorollaryQueryComplexity} is the basic analysis process for query complexity. all other methods with different parameters analysis are based on Corollary \ref{VRNonCS:CorollaryQueryComplexity} .
	
	\textbf{Proof of Corollary \ref{VRNonCS:CorollaryQueryComplexity}:}
	\begin{proof} Based on Corollary \ref{VRNonCS:CorollaryComplexityEstimation},
		the number of inner iteration $K$ is $\mathcal{O}(n^{3\alpha/2} )$, then the number of outer iteration S is
		\begin{align*}
		S = \frac{T}{K} =  \mathcal{O}\left( {\frac{n^\alpha}{{\varepsilon {n^{{3\alpha/2}}}}}} \right) = \mathcal{O}\left( {\frac{1}{{\varepsilon {n^{{\alpha/2}}}}}} \right).
		\end{align*}
		Then, the query complexity is 
		\begin{align*}
		\mathcal{O}( {( {2m + n + 2KA + 4K} )S} ) =& \mathcal{O}( {( {m + n + {n^{\frac{3}{2}\alpha  + \alpha }} + {n^{\frac{3}{2}\alpha }}} )( {{{{n^{ - \frac{\alpha }{2}}}}}/{\varepsilon }} )} )\\
		=&\mathcal{O}( {( {m + n + {n^{\frac{5}{2}\alpha  }}  } )( {{{{n^{ - \frac{\alpha }{2}}}}}/{\varepsilon }} )} ).
		\end{align*}
	\end{proof}
	
	\textbf{Proof of Corollary \ref{VRNonCS:CorollaryComplexityQCCS}:}
	\begin{proof}
		When the size of $G_i$ is $m=n^{m_0}$, $m_0>0$, the query complexity  of composition stochastic  (QCCS) problem becomes,
		\begin{align*}
		\mathcal{O}(( m + n + KA + K )S)= \mathcal{O}( ( n^{m_0} + n +n^{\frac{5}{2}\alpha }+ n^{\frac{3}{2}\alpha})(n^{ - \frac{\alpha }{2}}/\varepsilon)).
		\end{align*}
		We give three different ranges of $m_0$ to choose the best QCCS,
		\begin{itemize}
			\item  $0< m_0\le 1$: QCCS becomes $\mathcal{O}( (  n +n^{\frac{5}{2}\alpha })(n^{ - \frac{\alpha }{2}}/\varepsilon))$.
			
			Then, $\text{QCCS} = \left\{ {\begin{array}{*{20}{l}}
				{\mathcal{O}({n^{1 - \alpha /2}}/\varepsilon )},&{\alpha  \le 2/5};\\
				{\mathcal{O}({n^{2\alpha }}/\varepsilon )},&{\alpha  > 2/5}.
				\end{array}} \right.$
			\item $1<m_0<\frac{5}{2} $: QCCS becomes $\mathcal{O}( (  n^{m_0} +n^{\frac{5}{2}\alpha })(n^{ - \frac{\alpha }{2}}/\varepsilon))$. 
			
			Then, $\text{QCCS} = \left\{ {\begin{array}{*{20}{l}}
				{\mathcal{O}({n^{m_0 - \alpha /2}}/\varepsilon )},&{\alpha  \le 2m_0/5};\\
				{\mathcal{O}({{ {n^{2\alpha }}}}/\varepsilon )},&{\alpha  > 2m_0/5}.
				\end{array}} \right.$
		\end{itemize}
		Then, we have the QCCS with different $m_0$ and the best value of $\alpha$. Note that when ${m_0} \ge 5/2$, 
		\begin{align*}
		\text{QCCS }= \left\{ {\begin{array}{*{20}{l}}
			{\mathcal{O}({n^{4/5}}/\varepsilon) ,}&{\alpha  = 2/5,}&{0< m_0\le 1};\\
			{\mathcal{O}({n^{4{m_0}/5}}/\varepsilon) ,}&{\alpha  = 2{m_0}/5,}&{1<m_0<\frac{5}{2}.}
			\end{array}} \right.
		\end{align*}
	\end{proof}
	
	\textbf{Proof of Corollary \ref{VRNonCS:CorollaryComplexityQCS}:}
	\begin{proof} For the size of mini-batch  $\mathcal{A}_k$  is $m=A=n^{m_0}=n^{A_0}$, the query complexity  of stochastic (QCS) problem becomes,
		\begin{align*}
		\mathcal{O}( {( {m + n + K{A} + K} )S} )
		=& \mathcal{O}(({n^{{m_0}}} + n + {n^{\frac{3}{2}\alpha  + {m_0}}} + {\rm{ }}{n^{\frac{3}{2}\alpha }})({n^{ - \frac{\alpha }{2}}}/\varepsilon)\\
		=& \mathcal{O}(( {n^{m_0 }}+n + {n^{\frac{3}{2}\alpha  + {m_0}}})({n^{ - \frac{\alpha }{2}}}/\varepsilon)\\
		=& \mathcal{O}( {( { {n^{m_0 - \frac{\alpha }{2}}}+{n^{1 - \frac{\alpha }{2}}} + {n^{\alpha  + {m_0}}} } )( 1/\varepsilon)} ).
		\end{align*}
		With different range of $m_0$, we obtain the better query complexity by setting the best parameter $\alpha$,
		\begin{align*}
		\text{QCS} = \left\{ {\begin{array}{*{20}{l}}
			{\mathcal{O}({n^{\frac{2}{3} + \frac{1}{3}{m_0}}}/\varepsilon ),}&{\alpha  = \frac{{2\left( {1 - {m_0}} \right)}}{3},}&{{m_0} \le 1;}\\
			{\mathcal{O}({n^{{m_0}}}/\varepsilon ),}&{\alpha  = 0,}&{{m_0} > 1}.
			\end{array}} \right.
		\end{align*}	
	\end{proof}
	\subsection{Stochastic Composition with Variance reduction  II}
	This section gives the proof analysis of SCVRII for convergence rate and query complexity.
	
	\textbf{Proof of Theorem \ref{VRNonCS:TheoremMainFull}:}
	\begin{proof} The proof process is similar to the Theorem \ref{VRNonCS:TheoremMain}.	Based on Lemma \ref{VRNonCS:LemmaBoundFunctionFFull} and Lemma \ref{VRNonCS:LemmaBoundvariableXFull}, we form a Lyapunov function,
		\begin{align*}
		&\mathbb{E}[ {f( {{x_{k + 1}}} )} ] + {c_{k + 1}}\mathbb{E}[ {\| {{x_{k + 1}} - \tilde x_s} \|^2} ]\\
		=&E[f({x_k})] - \left( {\frac{1}{2}\eta  - 2{L_f}{\eta ^2}} \right)E[{\left\| {\nabla f({x_k})} \right\|^2}]\\
		&+ \left( {B_G^4L_F^2\frac{1}{{2A}}\eta  + 2{L_f}\left( {B_G^4L_F^2\frac{1}{A} + B_F^2L_G^2\frac{1}{B} + L_f^2} \right){\eta ^2}} \right)E[{\| {{x_k} - \tilde x_s} \|^2}]\\
		&+c_{k+1}\left(\left( {1 + \left( {\frac{1}{h} + \frac{1}{d} + dB_G^4L_F^2\frac{{1}}{A}} \right){\eta} + 4\left( {B_G^4L_F^2\frac{1}{A} + B_F^2L_G^2\frac{1}{B} + L_f^2} \right)\eta^2} \right)\mathbb{E}[\| {{x_k} - \tilde x} \|^2]\right)\\
		&+ c_{k+1}( {{\eta}h + 4\eta^2} )\| {\nabla f( {{x_k}} )} \|^2\\
		=& \mathbb{E}[f({x_k})] + {c_k}\mathbb{E}[{\| {{x_k} - {{\tilde x}_s}} \|^2}]- {u_k}\mathbb{E}[{\| {\nabla f({x_k})} \|^2}],
		\end{align*}
		where 
		\begin{align*}
		{u_k} =& \left( {1/2 - {c_{k + 1}}h} \right)\eta  - \left( {2{L_f} + 4{c_{k + 1}}} \right){\eta ^2},\\
		{c_k} =& {c_{k + 1}}\left( {1 + \left( {\frac{1}{h} + \frac{1}{d} + dB_G^4L_F^2\frac{1}{A}} \right)\eta  + 4\left( {B_G^4L_F^2\frac{1}{A} + B_F^2L_G^2\frac{1}{B} + L_f^2} \right){\eta ^2}} \right)\\
		&+ B_G^4L_F^2\frac{1}{{2A}}\eta  + 2{L_f}\left( {B_G^4L_F^2\frac{1}{A} + B_F^2L_G^2\frac{1}{B} + L_f^2} \right){\eta ^2}.
		\end{align*}
		Then, we get 
		\begin{align*}
		{u_k}\mathbb{E}[ {\| {\nabla f( {{x_k}} )} \|^2} ] \le E[ {f( {{x_k}} )} ] + {c_k}\mathbb{E}[ {\| {{x_k} - \tilde x_s} \|^2} ] - ( {\mathbb{E}[ {f( {{x_{k + 1}}} )} ] + {c_{k + 1}}\mathbb{E}[ {\| {{x_{k + 1}} - \tilde x_s} \|^2} ]} ).
		\end{align*}
		Define $u = {\min _{0 \le k \le K - 1}}\{ {u_k}\} $, sum from $k=0$ to $k=K-1$, we can get
		\begin{align*}
		\frac{1}{K}\sum\limits_{k = 0}^{K-1} {\mathbb{E}[ {\| {\nabla f( {{x_k}} )} \|^2} ]}  &\le \frac{{\mathbb{E}[ {f( {{x_0}} )} ] - ( {\mathbb{E}[ {f( {{x_K}} )} ] + {c_K}\mathbb{E}[ {\| {{x_K} - \tilde x_s} \|^2} ]} )}}{{uK}}\\
		&\le \frac{{\mathbb{E}[ {f( {{x_0}} )} ] -\mathbb{E}[ {f( {{x_K}} )} ]}}{{uK}}.
		\end{align*}
		Since $x_0=\tilde{x}_s$, let $\tilde{x}_{s+1}=x_K$, we obtain,
		\begin{align*}
		\frac{1}{K}\sum\limits_{k = 0}^{K-1} {\mathbb{E}[ {\| {\nabla f( {{x_k}} )} \|^2} ]} \le \frac{{\mathbb{E}[ {f( {{{\tilde x}_s}} )} ] - \mathbb{E}[ {f( {{{\tilde x}_{s + 1}}} )} ]}}{uK}.
		\end{align*} 
		Summing the outer iteration from $s=0$ to $S-1$, we have
		\begin{align*}
		\mathbb{E}[{\left\| {\nabla f({\tilde x_k^s })} \right\|^2}] = \frac{1}{S}\sum\limits_{s = 0}^{S - 1} {\frac{1}{K}\sum\limits_{k = 0}^{K - 1} {\mathbb{E}[{{\left\| {\nabla f(x_k^s)} \right\|}^2}]} }  \le \frac{{\mathbb{E}[f({{\tilde x}_0})] - \mathbb{E}[f({{\tilde x}_S})]}}{{uKS}} \le \frac{{f({x_0}) - f({x^*})}}{{uKS}},
		\end{align*}
		where $x_k^s$ indicates the $s$-th outer iteration at $k$-th inner iteration, and $\tilde x_k^s$ is uniformly and randomly chosen from  $s=\{0,...,S-1\}$ and k=$\{0,.., K-1\}$.
	\end{proof}
	\textbf{Proof of Corollary \ref{VRNonCS:CorollaryComplexityEstimationFull}:}
	\begin{proof}
		To analysis the bound of $c_0$, we use sequence $c_k$.
		Based on Corollary \ref{VRNonCS:LemmaGeometriProgression}, we have
		\begin{align*}
		{c_K} = {\left( {\frac{1}{Y}} \right)^K}\left( {{c_0} + \frac{U}{{Y - 1}}} \right) - \frac{U}{{Y - 1}},
		\end{align*}
		where
		\begin{align*}
		Y=& {1 + \left( {\frac{1}{h} + \frac{1}{d} + dB_G^4L_F^2\frac{1}{A}} \right)\eta  + 4\left( {B_G^4L_F^2\frac{1}{A} + B_F^2L_G^2\frac{1}{B} + L_f^2} \right){\eta ^2}} >1, \\
		U=&B_G^4L_F^2\frac{1}{{2A}}\eta  + 2{L_f}\left( {B_G^4L_F^2\frac{1}{A} + B_F^2L_G^2\frac{1}{B} + L_f^2} \right){\eta ^2}>0.
		\end{align*}
		Setting  $c_K=0$, we obtain ${c_0} = {U\left( {{Y^K} - 1} \right)}/(Y - 1)$.	Then, putting the Y and U  into the above equation. We have
		\begin{align}
		{c_0} =& \frac{{B_G^4L_F^2\frac{1}{{2A}} + 2{L_f}( {B_G^4L_F^2\frac{1}{A} + B_F^2L_G^2\frac{1}{B} + L_f^2} )\eta }}{{( {\frac{1}{h} + \frac{1}{d} + dB_G^4L_F^2\frac{1}{A}} ) + 4( {B_G^4L_F^2\frac{1}{A} + B_F^2L_G^2\frac{1}{B} + L_f^2} )\eta }}C\nonumber\\
		\label{VRNonCS:Definitionc_0Full}
		=& \frac{{{n^{ - {A_0}}} + {n^{ - \alpha }}}}{{( {{n^{ - {h_0}}}\left( {e - 1} \right) + {n^{ - {d_0}}} + 2{n^{ - ({A_0} - {d_0})}}} ) + 2{n^{ - \alpha }}/{L_f}}}C,
		\end{align}
		where
		\begin{align}
		\label{VRNonCS:DefinitionParametersAllFull}
		A =& \frac{1}{2}B_G^4L_F^2{n^{{A_0}}},B=B_F^2L_G^2n^{B_0}, d = {n^{{d_0}}},h = \frac{1}{{e - 1}}{n^{{h_0}}},\eta  = \frac{1}{{2{L_f}( {B_G^4L_F^2\frac{1}{A} + B_F^2L_G^2\frac{1}{B} + L_f^2} )}}{n^{ - \alpha }},\\
		\label{VRNonCS:DefinitionCFull}
		C =& {\left( {1 + \left( {\frac{1}{h} + \frac{1}{d} + dB_G^4L_F^2\frac{1}{A}} \right)\eta  +4\left( {B_G^4L_F^2\frac{1}{A} + B_F^2L_G^2\frac{1}{B} + L_f^2} \right)} \right)^K} - 1.
		\end{align}
		$c_k$ is a decrease sequence. $u_k$ is defined in (\ref{VRNonCS:DefinitionUk}), we have,
		\begin{align*}
		{u_k} &= ( 1/2 + {c_{k + 1}}h  ){\eta} - ( {2{L_f} + 4{c_{k + 1}}} )\\
		&\ge (  {1/2 - {c_0}h}  ){\eta} - ( {2{L_f} + 4{c_0}} )\eta^2,\,\,\,\,\,\,\,\forall k > 0.
		\end{align*}
		In order to satisfy ${c_0}h <1/2$, we set ${A_0} = \alpha ,{h_0} = {d_0} =\alpha/2$, ${A_0} = B_0$ and obtain
		\begin{align}\label{VRNonCS:Definitionc_0hFull}
		{c_0}h = \frac{1}{{( {e/2 + 1} ) + {{{n^{ - \frac{1}{2}\alpha }}}/L_f}}} C\frac{1}{{e - 1}}.
		\end{align}
		Based on the setting  in (\ref{VRNonCS:DefinitionParametersAllFull}) and (\ref{VRNonCS:DefinitionCFull}),  we have
		\begin{align*}
		C =& {( {1 + ( {( {e - 1} ){n^{ - \frac{1}{2}\alpha }} + {n^{ - \frac{1}{2}\alpha }} + {n^{ - \frac{1}{2}\alpha }}} )\eta  + 4( {2n^{-\alpha}+n^{-\alpha}+L_f^2} ){\eta ^2}} )^K} - 1\\
		=& {( {1 + (e + 1){n^{ - \frac{1}{2}\alpha }}\eta  +  4( {2n^{-\alpha}+n^{-\alpha}+L_f^2} ){\eta ^2}} )^K} - 1< e - 1,
		\end{align*}
		where the last equation is from the  character of function ${\left( {1 + \frac{1}{t}} \right)^t} \to e$, as $t \to  + \infty $. There exists a constant $w_4>0$ such that
		\begin{center}
			$K = 1/((e + 1){n^{ - \frac{1}{2}\alpha }}\eta  + (1 + B_G^4L_F^2{n^{ - \alpha }}){\eta ^2})={w_4}L_f^3{n^{\frac{3}{2}\alpha }}.$
		\end{center}
		Thus, we  obtain $
		{c_0}h <1/{{( {e/2 + 1} ) + {n^{ - \frac{1}{2}\alpha } /L_f}}} <{1}/{2},$
		which satisfy $  {1/2 -{c_0}h} > 0$.
		
		For $u_k$ defined in (\ref{VRNonCS:DefinitionUkFull}), there also exists a constant $v_4>0$  that satisfy, 
		\begin{align}\label{VRNonCS:DefinitionMaxuTilde}
		u = \mathop {\max }\limits_k \{ {{u_k}} \} \ge ( {{1/2- {c_0}h} } )\eta  - ( {2{L_f} + 4{c_0}} ){\eta ^2} =v_4 {n^{ - \alpha }}/L_f^3.
		\end{align}
		Combine with Theorem \ref{VRNonCS:TheoremMainFull}, we have
		\begin{align*}
		\mathbb{E}[{\| {\nabla f({\tilde x_k^s })} \|^2}]   \le \frac{{{n^\alpha }L_f^3(f({x_0}) - f({x^*}))}}{{v_4SK}}.
		\end{align*}
	\end{proof}
	\textbf{Proof of Corollary \ref{VRNonCS:CorollaryQueryComplexityFull}:}
	\begin{proof} Based on Corollary \ref{VRNonCS:CorollaryComplexityEstimationFull}, the outer number of iteration is $\mathcal{O}(T/K)$; the inner number of query is ${\cal O}(2m + n + 2K\left( {A + B} \right) + 4K)$, then the query complexity is 
		\begin{align*}
		{\cal O}\left( {(2m + n + 2K\left( {A + B} \right) + 4K)\frac{T}{K}} \right) 
		=& {\cal O}((2m + n + 2K\left( {A + B} \right) + 4K)S)\\
		=& {\cal O}((m + n + {n^{\frac{3}{2}\alpha  + \alpha }} + {n^{\frac{3}{2}\alpha  + {B_0}}} )({n^{ - \frac{\alpha }{2}}}/\varepsilon )).
		\end{align*}
	\end{proof}
	
	\subsection{Mini-batch of  Stochastic Composition with Variance reduction}
	We give the proofs for mini-batch SCVR, all the proof process are based on the SCVRI and SCVRII, but with different parameters setting such that gives the different results.
	
	\textbf{Proof of Corollary \ref{VRNonCS:CorollaryComplexityEstimationFullMiniBatch}:}
	\begin{proof}
		To analysis the bound of $c_0$, we use sequence $c_k$ \ref{VRNonCS:DefinitionCkfullMiniBatch}.
		Based on Corollary \ref{VRNonCS:LemmaGeometriProgression}, we have
		\begin{align*}
		{c_K} = {\left( {\frac{1}{Y}} \right)^K}\left( {{c_0} + \frac{U}{{Y - 1}}} \right) - \frac{U}{{Y - 1}},
		\end{align*}
		where
		\begin{align*}
		Y=& {1 + \left( {\frac{1}{h} + \frac{1}{d} + dB_G^4L_F^2\frac{1}{A}} \right)\eta  + 4\left( {B_G^4L_F^2\frac{1}{A} + B_F^2L_G^2\frac{1}{B} + bL_f^2} \right)\frac{1}{b}{\eta ^2}} >1, \\
		U=&B_G^4L_F^2\frac{1}{{2A}}\eta  + 2{L_f}\left( {B_G^4L_F^2\frac{1}{A} + B_F^2L_G^2\frac{1}{B} + bL_f^2} \right)\frac{1}{b}{\eta ^2}>0.
		\end{align*}
		Setting  $c_K=0$, we obtain ${c_0} = {U( {{Y^K} - 1} )}/(Y - 1)$.	Then, putting the Y and U  into the above equation. We have
		\begin{align*}
		{c_0} =& \frac{{B_G^4L_F^2\frac{1}{{2A}} + 2{L_f}( {B_G^4L_F^2\frac{1}{A} + B_F^2L_G^2\frac{1}{B} + bL_f^2} )\frac{1}{b}\eta }}{{( {\frac{1}{h} + \frac{1}{d} + dB_G^4L_F^2\frac{1}{A}} ) + 4( {B_G^4L_F^2\frac{1}{A} + B_F^2L_G^2\frac{1}{B} + bL_f^2} )\frac{1}{b}\eta }}C\\
		=& \frac{{{n^{ - {A_0}}} + {n^{ - \alpha }}}}{{( {{n^{ - {h_0}}}( {e - 1} ) + {n^{ - {d_0}}} + 2{n^{ - ({A_0} - {d_0})}}} ) + 2{n^{ - \alpha }}/{L_f}}}C,
		\end{align*}
		where
		\begin{align}
		\label{VRNonCS:DefinitionParametersAllFullMiniBatch}
		A =& \frac{1}{2}B_G^4L_F^2{n^{{A_0}}},B=B_F^2L_G^2n^{B_0},d = {n^{{d_0}}},h = \frac{1}{{e - 1}}{n^{{h_0}}},\eta  = \frac{b}{{2{L_f}( {B_G^4L_F^2\frac{1}{A} + B_F^2L_G^2\frac{1}{B} + bL_f^2} )}}{n^{ - \alpha }},\\
		\label{VRNonCS:DefinitionCFullMiniBatch}
		C =& {\left( {1 + \left( {\frac{1}{h} + \frac{1}{d} + dB_G^4L_F^2\frac{1}{A}} \right)\eta  + 4\left( {B_G^4L_F^2\frac{1}{A} + B_F^2L_G^2\frac{1}{B} + bL_f^2} \right)\frac{1}{b}{\eta ^2}} \right)^K} - 1.
		\end{align}
		$c_k$ is a decrease sequence. $u_k$ is defined in (\ref{VRNonCS:DefinitionUkFullMiniBatch}), we have,
		\begin{align*}
		{u_k} = ( 1/2 + {c_{k + 1}}h  ){\eta} - ( {2{L_f} + 4{c_{k + 1}}} )
		\ge (  {1/2 - {c_0}h}  ){\eta} - ( {2{L_f} + 4{c_0}} )\eta^2,\,\,\,\,\,\,\,\forall k > 0.
		\end{align*}
		In order to satisfy ${c_0}h <1/2$, we set ${A_0} = \alpha ,{h_0} = {d_0} =\alpha/2$, $A_0=B_0$\footnote{we give the comments about the size of $B_0$ in Remark \ref{VRNonCS:RemarkVRNCSIISizeB}} and obtain
		\begin{align*}
		{c_0}h = \frac{1}{{( {e/2 + 1} ) + n^{ - \frac{1}{2}\alpha }/L_f}} C\frac{1}{{e - 1}}.
		\end{align*}
		Based on the setting  in (\ref{VRNonCS:DefinitionParametersAllFullMiniBatch}) and (\ref{VRNonCS:DefinitionCFullMiniBatch}),  we have	
		\begin{align*}
		C =& {(1 + ((e - 1){n^{ - \frac{1}{2}\alpha }} + {n^{ - \frac{1}{2}\alpha }} + {n^{ - \frac{1}{2}\alpha }})\eta  + 4(bL_f^2 +2n^{-\alpha}+n^{-\alpha}){\eta ^2}/b)^K} - 1\\
		=& {( {1 + (e + 1){n^{ - \frac{1}{2}\alpha }}\eta + 4(bL_f^2 +2n^{-\alpha}+n^{-\alpha}){\eta ^2}/b} )^K} - 1<e - 1,
		\end{align*}
		where the last equation is from the  character of function ${\left( {1 + \frac{1}{t}} \right)^t} \to e$, as $t \to  + \infty $. Combining with $\eta$ defined in (\ref{VRNonCS:DefinitionParametersAllFullMiniBatch}), there exist a constant $w_5>0$ such that
		\begin{align}\label{VRNonCS:DefinitionKFullMiniBatch}
		K = {1}/({{(e + 1){n^{ - \frac{1}{2}\alpha }}\eta  + (1 + B_G^4L_F^2{n^{ - \alpha }}){\eta ^2}}})= {w_5}L_f^3{n^{\frac{3}{2}\alpha }}/b.
		\end{align}		
		Thus, we  obtain 
		\begin{align*}
		{c_0}h < {1}/({{( {e/2 + 1} ) + {n^{ - \frac{1}{2}\alpha }/L_f }}}) <{1}/{2},
		\end{align*}
		which satisfy $  {1/2 -{c_0}h}  > 0$.		
		For $u_k$ defined in (\ref{VRNonCS:DefinitionUkFullMiniBatch}), there also exist a constant $v_5>0$  that satisfy, 
		\begin{align*}
		u = \mathop {\max }\limits_k \{ {{u_k}} \} \ge ( {{1/2- {c_0}h} } )\eta  - ( {2{L_f} + 4{c_0}} ){\eta ^2} = {bv_5}{n^{ - \alpha }}/{{L_f^3}}.
		\end{align*}
		Combine with Theorem \ref{VRNonCS:TheoremMainFullMiniBatch}, we have
		\begin{align*}
		\mathbb{E}[{\| {\nabla f({{\tilde x_k^s } })} \|^2}]   \le \frac{{{n^\alpha }L_f^3(f({x_0}) - f({x^*}))}}{{bv_5SK}}.
		\end{align*}
	\end{proof}
	The following gives the detailed analysis of query complexity.
	
	\textbf{Proof of Corollary \ref{VRNonCS:CorollaryQueryComplexityMiniBatchParallel}:}
	\begin{proof} Based on Corollary \ref{VRNonCS:CorollaryComplexityEstimationFullMiniBatch}, the number of inner iteration K in (\ref{VRNonCS:DefinitionKFullMiniBatch}), $K = \mathcal{O}({n^{3\alpha /2}}/b)$. The number of outer iteration is 
		\begin{center}
			$\text{QC} = \left\{ {\begin{array}{*{20}{l}}
				{{\cal O}({n^{4/5 - {b_0}/5}}/\varepsilon )},&{0 < {m_0} \le 1};\\
				{{\cal O}({n^{4/5{m_0} - {b_0}/5}}/\varepsilon )},&{1 < {m_0}}.
				\end{array}} \right.$
		\end{center}
		Thus, the query complexity is 
		\begin{align*}
		&\mathcal{O}( {(2m + n + 2K( {A + B} ) + K)S} )\\
		=& \mathcal{O}((m + n + {n^{\frac{3}{2}\alpha  + \alpha }}/b + {n^{\frac{3}{2}\alpha }}/b)({n^{ - \frac{\alpha }{2}}}/\varepsilon ))\\
		=& \mathcal{O}(({n^{{m_0}}} + n + {n^{\frac{3}{2}\alpha  + \alpha  - {b_0}}} + {n^{\frac{3}{2}\alpha  - {b_0}}})({n^{ - \frac{\alpha }{2}}}/\varepsilon ))\\
		=& \mathcal{O}(({n^{{m_0}}} + n + {n^{\frac{3}{2}\alpha  + \alpha  - {b_0}}})({n^{ - \frac{\alpha }{2}}}/\varepsilon )).
		\end{align*}
		Based on the different range of $m_0$, we obtain
		\begin{itemize}
			\item $0 < {m_0} \le 1$: QCCS becomes $\mathcal{O}((n + {n^{\frac{3}{2}\alpha  + \alpha  - {b_0}}})({n^{ - \frac{\alpha }{2}}}/\varepsilon ))$, then
			
			$\text{QCCS} = \left\{ {\begin{array}{*{20}{l}}
				{{\cal O}({n^{1 - \alpha /2}}/\varepsilon ),}&{\alpha  \le 2/5(1 + {b_0});}\\
				{{\cal O}({n^{2\alpha  - {b_0}}}/\varepsilon ),}&{\alpha  > 2/5(1 + {b_0}).}
				\end{array}} \right.$
			\item $1 < {m_0}$: QCCS becomes $\mathcal{O}(({n^{{m_0}}}  + {n^{\frac{3}{2}\alpha  + \alpha  - {b_0}}})({n^{ - \frac{\alpha }{2}}}/\varepsilon ))$, then 
			
			$\text{QCCS} = \left\{ {\begin{array}{*{20}{l}}
				{{\cal O}({n^{{m_0} - \alpha /2}}/\varepsilon ),}&{\alpha  \le 2/5({m_0} + {b_0});}\\
				{{\cal O}({n^{2\alpha  - {b_0}}}/\varepsilon ),}&{\alpha  > 2/5({m_0} + {b_0}).}
				\end{array}} \right.$
		\end{itemize}
		Choose the best $\alpha= 2/5(1 + {b_0})$ and $\alpha= 2/5(m_0 + {b_0})$, we obtain
		\begin{center}
			$\text{QC} = \left\{ {\begin{array}{*{20}{l}}
				{{\cal O}({n^{4/5 - {b_0}/5}}/\varepsilon )},&{0 < {m_0} \le 1};\\
				{{\cal O}({n^{4/5{m_0} - {b_0}/5}}/\varepsilon )},&{1 < {m_0}}.
				\end{array}} \right.$
		\end{center}
	\end{proof}
	\textbf{Proof of Corollary \ref{VRNonCS:CorollaryQueryComplexityMiniBatchParallelAddG}:}
	
	\begin{proof} When the partial gradient  $\partial\hat G_k$ and inner function $\hat G$  are computed in parallel, the query complexity analysis is the same as in on Corollary \ref{VRNonCS:CorollaryComplexityQCS} but with mini-batch $\mathcal{I}_k$. The number of inner iteration K becomes $K = \mathcal{O}({n^{3\alpha /2}}/b)$. The number of outer iteration is 
		\begin{center}
			$S = \frac{T}{K} = \mathcal{O}\left( {\frac{{{n^\alpha }}}{{\varepsilon bK}}} \right) = \mathcal{O}\left( {\frac{{{n^\alpha }}}{{\varepsilon {n^{3\alpha /2}}}}} \right) = \mathcal{O}\left( {\frac{{{n^{ - \alpha /2}}}}{\varepsilon }} \right)$.
		\end{center}
		Thus, the query complexity becomes 
		\begin{align*}
		\mathcal{O}\left( {(2m + n + 2K + K)S} \right)
		= \mathcal{O}(({n^{{m_0}}} + n  + {n^{\frac{3}{2}\alpha  - {b_0}}})({n^{ - \frac{\alpha }{2}}}/\varepsilon )).
		\end{align*}
		Based on the different range of $m_0$, we obtain
		\begin{center}
			$	\text{QC} = \left\{ {\begin{array}{*{20}{l}}
				{{\cal O}({n^{2/3 - {b_0}/3}}/\varepsilon )},&{0 < {m_0} \le 1};\\
				{{\cal O}({n^{2/3{m_0} - {b_0}/3}}/\varepsilon )},&{1 < {m_0} }.
				\end{array}} \right.$
		\end{center}
	\end{proof}
	\textbf{Proof of Corollary \ref{VRNonCS:CorollaryQueryComplexityMiniBatchUnParallel}:}
	
	\begin{proof} Based on Corollary \ref{VRNonCS:CorollaryComplexityEstimationFullMiniBatch},  the number of inner iteration K in (\ref{VRNonCS:DefinitionKFullMiniBatch}), $K = \mathcal{O}({n^{3\alpha /2}}/b)$. The query complexity is 
		\begin{align*}
		&\mathcal{O}\left( {(2m + n + 2K\left( {A + B} \right) + Kb)S} \right)\\
		=& \mathcal{O}((n^{m_0} + n + {n^{\frac{3}{2}\alpha  + \alpha }}/b + {n^{\frac{3}{2}\alpha }})({n^{ - \frac{\alpha }{2}}}/\varepsilon ))\\
		=& \mathcal{O}((n^{m_0} + n + {n^{\frac{3}{2}\alpha  + \alpha  - {b_0}}} + {n^{\frac{3}{2}\alpha }})({n^{ - \frac{\alpha }{2}}}/\varepsilon )).
		\end{align*}
		Based on the different range of $b_0$, we have
		\begin{itemize}
			\item  $b_0\le 2/3$, QC becomes
			
			$	\mathcal{O}((n^{m_0} + n + {n^{\frac{3}{2}\alpha  + \alpha  - {b_0}}} )({n^{ - \frac{\alpha }{2}}}/\varepsilon )) = \left\{ {\begin{array}{*{20}{l}}
				{{\cal O}({n^{4/5 - {b_0}/5}}/\varepsilon )},&{0 < {m_0} \le 1};\\
				{{\cal O}({n^{4/5{m_0} - {b_0}/5}}/\varepsilon )},&{1 < {m_0} }.
				\end{array}} \right.$
			
			\item $b_0> 2/3$, QC becomes
			
			$	\mathcal{O}((n^{m_0} + n + {n^{\frac{3}{2}\alpha }} )({n^{ - \frac{\alpha }{2}}}/\varepsilon )) = \left\{ {\begin{array}{*{20}{l}}
				{{\cal O}({n^{2/3}}/\varepsilon )},&{0 < {m_0} \le 1};\\
				{{\cal O}({n^{2/3{m_0} }}/\varepsilon )},&{1 < {m_0}}.
				\end{array}} \right.$
		\end{itemize}	
	\end{proof}
	\section{Experiments}

	In this section, we apply the proposed SCVR method to the non-convex nonlinear embedding problem. We use the dissimilar distance between $x_i$ and $x_j$, $d\left( {{x_i},{x_j}} \right) = \exp ( - {\left\| {{x_i} - {x_j}} \right\|^2}/2\sigma _i^2)$ as an example. The problem can be formulated as the composition optimization. Note that, here, $n=m=n^{m_0}$, the number of inner sub-functions $G_k$ is the same as that of outer sub-functions $F_j$, that is, $m_0=1$.
	
	\begin{figure*}
		\centering
		\subfigure{
			\begin{minipage}[b]{0.48\textwidth}
				\includegraphics[width=1.0\textwidth]{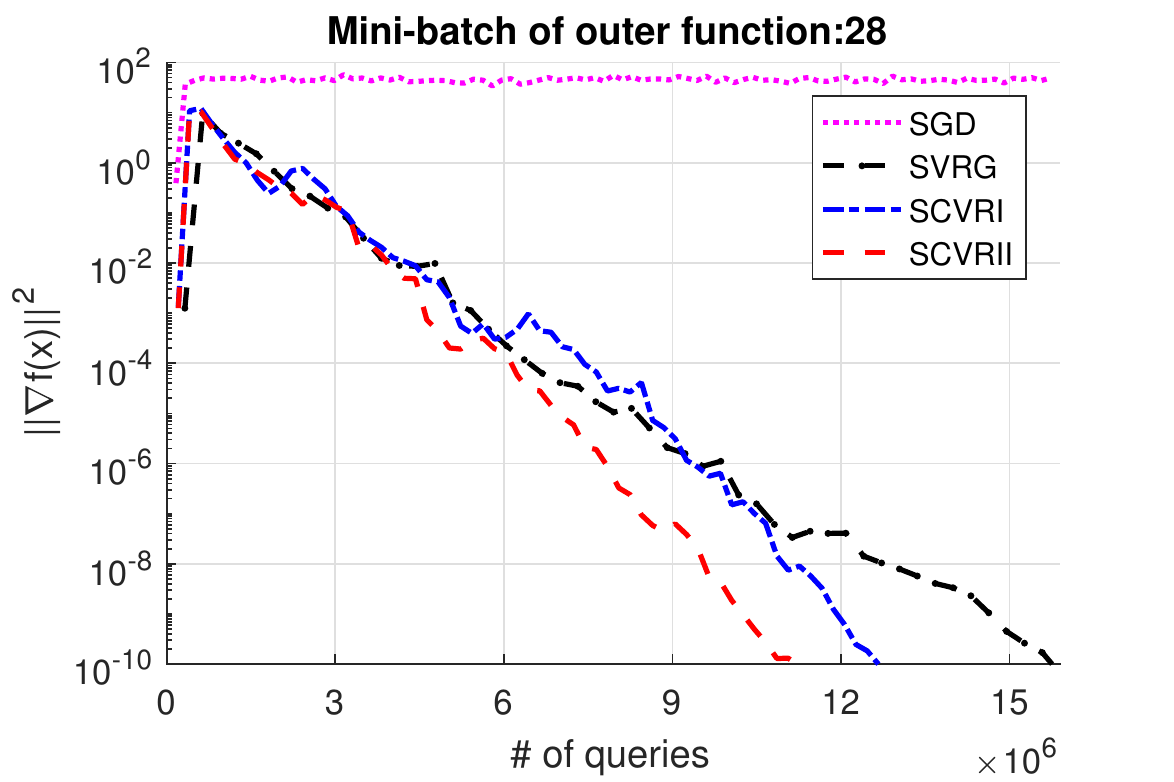}\\
				\includegraphics[width=1.0\textwidth]{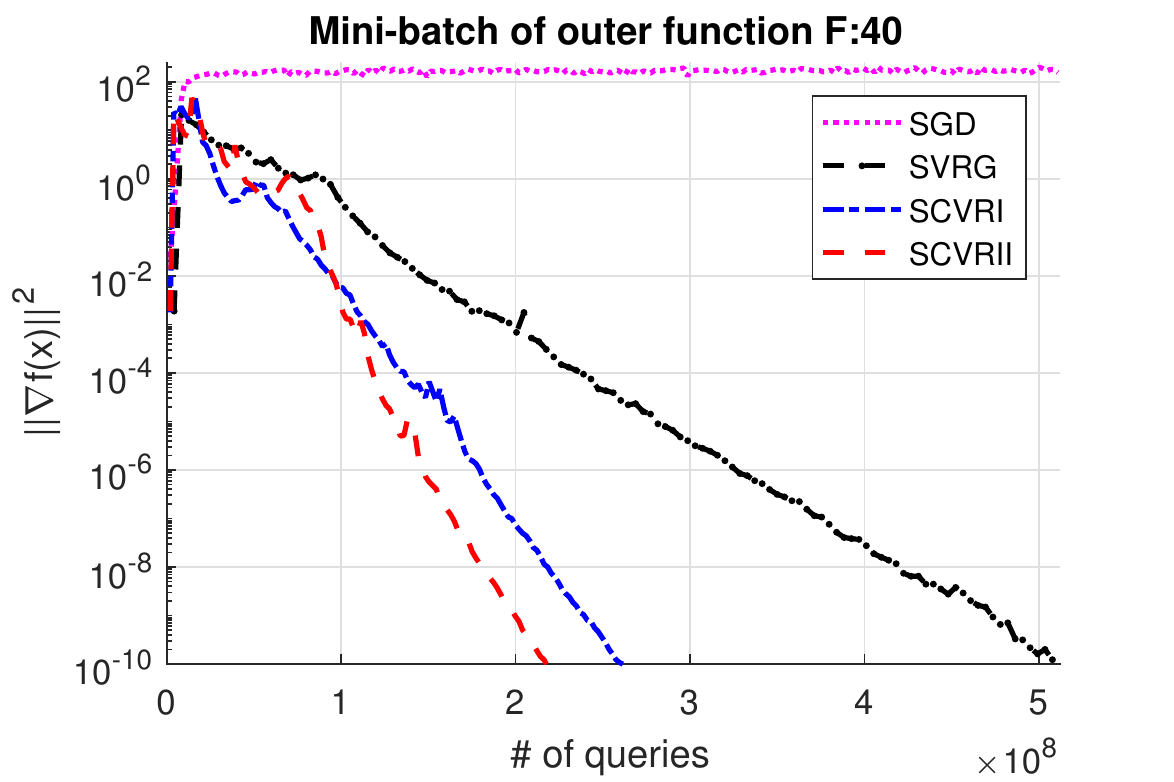}\\
				\includegraphics[width=1.0\textwidth]{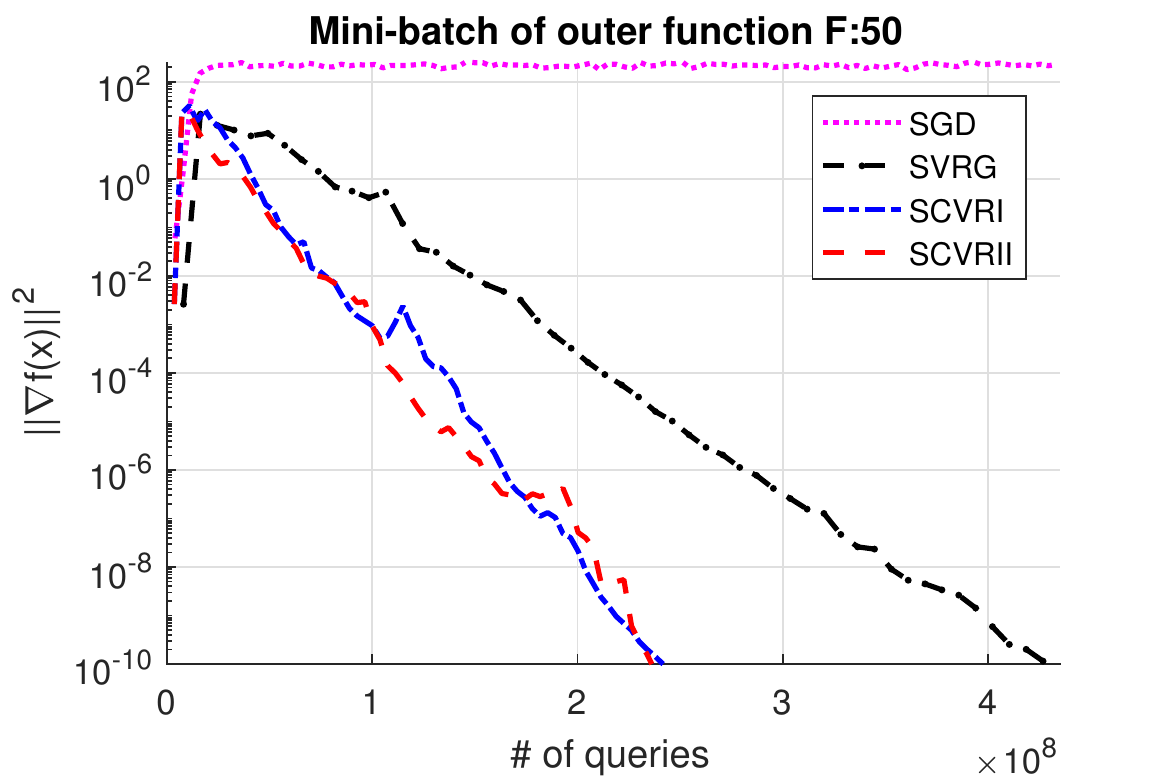}
			\end{minipage}
		}
		\subfigure{
			\begin{minipage}[b]{0.48\textwidth}
				\includegraphics[width=1.0\textwidth]{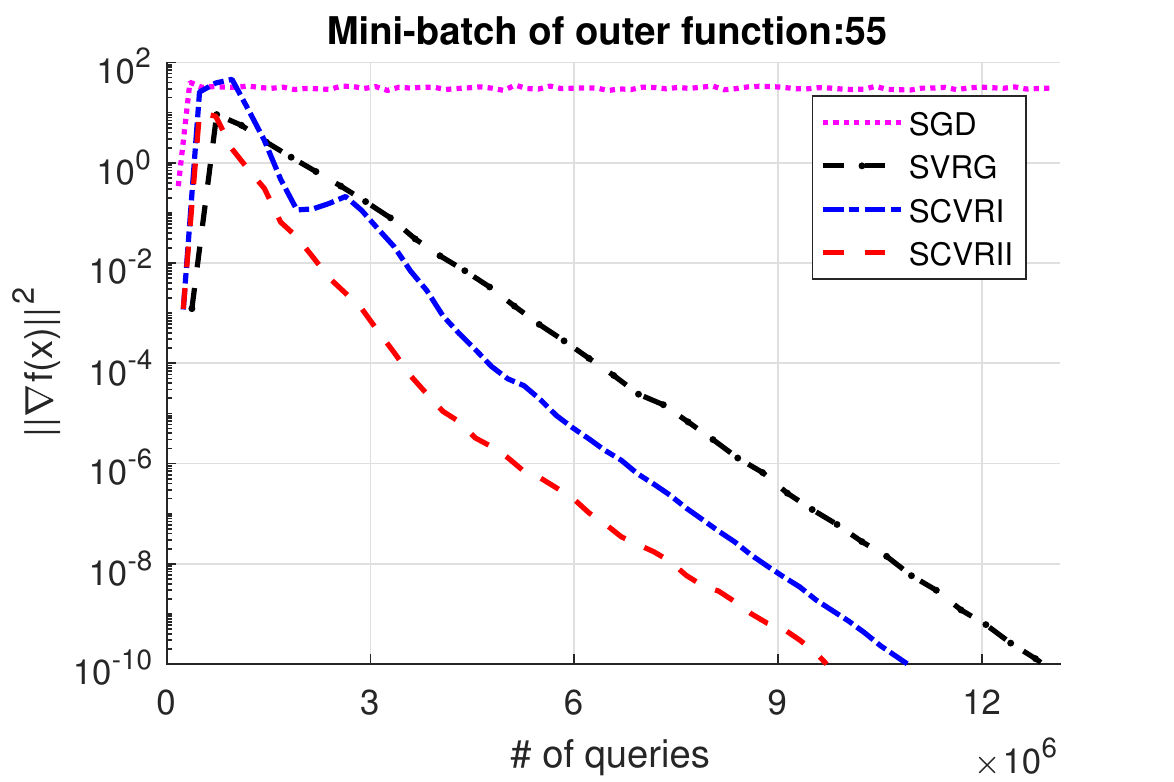}\\
				\includegraphics[width=1.0\textwidth]{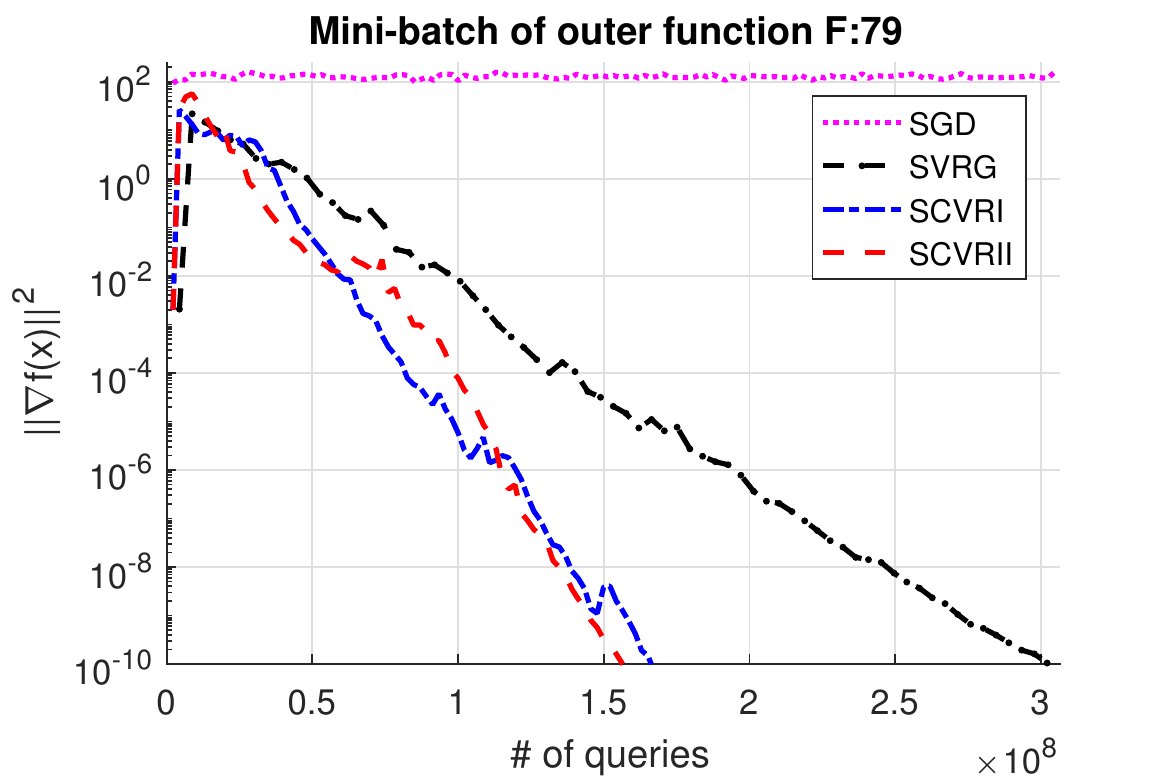}\\
				\includegraphics[width=1.0\textwidth]{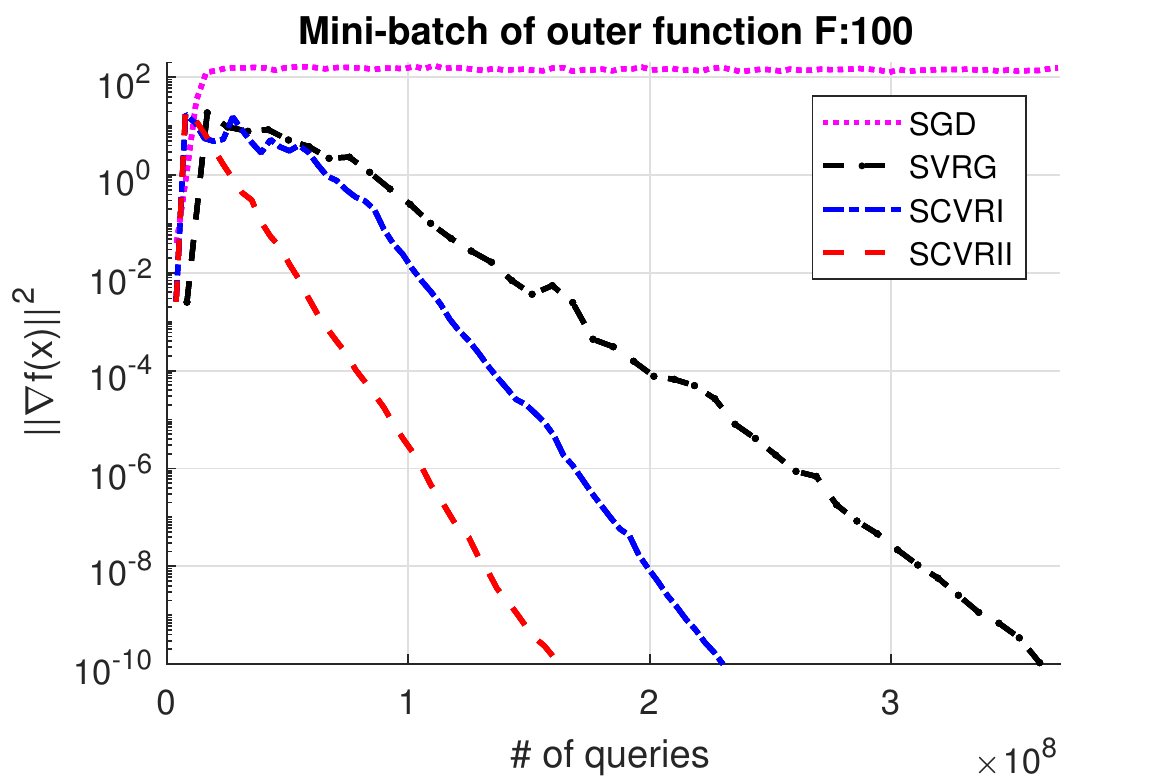}
			\end{minipage}
		}
		\caption{Comparison of the convergence rate  with four different algorithms: SGD, SVRG, SCVRI and SCVRII. The $x$-axis represents the number of queries, and the $y$-axis represents the  value of ${\| {\nabla f( x )} \|^2}$. From top to bottom, we respectively use three different datasets: Olivetti faces, COIL-20 and MINIST; From left to right, we use two different sizes of mini-batch of outer function $F$: $n^{2/3}$  and $n^{2/3}/2$, where $n$ is the number of the outer sub-function.}
		\label{VRNonCS:FigurResult}	
	\end{figure*}
	
	\begin{align*}
	\frac{1}{n}\sum\limits_{j = 1}^n {{F_j}\left( w \right)}  =&\frac{1}{n} \sum\limits_{j = 1}^n {{F_j}\left( {\frac{1}{n}\sum\limits_{k = 1}^n {{G_k}\left( y \right)} } \right)}\\  =& \sum\limits_{j = 1}^n {\left( {\sum\limits_{i = 1}^n {{p_{j|i}}\left( {{{\left\| {{y_i} - {y_j}} \right\|}^2} + \log \left( {\sum\limits_{k = 1}^n {{e^{ - {{\left\| {{y_i} - {y_k}} \right\|}^2}}}}  - 1} \right)} \right)} } \right)} \\
	=& \sum\limits_{j = 1}^n {\left( {\sum\limits_{i = 1}^n {{p_{j|i}}( {{{\| {{w_i} - {w_j}} \|}^2} + \log ( {{w_{n + i}}} )} )} } \right)}, 
	\end{align*}
	where 
	\begin{align*}
	w =& \frac{1}{n}\sum\limits_{k = 1}^n {{G_k}\left( y \right)}; \\
	{G_k}\left( y \right) =& {[ {y,n{e^{ - {{\left\| {{y_1} - {y_k}} \right\|}^2}}} - 1,...,n{e^{ - {{\left\| {{y_n} - {y_k}} \right\|}^2}}} - 1} ]^\mathsf{T}},k \in [n];\\
	{F_j}\left( w \right) =& n\sum\limits_{i = 1}^n {{p_{j|i}}( {{{\| {{w_i} - {w_j}} \|}^2} + \log ( {{w_{n + i}}} )} )} ,j \in [n].
	\end{align*}

	We use three datasets: Olivetti faces,  COIL-20, and MNIST, with  the  sample size and dimension   $400 \times 4096$, $1440\times 16384$, and $2000 \times 784$, respectively. We use two proposed algorithms: mini-batch SCVRI and mini-batch SCVRII (we use SCVRI and SCVRII for short in Figure \ref{VRNonCS:FigurResult}). We choose the sample times for forming the mini-batch multiset $\mathcal{A}$ and $\mathcal{B}$ with the same order of $O(m^{2/5})$, where $m$ represents the number of inner sub-functions. We use two different sizes of  mini-batch of outer sub-function $F_i$ with $n^{2/3}$  and $n^{2/3}/2$  for each algorithm, where $n$ is the number of  outer sub-functions. We compare our proposed algorithms with SGD and SVRG. For SGD and SVRG, we compute the inner function $G$ and partial gradient $\partial G$ directly without forming the mini-batch multiset $\mathcal{A}$ and $\mathcal{B}$, but use the same mini-batch sizes of outer sub-function as SCVRI and SCVRII. The preliminary processes are the same for all methods, namely to normalize the data and use PCA to reduce the dimension to 30 first. To verify the proposed algorithm, we set $y_i\in R^2$, $i\in [n]$.  We select the best $\eta$ for each method to give the fastest convergence. The experimental results are shown in Figure \ref{VRNonCS:FigurResult}. From these results, we make the following observations:
	\begin{itemize}
		\item Consistent with theory, the proposed SCVRI and SCVRII enjoy faster convergence rates than those of SVRG and SGD.
		\item The query complexities of SCVRI and SCVRII are comparable, which is also verified by theory that they have the same order of convergence rate but with different parameters setting. For most cases, the convergence rate of SCVRII is better than that of SCVRI.
		\item The convergence rate of  the mini-batch size of the outer function $F$ with $n^{2/3}$ is better than that of $n^{2/3}/2$, which is also consistent with our theoretical analysis.
	\end{itemize}

	\section{Conclusions}
	In this paper, we present a variance reduction-based method for the non-convex stochastic composition  problem. Based on different gradient estimators, we present three methods: SCVRI, SCVRII, and mini-batch SCVR. The convergence analysis shows that the convergence rate of the proposed method depends on $n$. Furthermore, we analyze  different sizes $m=n^{m_0}$ of inner subfunction $G_i$ and give the best query complexity with different $m_0$.  Our theoretical analysis shows that the proposed methods have better query complexity than that of both SGD and SVRG under the condition $m_0>2/5$. Our hope is that these theoretical results will provide new insights into other applications such as deep learning. 
	
	
	

	\appendix
	\section*{Appendix A.}
	\begin{lemma}\label{VRNonCS:AppendixLemmaRandomVariableInequation}
		For the random variable $X$, we have $
		\mathbb{E}[ {{{\| {X - \mathbb{E}[ X ]} \|}^2}} ] { = } \mathbb{E}[ {{X^2} - {{\| {\mathbb{E}[ X ]} \|}^2}} ] \le \mathbb{E}[ {{X^2}} ].$
	\end{lemma}
	\begin{lemma}\label{VRNonCS:AppendixLemmaInequation}
		For  random variables $X_1,...,X_r$, we have $\mathbb{E}[ {{{\| {{X_1} + ... + {X_r}} \|}^2}} ] \le r( \mathbb{E}[ {{{\| {{X_1}} \|}^2}} ] + ... + \mathbb{E}[ {{{\| {{X_r}} \|}^2}} ] ).$
	\end{lemma}
	\begin{lemma}\label{VRNonCS:AppendixLemmaInEquationWithab}
		For $a$ and $b$, we have $2\langle {a,b} \rangle  \le 1/q\| a \|^2 + q\| b \|^2,\forall q > 0$.
	\end{lemma}
	\begin{lemma}\label{VRNonCS:LemmaGeometriProgression}
		For the sequences that satisfy ${c_{k - 1}} = {c_k}Y + U$, where $Y>1$, $U>0$, $k\ge 1$ and $c_0>0$, we can get the geometric progression
		
		\centering{${c_k} + \frac{U}{{Y - 1}} = \frac{1}{Y}\left( {{c_{k - 1}} + \frac{U}{{Y - 1}}} \right),$}
		\leftline{then $c_k$ can be represented as decrease sequences,}
		\centering{${c_k} = {\left( {\frac{1}{Y}} \right)^k}\left( {{c_0} + \frac{U}{{Y - 1}}} \right) - \frac{U}{{Y - 1}}.$}
	\end{lemma}
	\section*{Appendix B.}\label{VRNonCS:AppendixB}
	\subsection*{Transform the general SNE problem to composition optimization problem}
	Given the dissimilarity $d(x_i,x_j)$ with respect to $x_i$ and $x_j$, we define 
	\begin{align*}
	{q_{i|t}} = \frac{{d\left( {{x_t},{x_i}} \right)}}{{\sum\limits_{j = 1,j \ne t}^n {d\left( {{x_t},{x_j}} \right)} }}.
	\end{align*}
	For example, \cite{hinton2003stochastic} use $d( {{x_i},{x_j}} ) = \exp ( - {\| {{x_i} - {x_j}} \|^2}/2\sigma _i^2)$, \cite{maaten2008visualizing} use $d( {{x_i},{x_j}} ) = {( {1 - {{\| {{x_i} - {x_j}} \|}^2}} )^{ - 1}}$.
	
	For the objective function with Kullback-Leibler divergences, the objective function can be formed as,
	\begin{align*}
	\sum\limits_{k = 1}^n {\left( {\sum\limits_{i = 1}^n {{p_{i|t}}\log \left( {\frac{{{p_{i|t}}}}{{{q_{i|t}}}}} \right)} } \right)}  = \sum\limits_{t = 1}^n {\left( {\sum\limits_{i = 1}^n {{p_{i|t}}\log {p_{i|t}}}  - \sum\limits_{i = 1}^n {{p_{i|t}}\log {q_{i|t}}} } \right)} .
	\end{align*}
	We can delete the first term and define the second term as the new objective function for briefness,
	\begin{align*}
	f\left( y \right) =& \sum\limits_{t = 1}^n {\left( {\sum\limits_{i = 1}^n { - {p_{i|t}}\log {q_{i|t}}} } \right)} \\
	=& \sum\limits_{t = 1}^n {\left( {\sum\limits_{i = 1}^n { - {p_{i|t}}\log \left( {\frac{{d\left( {{x_t},{x_i}} \right)}}{{\sum\limits_{j = 1,j \ne t}^n {d\left( {{x_t},{x_j}} \right)} }}} \right)} } \right)} \\
	=&  - \sum\limits_{t = 1}^n {\left( {\sum\limits_{i = 1}^n {{p_{i|t}}\left( {\log \left( {d\left( {{x_t},{x_i}} \right)} \right) - \log \left( {\sum\limits_{j = 1,j \ne t}^n {d\left( {{x_t},{x_j}} \right)} } \right)} \right)} } \right)} \\
	=&  - \sum\limits_{t = 1}^n {\left( {\sum\limits_{i = 1}^n {{p_{i|t}}\left( {\log \left( {d\left( {{x_t},{x_i}} \right)} \right) - \log \left( {\sum\limits_{j = 1}^n {d\left( {{x_t},{x_j}} \right)}  - d\left( {{x_t},{x_t}} \right)} \right)} \right)} } \right)} \\
	=&  - \sum\limits_{i = 1}^n {\left( {\sum\limits_{t = 1}^n {{p_{i|t}}\left( {\log \left( {d\left( {{x_t},{x_i}} \right)} \right) - \log \left( {\sum\limits_{j = 1}^n {d\left( {{x_t},{x_j}} \right)}  - d\left( {{x_t},{x_t}} \right)} \right)} \right)} } \right)} .
	\end{align*}
	Note that the difference between third and forth equalities is the exchange of the sum order, which is the key process for transforming the original problem to the composition problem (\ref{VRNonCS:ProblemMainComposition}). 
	
	Define the inner function $G$ as 
	\begin{align*}
	w = G\left( x \right) = \frac{1}{n}\sum\limits_{j = 1}^n {{G_j}\left( x \right)}  = &{\left[ {x,\sum\limits_{j \ne 1}^n {d\left( {{x_1},{x_j}} \right)} ,...,\sum\limits_{j \ne n}^n {d\left( {{x_n},{x_j}} \right)} } \right]^\mathsf{T}}\\
	=& {\left[ {x,\sum\limits_{j = 1}^n {d\left( {{x_1},{x_j}} \right)}  - d\left( {{x_1},{x_1}} \right),...,\sum\limits_{j = 1}^n {d\left( {{x_n},{x_j}} \right)}  - d\left( {{x_n},{x_n}} \right)} \right]^\mathsf{T}}\\
	=& \left[ {{w_1},...,{w_n},{w_{n + 1}},...,{w_{2n}}} \right],
	\end{align*}
	where 
	\begin{align*}
	{G_j}\left( x \right) =& n{\left[ {\frac{1}{n}y,d\left( {{x_1},{x_j}} \right) - d\left( {{x_1},{x_1}} \right),...,d\left( {{x_n},{x_j}} \right) - d\left( {{x_n},{x_n}} \right)} \right]^\mathsf{T}};\\
	{w_l} =& \left\{ {\begin{array}{*{20}{c}}
		{{x_l},}&{1 \le l \le n,}\\
		{\sum\limits_{j = 1}^n {d\left( {{x_{l - n}},{x_j}} \right)}  - d\left( {{x_{l - n}},{x_{l - n}}} \right),}&{n + 1 \le l \le 2n.}
		\end{array}} \right.
	\end{align*}
	Define the outer function $F$ as 
	
	\begin{align*}
	\frac{1}{n}\sum\limits_{i = 1}^n {{F_i}\left( w \right)}  =& \frac{1}{n}\sum\limits_{i = 1}^n {{F_i}\left( {\frac{1}{n}\sum\limits_{j = 1}^n {{G_j}\left( x \right)} } \right)} \\
	= & - \sum\limits_{i = 1}^n {\left( {\sum\limits_{t = 1}^n {{p_{i|t}}\left( {\log \left( {d\left( {{x_t},{x_i}} \right)} \right) - \log \left( {\sum\limits_{j = 1}^n {d\left( {{x_t},{x_j}} \right)}  - d\left( {{x_t},{x_t}} \right)} \right)} \right)} } \right)} \\
	=& \sum\limits_{i = 1}^n {\left( {\sum\limits_{t = 1}^n {{p_{i|t}}\left( {\log \left( {d\left( {{w_i},{w_j}} \right)} \right) - \log \left( {{w_{n + i}}} \right)} \right)} } \right)} ,
	\end{align*}
	where 
	\begin{align*}
	{F_i}\left( w \right) =n \sum\limits_{t = 1}^n {{p_{i|t}}\left( {\log \left( {d\left( {{w_t},{w_i}} \right)} \right) - \log \left( {{w_{n + i}}} \right)} \right)} 
	\end{align*}
	

\begin{thebibliography}{10}
	
	\bibitem{wang2017stochastic}
	Mengdi Wang, Ethan~X Fang, and Han Liu.
	\newblock Stochastic compositional gradient descent: algorithms for minimizing
	compositions of expected-value functions.
	\newblock {\em Mathematical Programming}, 161(1-2):419--449, 2017.
	
	\bibitem{wang2016accelerating}
	Ji~Liu, Mengdi Wang, and Ethan Fang.
	\newblock Accelerating stochastic composition optimization.
	\newblock In {\em Advances in Neural Information Processing Systems}, pages
	1714--1722, 2016.
	
	\bibitem{dai2016learning}
	Bo~Dai, Niao He, Yunpeng Pan, Byron Boots, and Le~Song.
	\newblock Learning from conditional distributions via dual kernel embeddings.
	\newblock {\em arXiv preprint arXiv:1607.04579}, 2016.
	
	\bibitem{hinton2003stochastic}
	Geoffrey~E Hinton and Sam~T Roweis.
	\newblock Stochastic neighbor embedding.
	\newblock In {\em Advances in neural information processing systems}, pages
	857--864, 2003.
	
	\bibitem{allen2016improved}
	Zeyuan Allen-Zhu and Yang Yuan.
	\newblock Improved svrg for non-strongly-convex or sum-of-non-convex
	objectives.
	\newblock In {\em International conference on machine learning}, pages
	1080--1089, 2016.
	
	\bibitem{reddi2016stochastic}
	Sashank~J Reddi, Ahmed Hefny, Suvrit Sra, Barnabas Poczos, and Alex Smola.
	\newblock Stochastic variance reduction for nonconvex optimization.
	\newblock In {\em International conference on machine learning}, pages
	314--323, 2016.
	
	\bibitem{ghadimi2016accelerated}
	Saeed Ghadimi and Guanghui Lan.
	\newblock Accelerated gradient methods for nonconvex nonlinear and stochastic
	programming.
	\newblock {\em Mathematical Programming}, 156(1-2):59--99, 2016.
	
	\bibitem{xiao2014proximal}
	Lin Xiao and Tong Zhang.
	\newblock A proximal stochastic gradient method with progressive variance
	reduction.
	\newblock {\em SIAM Journal on Optimization}, 24(4):2057--2075, 2014.
	
	\bibitem{johnson2013accelerating}
	Rie Johnson and Tong Zhang.
	\newblock Accelerating stochastic gradient descent using predictive variance
	reduction.
	\newblock In {\em Advances in neural information processing systems}, pages
	315--323, 2013.
	
	\bibitem{lian2016finite}
	Xiangru Lian, Mengdi Wang, and Ji~Liu.
	\newblock Finite-sum composition optimization via variance reduced gradient
	descent.
	\newblock In {\em AISTATS}, 2017.
	
	\bibitem{nesterov2013introductory}
	Yurii Nesterov.
	\newblock {\em Introductory lectures on convex optimization: A basic course},
	volume~87.
	\newblock Springer Science \& Business Media, 2013.
	
	\bibitem{li2015accelerated}
	Huan Li and Zhouchen Lin.
	\newblock Accelerated proximal gradient methods for nonconvex programming.
	\newblock In {\em Advances in neural information processing systems}, pages
	379--387, 2015.
	
	\bibitem{wang2016stochastic}
	Mengdi Wang and Ji~Liu.
	\newblock A stochastic compositional gradient method using markov samples.
	\newblock In {\em Proceedings of the 2016 Winter Simulation Conference}, pages
	702--713, 2016.
	
	\bibitem{dentcheva2016statistical}
	Darinka Dentcheva, Spiridon Penev, and Andrzej Ruszczy{\'n}ski.
	\newblock Statistical estimation of composite risk functionals and risk
	optimization problems.
	\newblock {\em Annals of the Institute of Statistical Mathematics}, pages
	1--24, 2016.
	
	\bibitem{defazio2014saga}
	Aaron Defazio, Francis Bach, and Simon Lacoste-Julien.
	\newblock Saga: A fast incremental gradient method with support for
	non-strongly convex composite objectives.
	\newblock In {\em Advances in Neural Information Processing Systems}, pages
	1646--1654, 2014.
	
	\bibitem{shalev2014accelerated}
	Shai Shalev-Shwartz and Tong Zhang.
	\newblock Accelerated proximal stochastic dual coordinate ascent for
	regularized loss minimization.
	\newblock In {\em International Conference on Machine Learning}, pages 64--72,
	2014.
	
	\bibitem{shalev2013stochastic}
	Shai Shalev-Shwartz and Tong Zhang.
	\newblock Stochastic dual coordinate ascent methods for regularized loss
	minimization.
	\newblock {\em Journal of Machine Learning Research}, 14(Feb):567--599, 2013.
	
	\bibitem{NIPS2016_6116}
	Sashank J.~Reddi, Suvrit Sra, Barnabas Poczos, and Alexander~J Smola.
	\newblock Proximal stochastic methods for nonsmooth nonconvex finite-sum
	optimization.
	\newblock In D.~D. Lee, M.~Sugiyama, U.~V. Luxburg, I.~Guyon, and R.~Garnett,
	editors, {\em Advances in Neural Information Processing Systems 29}, pages
	1145--1153. 2016.
	
	\bibitem{ge2015escaping}
	Rong Ge, Furong Huang, Chi Jin, and Yang Yuan.
	\newblock Escaping from saddle points----online stochastic gradient for tensor
	decomposition.
	\newblock In {\em Conference on Learning Theory}, pages 797--842, 2015.
	
	\bibitem{lee2016gradient}
	Jason~D Lee, Max Simchowitz, Michael~I Jordan, and Benjamin Recht.
	\newblock Gradient descent only converges to minimizers.
	\newblock In {\em Conference on Learning Theory}, pages 1246--1257, 2016.
	
	\bibitem{carmon2016accelerated}
	Yair Carmon, John~C Duchi, Oliver Hinder, and Aaron Sidford.
	\newblock Accelerated methods for non-convex optimization.
	\newblock {\em arXiv preprint arXiv:1611.00756}, 2016.
	
	\bibitem{agarwal2016finding}
	Naman Agarwal, Zeyuan Allen-Zhu, Brian Bullins, Elad Hazan, and Tengyu Ma.
	\newblock Finding approximate local minima for nonconvex optimization in linear
	time.
	\newblock {\em arXiv preprint arXiv:1611.01146}, 2016.
	
	\bibitem{maaten2008visualizing}
	Laurens van~der Maaten and Geoffrey Hinton.
	\newblock Visualizing data using t-sne.
	\newblock {\em Journal of Machine Learning Research}, 9(Nov):2579--2605, 2008.
	
\end{thebibliography}

\end{document}